\documentclass[letterpaper]{article} 
\usepackage{aaai25}  
\usepackage{times}  
\usepackage{helvet}  
\usepackage{courier}  
\usepackage[hyphens]{url}  
\usepackage{graphicx} 
\usepackage{natbib}  
\usepackage{caption} 
\usepackage{amsthm}
\usepackage{color}

\newtheorem{theorem}{Theorem}
\newtheorem{lemma}[theorem]{Lemma}
\usepackage{algorithm}  
\usepackage{algorithmic}  
\usepackage{microtype}
\usepackage{graphicx}
\usepackage{subfigure}
\usepackage{booktabs}

\usepackage{amsmath,amsfonts,bm}

\def\eqref#1{equation~\ref{#1}}

\def\1{\bm{1}}

\def\vphi{{\bm{\phi}}}

\def\vepsilon{{\bm{\epsilon}}}

\def\vtheta{{\bm{\theta}}}

\def\vc{{\bm{c}}}

\def\vf{{\bm{f}}}

\def\vw{{\bm{w}}}
\def\vx{{\bm{x}}}

\def\vz{{\bm{z}}}

\DeclareMathAlphabet{\mathsfit}{\encodingdefault}{\sfdefault}{m}{sl}
\SetMathAlphabet{\mathsfit}{bold}{\encodingdefault}{\sfdefault}{bx}{n}

\newcommand{\pdata}{p_{\rm{data}}}

\newcommand{\E}{\mathbb{E}}

\newcommand{\R}{\mathbb{R}}

\usepackage{newfloat}
\usepackage{listings}
\DeclareCaptionStyle{ruled}{labelfont=normalfont,labelsep=colon,strut=off} 
\floatstyle{ruled}
\newfloat{listing}{tb}{lst}{}
\floatname{listing}{Listing}
\title{SCott: Accelerating Diffusion Models with Stochastic Consistency Distillation }
\author{
    Hongjian Liu\textsuperscript{\rm 1}\thanks{These authors contributed equally. Work done during Hongjian Liu's internship at OPPO.}, Qingsong Xie\textsuperscript{\rm 2}\footnotemark[1]\thanks{Corresponding author.}, Tianxiang Ye\textsuperscript{\rm 3}, Zhijie Deng\textsuperscript{\rm 3}\footnotemark[2], Chen Chen\textsuperscript{\rm 2}, Shixiang Tang\textsuperscript{\rm 4}, Xueyang Fu\textsuperscript{\rm 1}, Haonan Lu\textsuperscript{\rm 2}, Zheng-Jun Zha\textsuperscript{\rm 1}
}
\affiliations{
    \textsuperscript{\rm 1} University of Science and Technology of China, China \quad
    \textsuperscript{\rm 2} OPPO AI Center \\
    \textsuperscript{\rm 3} Shanghai Jiao Tong University, China \quad
    \textsuperscript{\rm 4} The Chinese University of Hong Kong \\
    jeffeey@mail.ustc.edu.cn, \{xieqingsong1, chenchen4, luhaonan\}@oppo.com, \{coulson\_tx, zhijied\}@sjtu.edu.cn, \\
    \{xyfu, zhazj\}@ustc.edu.cn, shixiangtang@cuhk.edu.hk,
      \\
    
}

\usepackage{bibentry}
\usepackage[capitalize,noabbrev]{cleveref}

\begin{document}

\maketitle

\begin{abstract}
The iterative sampling procedure employed by diffusion models (DMs) often leads to significant inference latency. 
To address this, we propose Stochastic Consistency Distillation (SCott) to enable accelerated text-to-image generation, where high-quality and diverse generations can be achieved within just 2-4 sampling steps.
In contrast to vanilla consistency distillation (CD) which distills the ordinary differential equation solvers-based sampling process of a pre-trained teacher model into a student, SCott explores the possibility and validates the efficacy of integrating stochastic differential equation (SDE) solvers into CD to fully unleash the potential of the teacher. 
SCott is augmented with elaborate strategies to control the noise strength and sampling process of the SDE solver.
An adversarial loss is further incorporated to strengthen the consistency constraints in rare sampling steps.
Empirically, on the MSCOCO-2017 5K dataset with a Stable Diffusion-V1.5 teacher, SCott achieves an FID of 21.9 with 2 sampling steps, surpassing that of the 1-step InstaFlow (23.4) and the 4-step UFOGen (22.1).
Moreover, SCott can yield more diverse samples than other consistency models for high-resolution image generation, with up to $16\%$ improvement in a qualified metric. 

\end{abstract}


%

\section{Introduction}
Diffusion models (DMs)~\cite{ho2020denoising,sohl2015deep,song2020score} have emerged as a pivotal component in the realm of generative modeling, facilitating notable progress in domains including image generation~\cite{ramesh2022hierarchical,rombach2022high}, video synthesis~\cite{blattmann2023align,ho2022imagen}, 
and beyond. 
In particular, latent diffusion models (LDMs), such as Stable Diffusion~\cite{rombach2022high}, have exhibited exceptional capabilities for high-resolution text-to-image synthesis and are acting as fundamental building components for a wide spectrum of downstream applications~\cite{gal2022image,chen2023subject,mou2023t2i}.

However, it is widely recognized that DMs' iterative reverse sampling process leads to slow inference. 
One remediation is improving the solvers used for discretizing the reverse process~\cite{song2020ddim,lu2022dpm,lu2022dpm++}, but it is still hard for them to generate within a limited number of steps (e.g., $5$) due to the inevitable discretization errors. 
Another strategy involves distilling an ordinary differential equation (ODE) based generation process of a pre-trained DM into a shorter one~\cite{salimans2022progressive,meng2023distillation}, but the cost of such progressive distillation is routinely high.  
Alternatively, consistency distillation (CD) trains unified consistency models (CMs) to fit the consistency mappings characterized by the diffusion ODE for few-step generation~\cite{song2023consistency}.  
Latent consistency model (LCM)~\cite{luo2023latent} further applies CD to the latent space of a pre-trained autoencoder to enable high-resolution image generation.
However, its sample quality is poor within 2 sampling steps. 
Recently, InstaFlow~\cite{liu2023instaflow}, UFOGen~\cite{xu2024ufogen}, and ADD~\cite{sauer2023adversarial} have succeeded in faithfully generating high-resolution images in just 1-2 steps, but they share the limitation of failing to trade additional sampling steps for improved outcomes. 

One-step samplers are hard to generate satisfactory outputs, and numerous works~\cite{luo2023latent, salimans2022progressive, meng2023distillation} improve their one-step performance via additional steps with acceptable cost (e.g., 4 steps). Our objective is to strike a unified model capable of generating high-quality outputs with 2-4 steps. 
This is based on the fact that over-emphasizing the one-step generation capacity would unavoidably bias the DM, hence weakening the multi-step generation capacity. Sampling with 2-4 steps does not substantially increase the practical cost compared to 1 step but is likely to improve the upper bound of the sampling quality significantly. 
We base our solution on CMs because they enjoy the cost-quality trade-off by alternating denoising and noise injection at inference time. 
Yet, current CD approaches have not fully unleashed the potential of the teacher, considering that for a well-trained DM, ODE-based solvers usually underperform stochastic differential equation (SDE) ones with adequate sampling steps ~\cite{xu2023restart,karras2022elucidating,gonzalez2023seeds}.
This is verified by results in \cref{table:solver}. 
Empowered by these, we aim to develop  \textbf{S}tochastic \textbf{Co}nsis\textbf{t}ency Dis\textbf{t}illation (SCott), to combine CD with SDE solvers to accelerate the sampling of high-resolution images.

\begin{figure*}[t]
\centering
\vspace{-2ex}
\includegraphics[width=1\textwidth]{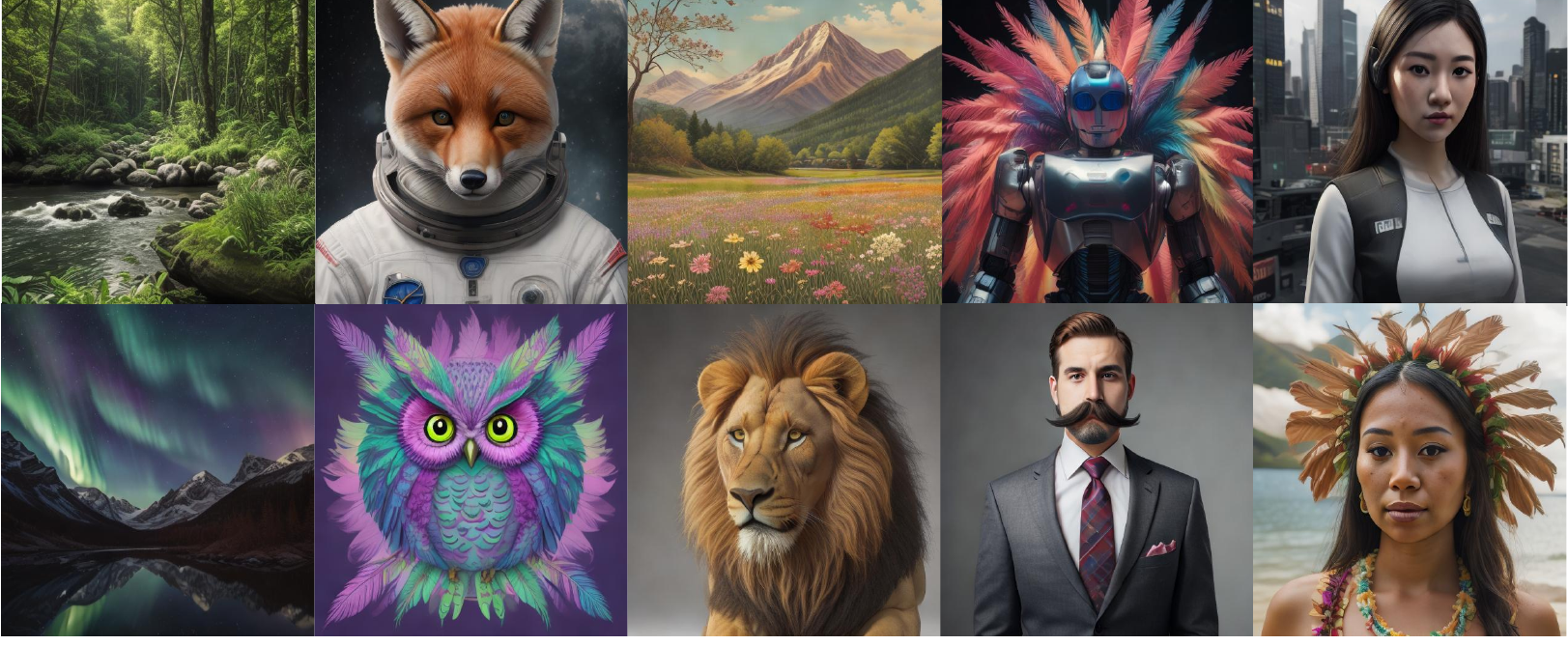}
\vspace{-4ex}
\caption{$512 \times 512$ resolution images generated by SCott using 2 sampling steps. SCott is trained based on Realistic-Vision-v51.}
\label{fig: main}
\vspace{-3ex}
\end{figure*}

By \cite{song2023consistency}, CMs are originally defined and learned based on ODE solvers. 
Naturally, a straightforward adoption of regular SDE solvers suffers from low training stability and poor convergence. 
To address this, we provide a theoretical justification for using SDE solvers for CD and empirically identify several critical factors that render SCott workable. 
On one hand, we find it necessary to keep the injected noise in SDE solvers at a moderate intensity
to stabilize training while enjoying benefits from the stochasticity inherent in SDE. 
On the other hand, we find it vital to extend the one-step sampling strategy employed in vanilla CD to a multi-step one to further diminish the discretization errors.
The multi-step SDE solver also aids in correcting the accumulated errors in the sampling path, thanks to the injection of random noise~\cite{xu2023restart}. 
With these, we obtain a stronger and more versatile teacher for CD. 

Surprisingly, we find the uncertainty in SDE can lead to diverse generations for the distilled student model (see Appendix 5). From another perspective, the SDE solver implicitly makes stochastic data augmentation for CD, which aids in mode coverage. To mitigate the imperfect consistency constraints provided by L2 loss in CD, an adversarial learning loss is incorporated to correct student output, further boosting the sample quality at 1-4 sampling steps. Extensive experiments validate the efficacy of SCott in generating high-quality images with conspicuous details. On MSCOCO-2017 5K validation dataset with a Stable Diffusion-V1.5 (SD1.5)~\cite{rombach2022high} teacher, our 2-step method achieves an FID~\cite{heusel2017gans} of 21.9, surpassing the previous state-of-the-art, e.g., 2-step LCM~\cite{luo2023latent} (30.4) and 1-step InstaFlow~\cite{liu2023instaflow} (23.4) and UFOGen~\cite{xu2024ufogen} (22.5). 
Besides, SCott can smoothly improve the sample quality with increasing sampling steps. 
At a 4-step inference, SCott consistently outperforms LCM, InstaFlow, and UFOGen with $2-4$ steps in terms of both FID and CLIP Score~\cite{hessel2021clipscore}. 
On MJHQ-5K validation dataset with Realistic-Vision-v51 (RV5.1)\footnote{\url{https://huggingface.co/stablediffusionapi/realistic-vision-v51}.} teacher, 2-step SCott surpasses LCM by a remarkable margin (24.9 v.s. 37.2 in FID, 0.301 v.s. 0.296 in CLIP Score). 
For the Coverage metric measuring sample diversity~\cite{naeem2020reliable}, 2-step SCott obtains 0.1232 and 0.1778 gains over 2-step LCM on MSCOCO-2017 5K and MJHQ-5K, respectively.

We summarize our contributions as follows:
\begin{itemize}
    \item We propose SCott, which accelerates diffusion models to generate high-quality outputs with $1-2$ steps while maintaining the capability for further improvement via more inference time within $4$ steps. 
   \item We provide theoretical convergence analysis for SCott and explore crucial factors that render SCott workable. Furthermore, we integrate adversarial learning objectives into SCott to improve the few-step sample quality.
    \item SCott achieves 1) a state-of-the-art FID  of 21.9 in 2 steps, surpassing competing baselines such as 1-step InstaFlow (23.4), 2-step LCM (30.4), and 4-step UFOGen (22.1), and 2) much higher sample diversity, reflected by the Coverage metric, than LCM (0.9114 v.s. 0.7882). 
\end{itemize}

\section{Related Works}
\subsubsection{Diffusion Models} Diffusion models~\cite{sohl2015deep,ho2020denoising,song2020score,song2021maximum,song2020improved,karras2022elucidating,dhariwal2021diffusion} progressively perturb data to Gaussian noise and are trained to denoise the noise-corrupted data. During inference, diffusion models create samples from Gaussian distribution by reversing the noising process. 
They have achieved unprecedented success  in text-to-image generation~\cite{saharia2022photorealistic,ramesh2022hierarchical}, image inpainting~\cite{lugmayr2022repaint},   and image editing~\cite{meng2021sdedit, chen2023controlstyle}. 
To effectively improve the
sample quality of conditioned diffusion models,  classifier-free guidance (CFG)~\cite{ho2022classifier} technique is proposed without extra network training.

\subsubsection{Diffusion Acceleration} One of the primary challenges that hinder the practical adoption of diffusion models is the issue of sampling speed due to multiple iterations. Several approaches have been proposed to enhance the sampling efficiency of diffusion models. One type of methods concentrate  on training-free numerical solvers~\cite{song2020ddim,lu2022dpm,lu2022dpm++}, such as 
Denoising Diffusion  Implicit Model (DDIM)~\cite{song2020ddim}  and DPM++~\cite{lu2022dpm++}.
Some researchers explore the approaches of knowledge distillation to compress sampling steps. Progressive Distillation (PD)~\cite{salimans2022progressive} and Classifire-aware Distillation (CAD)~\cite{meng2023distillation} are designed to reduce sampling steps to below 10 steps via multi-stage step distillation.  
LCM~\cite{luo2023latent} extends CM~\cite{song2023consistency}  to text-to-image generation.
However, these methods synthesize blurry samples below four steps. 
Recently, InstaFlow~\cite{liu2023instaflow}, SwiftBrush~\cite{nguyen2024swiftbrush}, and DMD~\cite{yin2024one} achieve one-step generation in high-resolution text-to-image generation faithfully.
InstaFlow proposes a one-step sampling model for text-to-image generation by combining DMs and Rectified Flow~\cite{liu2022flow}.  SwiftBrush adopts variational score distillation (VSD)~\cite{wang2024prolificdreamer} to distill a one-step student. DMD employs distribution matching distillation to enhance the realism of the one-step generator. Nevertheless, they are unable to extend their sampler to multiple steps, and the synthesized images are not satisfactory enough with a single step only.
Consistency trajectory models (CTM)~\cite{kim2023consistency} and Diff-Instruct~\cite{luo2023diff} distill a pre-trained DM into a single-step generator, but their performance on large-scale text-to-image generation is unclear. Motivated by CMs and knowledge distillation, we propose Stochastic Consistency Distillation to generate high-quality images within few steps. Our method is able to not only produce high-quality samples with a 2-step sampler but also improve model performance with increasing steps. Furthermore, due to the introduced stochasticity in SDE solvers, our method exhibits better sample diversity.

\subsubsection{Diffusion GANs} 
With GANs as a core technique,  
UFOGen~\cite{xu2024ufogen}  proposes a one-step diffusion GAN for text-to-image generation.   However, the image quality can not be significantly improved or even gets worse when increasing sampling steps for UFOGen. Adversarial Diffusion  Distillation~\cite{sauer2023adversarial} distills pre-trained SD models by  GANs and score distillation~\cite{poole2022dreamfusion}, achieving one-step generation.  However, the training source is not described in the paper, e.g., training time and training data, and hence we can not make a fair comparison with it. Different from UFOGen and ADD which adopt adversarial learning as their core component, we propose to leverage adversarial loss to strengthen consistency constraints in SCott, producing high-quality images at few-step sampling. Consequently, our method inherits the property of CM which heightens text-to-image alignment and image sharpness with increasing sampling steps,   and the image quality is further enhanced by  GAN at few-step inference.

\section{Preliminary}

Let $\vx \in \R^k$ denotes a sample from the data distribution $\pdata(\vx)$ and $p(\vz_t)=\int \pdata(\vx) \mathcal{N}(\vz_t; \alpha_t \vx, \sigma_t^2 \mathbf{I}) \mathrm{d}\vx, \forall t\in[0, T]$ the marginal distribution specified by the forward diffusion process.
$\alpha_t$ and $\sigma_t$ are positive real-valued functions defining the diffusion schedule so that $p(\vz_0)=\pdata(\vx)$ and $p(\vz_T) \approx \mathcal{N}(\vz_T; \mathbf{0}, \tilde{\sigma}^2\mathbf{I})$ for some $\tilde{\sigma}$. 
A DM $\vepsilon_{\theta}(\cdot, t): \R^k \to \R^k$ is trained under score matching principles~\cite{vincent2011connection,song2019generative,ho2020denoising} for reversing the diffusion process.\footnote{The paper focuses on the $\epsilon$-prediction type of DMs. Other parameterizations are equivalent in theory~\cite{salimans2022progressive}.} 

According to the SDE/ODE explanation of the reverse process of DMs~\cite{song2020score}, we can obtain an approximate sample of $\pdata(\vx)$ by drawing a Gaussian noise $\vz_T \sim p(\vz_T)$ and then invoking a numerical SDE/ODE solver to discretize the reverse process.
Let $\hat{\vz}_{t}$ denotes the solution at timestep $t$, originating from $\vz_T$ based on a solver and model $\vepsilon_{\theta}(\cdot, t)$, and then $\hat{\vz}_{0}$ represents the sampled data. 


\subsubsection{Consistency Distillation} 
The mentioned solving process usually hinges on tens or hundreds of steps, causing significant inference overheads for practical application. 
A promising solution is to perform consistency distillation (CD)~\cite{song2023consistency,song2023improved} of DMs, yielding a student model $\vf_\vtheta(\cdot, t): \R^k \to \R^k$ which enjoys a shortened sampling procedure. 

Concretely, CD defines a novel consistency function $\mathbf{f} : (\mathbf{z}_t, t)  \mapsto \mathbf{z}_\tau$ for all $ t \in [\tau, T]$, where $\tau$ is the boundary time near 0. CD parameterize 
consistency function with networks $\vf_\vtheta(\vz_t, t)$, and minimizes the following loss for training
\begin{equation}
\label{eq:cd}
    \small
    \min_{\vtheta} \mathcal{L}_{CD}(\vtheta) = \E_{n, \vz_{t_n}} \Big[ \lambda(t_n) \big\Vert \vf_\vtheta(\vz_{t_n}, {t_n}) - \vf_{{\vtheta}^-}(\hat{\vz}_{t_{m}}, t_m) \big\Vert_2^2\Big],
\end{equation}
where $t$ in $[\tau, T]$ are uniformly discretized into $N$ time points
 with $t_1 = \tau<t_2<\cdots< t_N=T$ and 
$m \in \{1, \dots, n-1\}$ is a hyper-parameter ~\cite{song2023consistency, luo2023latent}, $\lambda(\cdot)$ refers to another positive weighting function, ${\vtheta}^-$ denotes the exponential moving average (EMA) of ${\vtheta}$.
The state $\hat{\vz}_{t_{m}}$ represents an intermediate state, starting from $\vz_{t_n}$ and obtained with the teacher model $\vepsilon_{\theta}(\cdot, t)$ and a one-step DDIM solver~\cite{song2020ddim}, expressed as
\begin{equation}
\label{eq: ddim-sde}
\begin{split}
    \hat{\vz}_{t_{m}}=& 
        \frac{\sqrt{\alpha_{t_m}}}{ \sqrt{\alpha_{t_n}}}( \vz_{t_n} - \sqrt{1 - \alpha_{t_n}}\cdot \vepsilon_{\theta}(\vz_{t_n},t_n))\\
        &\quad\quad\quad\quad+ \sqrt{1 - \alpha_{t_{m}} - \sigma_{t_n}^2} \cdot \vepsilon_{\theta}(\vz_{t_n}, t_n)  + \sigma_{t_n} \vepsilon.
\end{split}
\end{equation}
CD typically sets $\sigma_{t_n}=0$ to keep $\hat{\vz}_{t_{m}}$ and $\vz_{t_n}$ on the same ODE trajectory. 

Although we confine the distance measure in \cref{eq:cd} to the squared $\ell_2$ distance, the $\ell_1$ distance also applies here. 
We do not consider the Learned Perceptual Image Patch Similarity~\citep[LPIPS,][]{zhang2018unreasonable} because it can lead to inflated FID scores~\cite{song2023improved}. 
The trained $\vf_\vtheta$ allows for one-step or multi-step sampling for generating new data~\cite{song2023consistency}.

\textbf{Limitation. } A primary limitation of CD from the viewpoint of distillation is that it has not fully unleashed the potential of the model $\vepsilon_{\theta}(\cdot, t)$: it utilizes one-step DDIM to sample a preceding state of the current one to serve as the teacher for distillation. 
Yet, for a well-trained DM, ODE-based solvers unusually underperform SDE ones with adequate sampling steps~\cite{gonzalez2023seeds}. Naturally, we ask if we can build an SDE-based teacher for improved CD.

\subsubsection{SDE Solvers for Diffusion Models}
The SDE formulation of the reverse-time diffusion process~\cite{song2020score} takes the form of
\begin{equation}
    \label{eq: re-sde}
    \begin{split}
    \mathrm{d} {\vx_t} = \underbrace{[f_t\vx_t + \frac{g^2_t }{2\sigma_t} \vepsilon_{\theta}(\vx_t, t)]  \mathrm{d}t }_\text{Probabilistic ODE} +
    \underbrace{ \frac{g^2_t }{2\sigma_t} \vepsilon_{\theta}(\vx_t, t) \mathrm{d}t  + g_t \mathrm{d}\bar{\vw}_t,}_\text{Langevin process}
    \end{split}
\end{equation}
where $\bar{\vw}_t $ denotes the standard Wiener process in reverse time, and
\begin{equation}
    f_t := \frac{\mathrm{d log} \alpha_t}{\mathrm{d}t},\quad g^2_t := \frac{\mathrm{d}\sigma_t^2}{\mathrm{d}t} - 2\frac{\mathrm{d log}\alpha_t}{\mathrm{d}t}\sigma_t^2.
\end{equation}
We can discretize \cref{eq: re-sde} over time to get an approximated solution with various SDE solvers~\cite{gonzalez2023seeds, cui2023elucidating, zhang2022fast}.

\textbf{Benefits of SDE Solvers.}
According to the theoretical analysis of~\cite{xu2023restart}, the divergence between $p(\vz_0)$ and the sample distribution $p(\hat{\vz}_0)$ stems from the discretization errors along the sampling trajectory and the approximation error between the model and the ground-truth score function~\cite{song2020score}. 
When the number of function evaluations (NFE) is small, SDE solvers exhibit larger discretization errors than ODE ones---$O(\delta ^{\frac{3}{2}})$ v.s. $O(\delta ^2)$ with $\delta$ as the step size for discretization. 
On the contrary, the discretization errors become less significant as $\delta$ shrinks and the approximation errors dominate, thus SDE solvers achieve higher sample quality than ODE ones thanks to the injection of noise for correcting previous approximation errors~\cite{karras2022elucidating}.

\begin{table}[t]
\vskip 0.05in
\begin{center}
\begin{small}
\begin{sc}
\begin{tabular}{lcccr}
\toprule
Solver & Step&FID  $\downarrow$ & CS $\uparrow$  & CR$\uparrow$ \\
\midrule
DDIM& 50 &20.3&0.318&0.9130\\
DPM++& 25 &20.3&0.318&0.9118\\
ER-SDE 5& 50 &\textbf{20.2}&\textbf{0.320}&\textbf{0.9172}\\
\bottomrule
\end{tabular}
\end{sc}
\end{small}
\end{center}
\caption{Performance comparison of typical ODE solvers (DDIM and DPM++) and SDE ones (ER-SDE 5) on MSCOCO-2017 5K with the SD1.5 model. 
CS denotes CLIP Score to measure text-to-image consistency and CR represents the Coverage metric to assess sample diversity.}
\label{table:solver}
\vskip -0.2in
\end{table}

We validate these arguments with results in \cref{table:solver}, where the ER-SDE 5 solver~\cite{cui2023elucidating} outperforms ODE solvers DDIM and DPM++~\cite{lu2022dpm++} in terms of both sample quality and diversity with 50 sampling steps. Besides, the noise injected in SDE solvers implicitly makes stochastic data augmentation for CD. With such observations, we are motivated to explore the possibility of combining SDE solvers and CD, aiming at acquiring stronger teachers for CD.

\section{Methodology}
Previous diffusion acceleration methods either fail to generate high-quality samples within $4$ steps~\cite{lu2022dpm, luo2023latent} or are limited to one-step generation, unable to trade additional sampling steps for improved outcomes~\cite{liu2023instaflow, xu2024ufogen}. 
Our goal is to generate high-quality samples within 2-4 steps to strike a favorable tradeoff between efficiency and efficacy.
We observe that CMs enjoy the trade-off by alternating denoising and noise injection at inference time. However, CD approaches haven't unleashed the potential power of the pre-trained teacher DMs, considering ODE solvers underperform SDE ones when the NFE is large. 
We first lay out a justification for the feasibility of combining CD with SDE solvers. 
Then, we identify several critical factors to make SCott workable for high-resolution text-to-image generation. 
We also introduce GAN loss to further enhance the sampling quality at rare steps. 
Our method is summarized in~\cref{fig:arch}. 
\begin{figure}[t]
\vspace{-2ex}
    \centering
     \includegraphics[width=\linewidth]{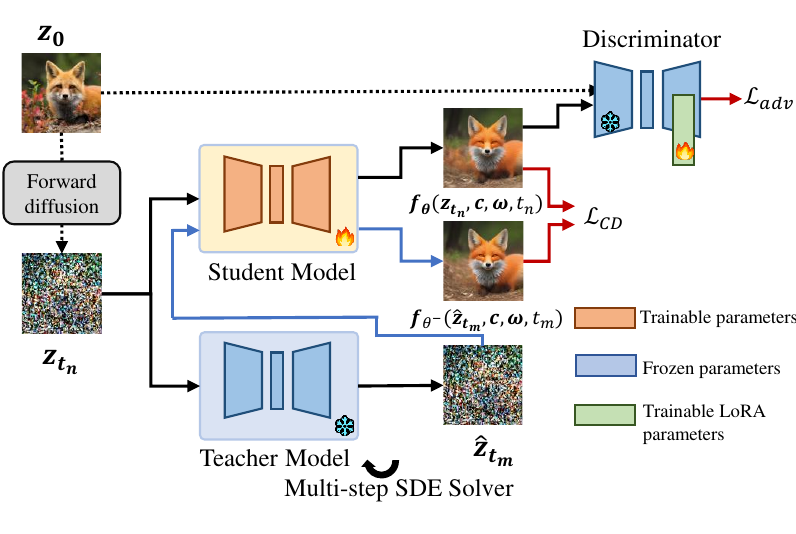}
    \vspace{-4.5ex}
    \caption{\textbf{Overview of SCott}. SCott distills a pre-trained teacher DM into a student one for accelerated sampling. Compared to the vanilla consistency distillation approach, we introduce a multi-step SDE solver to establish a stronger and more versatile teacher. 
    We train the student model with CD loss using SDE solvers.
    Additionally, we include an adversarial learning loss to correct student output, boosting the sample quality with rare sampling steps. Note that we omit the EMA operation for the teacher for brevity. }
    \label{fig:arch}
\end{figure}
\vskip -0.05 in
\setlength{\parskip}{0.2cm plus4mm minus3mm}

\vspace{-1ex}
\subsection{Justification of Using SDE Solvers for CD}
\vspace{-0.1cm}
Despite our motivation for combining SDE solvers and CD being sensible from the perspective of model distillation, a major concern arises as to whether it is reasonable to do so because CD is originally defined on ODE trajectories. 
In this section, we provide theoretical justification for this. 
 Following~\cite{song2023consistency}, we provide the convergence proof of using SDE solvers in CD.
\begin{theorem}
Let $ \Delta t:= \max\limits_{n\in [1,N]} {|t_{n+1} - t_{n}|} $ where 
$t\in [\tau, T]$. Assume $\vf_{\theta}(\cdot , \cdot)$ is Lipschitz in $\vx$ with constant $L_1$. Denote  $\vf(\cdot , \cdot) $ the consistency function of the SDE  defined in \cref{eq: re-sde}.
Assume the SDE solver $\Phi_{SDE}$ has  a local error bound of $O((\Delta t))^{p+1}$ with $p\ge 1$.Then, if $\mathcal{L}_{CD}^{N}(\vtheta, \Phi_{SDE}) = 0$, we have : 
\begin{center}
    $ \sup\limits_{n,\vx} ||\vf_\theta(\vx_{t_n}, t_n) -   \vf(\vx_{t_n}, t_n) ||_2 = O((\Delta t)^p) $ .
\end{center}
\end{theorem}
\vspace{-0.5ex}
\begin{proof}
We refer to Appendix 1 in our supplementary materials for the full proof.

\end{proof}
\vspace{-0.5ex}
Additionally, we also present an empirical investigation to demonstrate the feasibility of using SDE solvers for CD (Please refer to Appendix 3).

\subsection{Stochastic Consistency Distillation (SCott)} 

\label{stoc distillation}
For large-scale modeling on high-resolution natural images, we, however, observe that directly utilizing an SDE solver for CD leads to training instability and poor convergence. 
We address this issue by introducing multiple crucial modifications to straightforward implementations. 
\vspace{-0.2cm}
\subsubsection{Controlling the Level of Noise in SDE Solvers }
It is a natural idea to control the level of random noise injected in SDE solvers to stabilize training. 
Taking the DDIM solver in \cref{eq: ddim-sde} for example,
the $\sigma_{t_n}$ controls the intensity of injected noise and typically we can introduce a coefficient $\eta$ for scaling it~\cite{song2020ddim}:
\begin{equation}
    \sigma_{t_n}(\eta) := \eta \sqrt{(1 - \alpha_{t_{n-1}})/(1 - \alpha_{t_n})} \sqrt{1 - \alpha_{t_n}/\alpha_{t_{n-1}}}.
\end{equation}
Increasing $\eta$ from 0 to a positive value results in a set of SDE solvers with increasingly intensive noise, starting from an ODE one. 
We empirically experiment with the solvers for SCott and present the results in \cref{table:sde}. 
As shown, we indeed need to select a reasonable noise level to enjoy the benefits of SDE solvers for error correction while avoiding introducing excessive variance. 
Following this insight, we include the more advanced ER-SDE solver~\cite{cui2023elucidating} to SCott.
The novel noise scaling function in the ER-SDE solver also provides a simple way to control the intensity of noise, which aligns with our requirements for SDE solvers. 
\vspace{-0.2cm}
\subsubsection{Multi-step Sampling }
The discretization errors of SDE solvers are larger than those of ODE solvers~\cite{xu2023restart}. 
A simple remediation to this is to decrease the step size, but doing so leads to slow convergence and degraded results for SCott, consistent with the results in~\cite{luo2023latent}. 
To address this, we propose a multi-step sampling strategy for SCott. 
Specifically, for a sample $\vz_{t_n}$ at the time $t_n$ on the SDE trajectory, to obtain the estimated state $\hat{\vz}_{t_m}$ for a preceding timestep $t_m$, we split the time interval $t_n - t_m$ into $h$ intervals, namely, sampling with a step size of $\frac{t_n -t_m}{h}$, 
as illustrated in \cref{fig:multi_step_sde}. 
We denote the solution as $\hat{\vz}_{t_m} = \Phi_{SDE}(\vepsilon_{\theta},\vz_{t_n}, t_n, t_m,  h)$. 
\begin{figure}[H]
    \centering
    \includegraphics[width=.8\linewidth]{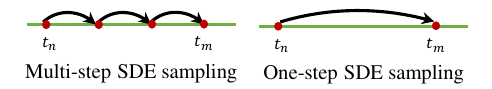}
    \vspace{-0.2cm}
    \caption{Multi-step SDE solver sampling}
    \label{fig:multi_step_sde}
\end{figure}
\vspace{-0.2in}

Through multi-step sampling, we can reduce the discretization errors in SDE solvers as the interval shrinks.
Besides, the multi-step SDE solver aids in correcting the accumulated errors in the sampling path with the injection of random noise.
Empirically, we set $t_m = t_n - 24$ and $h=3$.

\vspace{-0.2cm}
\subsubsection{Introducing GAN into SCott} \label{GAN accelerate}
 For high-resolution text-to-image generation, considering the high data dimensionality and complex data distribution,  simply using L2 loss fails to capture data discrepancy precisely, thus providing imperfect consistency constraints. 
Suggested by CTM~\cite{kim2023consistency} which leverages GAN to improve trajectory estimation in the process of distillation, we integrate GAN into SCott to mitigate the restriction of L2 loss. However, it is non-trivial to directly employ CTM's GAN. Firstly,
CTM's GAN performs on the pixel level while SCott performs on the latent space. For high-resolution image generation, employing a GAN that operates directly on pixel values would significantly increase the training cost. Secondly, SCott is designed to accelerate text-to-image tasks, the interaction between textual information and visual representation is not considered in the design of CTM's GAN. To address these issues, we design a low-rank adaptation (LoRA) \cite{hu2021lora} discriminator with time and text projection, which renders adversarial training successfully in SCott.

\vspace{-0.2cm}
\subsubsection{LoRA Discriminator}
To accelerate training, the discriminator is initialized by the pre-trained U-Net in SD1.5. 
However, we find training the whole U-Net is expensive and unstable. Therefore, we freeze the encoder of the U-Net and only train the LoRA parameters of the decoder of the network. The frozen encoder of the U-Net extracts rich latent representations from the input and conditions, while the LoRA decoder acts as discriminator heads to distinguish between real images and generated images. 
 By employing this approach, only $3.6\%$ of the parameters in the U-Net are updated, which not only contributes to computational efficiency but also increases the overall stability of the adversarial training.

\vspace{-0.2cm}
\subsubsection{Conditional Discriminator}
Since SCott predicts the clean image $\hat{\vz}_0$, we can directly perform adversarial learning between real image $\vz_0$ and generated image $\hat{\vz}_0$.  However, we find such a naive design makes training unstable for text-to-image tasks. To surmount this challenge,  the timestep $t_n$ and text $\vc$ are propagated into the discriminator as conditions.

Following~\cite{sauer2023stylegan}, we train our discriminator $D_{\vphi}$ with hinge loss:
\begin{equation}
    \begin{split}
        &\mathcal{L}_{adv}(\vtheta, \vphi)  = \mathbb{E}_{\vz_0, \vc} \Big[ \max(0, 1 - D_{\vphi}(\vz_0, \vc, 0))  \Big] \\
        &\quad + \mathbb{E}_{\vz_{t_n}, \vc, t_n} \Big[ \max(0, 1 + D_{\vphi}(\vf_{\vtheta}(\vz_{t_n}, \vc, \omega,t_n), \vc, t_n))  \Big],
    \end{split}
    \label{eq: gan}
\end{equation}
where $\vf_{\vtheta}$ denotes the SCott model while $D_{\vphi}$ represents our proposed LoRA Discriminator.  Overall, SCott is trained with the following objective:
\begin{equation}
    \begin{split}
    \mathcal{L}_{SCott}(\vtheta, \vphi)  = \mathcal{L}_{CD}(\vtheta, \vtheta^{-}; \Phi_{SDE} ) + \lambda_{adv}\mathcal{L}_{adv}(\vtheta, \vphi).
    \end{split}
    \label{eq: scott}
\end{equation}
In practice, we set $\lambda_{adv}=0.4$ to control the strength of the discriminator for refining the outputs of $\vf_\vtheta$. 
To save GPU memory, the training is performed in latent space. We find the introduction of GAN loss renders SCott to produce more realistic outputs, particularly at rare steps. 
We present more details of the training process for SCott in Appendix 4.

\begin{table*}[ht]
\setlength{\tabcolsep}{3mm}
\vspace{-0.3cm}
\begin{center}
\begin{sc}
\begin{tabular}{ll@{\hskip3pt}l@{\hskip3pt}cccccr}
\toprule
Method & \multicolumn{2}{c}{References}&Step & Time (s) & FID  $\downarrow$ & CS $\uparrow$  & CR $\uparrow$   \\
\midrule
DPM++~\cite{lu2022dpm++}& ArXiv &(2022) &25&0.88&20.1&0.318&0.9118\\
DDIM~\cite{song2020ddim}& ICLR &(2021) &50&$-$&20.3&0.318&0.9130\\
\hline
PD~\cite{salimans2022progressive}& ICLR &(2022) &4&0.21&26.4&0.300&$-$ \\
CAD~\cite{meng2023distillation}& CVPR &(2023) &8&0.34&24.2&0.300&$-$\\
LCM~\cite{luo2023latent}& ArXiv &(2023) &2&0.10&30.4&0.293&0.7882 \\
INSTAFLOW-0.9B~\cite{liu2023instaflow}& ICLR &(2024) &1&0.09&23.4&0.304&$-$ \\
INSTAFLOW-1.7B~\cite{liu2023instaflow}& ICLR &(2024) &1&0.12&22.4&0.309&$-$ \\
UFOGEN~\cite{xu2024ufogen}& CVPR &(2024) &1&0.09&22.5&0.311&$-$ \\
UFOGEN~\cite{xu2024ufogen}& CVPR &(2024)&4&$-$&22.1&0.307&$-$ \\
\hline
SCott (Ours)&&&1&0.09&26.8&0.295&0.8254\\
SCott (Ours)&&&2&0.12&21.9&0.310&0.9114\\
SCott (Ours)&&&4&0.18&\textbf{21.8}&\textbf{0.311}&\textbf{0.9145}\\
\bottomrule
\end{tabular}

\end{sc}
\end{center}
\caption{Comparisons with the state-of-the-art methods on MSCOCO-2017 5K in terms of  FID,  CS, CR with backbone SD1.5.}
\label{table:coco}
\vspace{-0.4cm}
\end{table*}

\section{Experiments}
In this section, we elaborate on the experimental results of our proposed SCott model for the text-to-image generation task.  We start with the comparison with other works, and then ablate components of SCott, highlighting
the effectiveness of our proposed components.

\subsubsection{Implementation Details}
We use LAION-Aesthetics-6+ dataset~\cite{schuhmann2022laion}. We train SCott with 4 A100 GPUs and a batch size of 40 for 40K iterations. The learning rate is 8e-6 for SCott and 2e-5 for the discriminator. (See more experimental setups in Appendix 4 )
\vspace{-0.1cm}
\begin{table}[t]
\footnotesize

\begin{center}
\begin{sc}
\resizebox{\linewidth}{!}{%
\begin{tabular}{lcr}
\toprule
Method&  Step  & FID  $\downarrow$   \\
\midrule
SWIFTBRUSH~\cite{nguyen2024swiftbrush}  & 1  &  16.67 \\ 
UFOGEN~\cite{xu2024ufogen} & 1 & 12.78 \\ 
INSTAFLOW-0.9B~\cite{xu2024ufogen}  & 1 & 13.10 \\ 
DMD~\cite{yin2024one}  & 1 & 11.49  \\ 
SCott (Ours) &  2  & 11.13 \\ 
SCott (Ours) & 4  & \textbf{10.68 }\\ 
\bottomrule
\end{tabular}
}
\end{sc}
\end{center}
\caption{Comparisons with the state-of-the-art methods on MSCOCO-2014 30K in terms of  FID with backbone SD1.5.}
\vspace{-0.2cm}
\label{table:coco 30k}
\end{table}

\subsubsection{Comparison on MSCOCO-2017 5K with SD1.5}
To kick-start the comparisons with the state-of-the-art methods, we evaluate the MSCOCO-2017 5K validation dataset~\cite{lin2014microsoft}. 
 Zero-shot FID and CLIP Score (CS) with ViT-g/14 backbone~\cite{radford2021learning} are exploited as objective metrics. To measure diversity, Coverage (CR)~\cite{naeem2020reliable}  is used as another metric. 
Table \ref{table:coco} summarizes the performance of our SCott and comparative methods consisting of PD~\cite{salimans2022progressive}, CAD~\cite{meng2023distillation}, LCM~\cite{luo2023latent}, InstaFlow~\cite{liu2023instaflow}, UFOGen~\cite{xu2024ufogen}.  
Since PD~\cite{salimans2022progressive}, CAD~\cite{meng2023distillation}, InstaFlow~\cite{liu2023instaflow}, and UFOGen~\cite{xu2024ufogen} do not list CR value, we leave them $-$. Time denotes inference time on a single A100. LCM is implemented according to the official in our setting.

\begin{table}[ht]
\setlength{\tabcolsep}{2.5mm}
\begin{center}
\begin{sc}
\begin{tabular}{lcccccr}
\toprule
Method& Step & FID  $\downarrow$ & CS $\uparrow$  & CR $\uparrow$  \\
\midrule
DPM++&25&22.1&0.320&0.8664 \\
DDIM&50&22.9&0.320&0.8536 \\
\midrule
LCM&2&37.2&0.296&0.7016 \\
SCott (Ours)&2&\textbf{24.9}&\textbf{0.301}&\textbf{0.8794}\\
\bottomrule
\end{tabular}
\end{sc}
\end{center}
\caption{Comparisons with the state-of-the-art methods on MJHQ-5K in terms of  FID,  CS, CR, with backbone RV5.1.}
\label{mjhq}
\vskip -0.1in
\end{table}

Among all the methods, our 2-step 
SCott presents superior FID and CS values than 4-step PD, 8-step CAD,  1-step InstaFlow-0.9B, 2-step LCM. 2-step SCott achieves better FID than 1-step UFOGen and comparable CS. It is impressive to see that our 2-step SCott beats InstaFLow-1.7B which doubles the parameter size. These results demonstrate our SCott significantly enhances the quality of the generated images while reducing inference steps. The improvements lie in that the proposed $\mathcal{L_{CD}}$ in SCott is able to reduce the sampling steps remarkably and adversarial loss further improves image quality. Interestingly, the CS decreases by 0.004 when the inference step increases from 1 to 4 for UFOGen. Meanwhile, by changing the inference step from 1 to 4,  all the metrics are improved for our SCott, and 4-step SCott outperforms  4-step UFOGen regarding both FID and CS. The result indicates the improvement is limited for UFOGen when increasing steps while our method is much more powerful in enhancing image quality with additional steps in a flexible manner, which is meaningful in the scenarios requiring high-quality images with more affordable computational budgets.    
The proposed SCott obtains higher CR  than LCM, and even outperforms the ODE solvers, 25-step DPM++, and 50-step DDIM, indicating our SCott successfully increases sample diversity due to the introduced randomness in SDE solvers.

We present the qualitative comparisons involving InstaFlow and LCM (See Figure 5 in Appendix 5). Notably, 2-step SCott gains significant improvements over 2-step LCM and 1-step InstaFlow in terms of image quality and text-to-image alignment. We also observe that our generated images exhibit sharper textures and finer details, compared to the images generated by InstaFlow and LCM.

\noindent \textbf{Comparison on MSCOCO-2014 30K with SD1.5} 
For complete comparisons, we also benchmark our method against DMD, SwiftBrush, and UFOGen on MSCOCO-2014 30K~\cite{lin2014microsoft}. We follow the evaluation setup in DMD and use a CFG scale of 3 following DMD. As depicted in \cref{table:coco 30k}, 
 our 2-step SCott surpasses all the above methods in FID. Additionally, it's worth noting that these methods can not improve output with more steps, whereas SCott offers the feasibility to enhance the sample quality with additional steps. With 4-step inference, SCott achieves a superior FID of 10.68.
 These results validate the effectiveness of our approach in achieving higher-quality samples within 2 steps and improving performance by leveraging additional steps.

\begin{table}[t]
\setlength{\tabcolsep}{4mm}
\begin{center}
\begin{small}
\begin{sc}
\begin{tabular}{lcccr} 
\toprule
 SDE Solver & FID  $\downarrow$ & CS $\uparrow$  & CR $\uparrow$   \\
\midrule
DDIM ($\eta=0$)&27.4&0.296&0.8450\\
DDIM ($\eta=0.1$)&26.2& 0.297&0.8713 \\
DDIM ($\eta=0.2$)&25.2&0.299&0.8786 \\
DDIM ($\eta=0.3$)&27.7&0.298&0.8717 \\
DDIM ($\eta=0.6$)&29.7&0.298&0.8162 \\
ER-SDE 5 & \textbf{24.9} & \textbf{0.301} & \textbf{0.8794}\\
\bottomrule
\end{tabular}
\end{sc}
\end{small}
\end{center}
\vskip -0.1in
\caption{Comparisons of different solvers in CD on MJHQ-5K with 2-step inference. All models are based on RV5.1.   }
\label{table:sde}
\end{table}

\subsubsection{Comparison on MJHQ-5K with RV5.1} To better assess the quality of produced images, we further conduct experiments on MJHQ-5K, randomly selected from MJHQ-30K\footnote{\url{https://huggingface.co/playgroundai/playground-v2-1024px-aesthetic}.}, since it owns high image quality and image-text alignment, and the correlation between human preference and FID score on the MJHQ is verified by user study. The evaluation metrics of several methods with RV5.1 as a teacher are listed in Table \ref{mjhq}. The reason for distilling RV5.1 is that RV5.1 is much stronger than SD1.5 for text-to-image consistency, which can be found in Appendix 5. LCM is also implemented according to the official code. Since the training codes of PD, CAD, InstaFlow, and UFOGen are unavailable, their metrics are not included. The result again presents our SCott outperforms LCM by a large margin, because our SDE-based CD provides a stronger and more versatile teacher than ODE-based CD in LCM, and the student output is further refined to be real data by the proposed discriminator.
\vspace{-0.2cm}
 
\subsubsection{Ablation Studies}
\label{sec:ablate}
 To analyze the key components of our method, we make a thorough ablation study to verify the effectiveness of the proposed SCott.

\textbf{SDE Solver.}
Table \ref{table:sde} depicts the results using deterministic and stochastic solvers for CD. DDIM ($\eta$) represents DDIM solver with noise coefficient $\eta$, where $\eta$ denotes the hyper-parameter that controls the strength of random noise injected in Eq.\ref{eq: ddim-sde}. $\eta$ achieves an interpolation between the deterministic DDIM ($\eta=0$) and original DDPM ($\eta$ = 1). As observed in Table \ref{table:sde}, the SDE-based DDIM ($\eta=0.1,0.2$) surpasses ODE-based DDIM ($\eta=0$) for CD,  demonstrating the superiority of SDE solver for CD. This is because the injection of noise at moderate intensity in SDE solvers aids
in correcting estimated errors of the teacher model with multiple sampling steps, leading to a more powerful teacher. However, excessive noise intensity leads to poor convergence and degraded samples, as shown in DDIM ($\eta=0.6$).
These results indicate it is crucial to control the noise strength in SCott since large noise leads to low training stability. The 2-order ER-SDE 5 solver further enhances the performance, 
this is because 1. the 2-order ER-SDE 5 exhibits smaller discretization errors compared to the 1-order DDIM solver. 2. the noise function in ER-SDE 5 mitigates excessive noise, which leads to good convergence. \cref{discriminator ab} shows that SCott without GAN also surpasses the ODE solver-based LCM, further illustrating the benefits of using SDE solvers for CD.

\begin{table}[t]

\setlength{\tabcolsep}{4.5mm}
\begin{center}
\begin{small}
\begin{sc}
\begin{tabular}{lcccr}
\toprule
Solver Step & FID  $\downarrow$ & CS $\uparrow$  & CR $\uparrow$   \\
\midrule
1&27.4&0.299&0.8461 \\
2&25.4&0.300&0.8724\\
3&\textbf{24.9}&\textbf{0.301}&\textbf{0.8794} \\
\bottomrule
\end{tabular}
\end{sc}
\end{small}
\end{center}
\vskip -0.1in
\caption{Comparisons of choosing different SDE solver steps during training. All models are based on RV5.1 with 2-step inference on MJHQ-5K.  }
\label{table:length}
\end{table}

\textbf{Multi-step SDE Solver Sampling.}
As outlined in Table \ref{table:length}, we study the sampling steps in the process of estimating $\hat{\vz}_{t_m}$  given $\vz_{t_n}$. 
The results indicate that multi-step SDE solvers are superior to single-step solvers. The reason is that SDE solvers usually require multiple iterations to reach the correct destination, demonstrating the choice of SDE solver steps during CD is critical to make SCott successful. 


\begin{table}[t]

\begin{center}
\begin{small}
\begin{sc}
\begin{tabular}{lccr}
\toprule
Loss &  FID  $\downarrow$ & CS $\uparrow$  & CR $\uparrow$   \\
\midrule
SCott (\textnormal{Without GAN})   & 25.0 & 0.300 &  0.867  \\
LCM  & 30.4 & 0.293 & 0.788 \\ \hline
LCM + GAN    & 29.2  & 0.297  & 0.8160   \\ 
SCott (\textnormal{Full Discriminator})  & 23.5& 0.302 & 0.8880 \\ 
SCott (\textnormal{Without Condition})  & 26.1 & 0.297  &  0.8412\\ 
SCott  & \textbf{21.9 }& \textbf{0.310} & \textbf{0.9145}  \\ 
\bottomrule
\end{tabular}
\end{sc}
\end{small}
\end{center}
\caption{The performance comparisons of Discriminator on MSCOCO-2017 5k with backbone SD 1.5.}
\label{discriminator ab}
\end{table}

\textbf{Discriminator.} \cref{discriminator ab} illustrates our discriminator leads to gains with respect to FID, CS, and CR using different inference steps. We observe our LoRA discriminator achieves better sample qualities than the fully parameterized one. 
Compared with the discriminator without condition, SCott leverages both the time and text conditions. Such a design helps the discriminator better distinguish between real and generated samples.
Interestingly, we observe that our discriminator can also improve the outputs of LCM,
which indicates the generalization capability of our proposed discriminator.

\section{Conclusions}
In this paper, we propose stochastic consistency distillation (SCott), a novel approach for accelerating text-to-image diffusion models. SCott integrates SDE solvers into consistency distillation to unleash the potential of the teacher, implemented by controlling noise strength and sampling step of SDE solvers. Adversarial learning is further utilized to aid SCott in generating high-quality images in rare-step sampling. SCott is capable of yielding high-quality images with 2 steps only, surpassing 2-step LCM, 1-step InstaFlow, and 4-step UFOGen. Additionally, SCott consistently improves performance by increased inference cost within 4 steps and exhibits higher diversity than competing baselines.

\newpage


\bibliography{aaai25}

\newpage
\title{ SCott: Accelerating Diffusion Models with Stochastic Consistency Distillation \\ Supplementary Materials}

\setcounter{secnumdepth}{2}
\onecolumn
\begin{centering}
\textbf{{\Huge SCott: Accelerating Diffusion Models with Stochastic Consistency Distillation}} \\
\vspace{ 0.5 cm}
\textbf{
{\Huge Supplementary Materials}} \\
\vspace{1cm}
\end{centering}

\section{Theoretical justification for using SDE solvers in CD}
\begin{theorem}\label{proof sde}
Let $ \Delta t:= \max\limits_{n\in [1,N]} {|t_{n+1} - t_{n}|} $ where 
$t\in [\tau, T]$. Assume $\vf_{\theta}(\cdot , \cdot)$ is Lipschitz in $\vx$ with constant $L_1$. Denote  $\vf(\cdot , \cdot) $ the consistency function of the reverse SDE  defined in Equation 3 in the main part of our paper.
Assume the SDE solver $\Phi_{SDE}$ has  a local error bound of $O((t_{n+1} - t_{n}))^{p+1}$ with $p\ge 1$.
Then, if we have $\mathcal{L}_{\mathcal{KL}}^{N}(\vtheta, \Phi_{SDE}) = 0$, we have

\centering{$  \sup\limits_{n,\vx} ||\vf_\theta(\vx_{t_n}, t_n) -   \vf(\vx_{t_n}, t_n) ||_2 = O((\Delta t)^p) $ }.
\end{theorem}

\noindent The following lemma provides the convergence proof of using ODE solvers in CD (proof in \cite{song2023consistency}), which is crucial to our proof for \cref{proof sde}.

\begin{lemma}[Proof in \cite{song2023consistency} ] 
Let $ \Delta t:= \max\limits_{n\in [1,N]} {|t_{n+1} - t_{n}|} $ where 
$t\in [\tau, T]$. Assume $\vf_{\theta}(\cdot , \cdot)$ is Lipschitz in $\vx$ with constant $L_1$. Denote  $\vf(\cdot, \cdot) $ the consistency function of the PF ODE.
Assume the ODE solver $\Phi_{ODE}$ has  a local error bound of $O((t_{n+1} - t_{n}))^{p+1}$ with $p\ge 1$.
Then, if we have $\mathcal{L}_{CD}^{N}(\vtheta, \Phi_{ODE}) = 0$, then

\centering{$  \sup\limits_{n,\vx} ||\vf_\theta(\vx_{t_n}, t_n) -   \vf(\vx_{t_n}, t_n) ||_2 = O((\Delta t)^p) $ }.
\label{lemma1}
\end{lemma}

\begin{proof}[Proof of lemma \ref{lemma1} in \cite{song2023consistency}]
From  $\mathcal{L}_{CD}^{N}(\mathbf{\theta}, \Phi_{ODE}) = 0$, derive
\[
\vf_\theta(\mathbf{x_{t_{n+1}}}, t_{n+1}) \equiv \vf_\theta({\hat{\vx}_{t_{n}}^{\Phi_{ODE}}}, t_{n}).
\]
\\
Denote

\begin{align*}
\mathbf{e_{n+1}} & =  \mathbf{f_\theta}(\mathbf{x_{t_{n+1}}}, t_{n+1}) - \mathbf{f}(\mathbf{x_{t_{n+1}}}, t_{n+1})  \\
& = \mathbf{f_\theta}(\mathbf{\hat{x}_{t_n}^{\Phi_{ODE}}}, t_n) - \mathbf{f_\theta}(\mathbf{x_{t_{n}}}, t_{n}) + \mathbf{f_\theta}(\mathbf{x_{t_{n}}}, t_{n})  - \mathbf{f}(\mathbf{x_{t_{n+1}}}, t_{n+1})     \\
& = \mathbf{f_\theta}(\mathbf{\hat{x}_{t_n}^{\Phi_{ODE}}}, t_n) - \mathbf{f_\theta}(\mathbf{x_{t_{n}}}, t_{n})+ \mathbf{f_\theta}(\mathbf{x_{t_{n}}}, t_{n})  - \mathbf{f}(\mathbf{x_{t_{n}}}, t_{n})      \\
& =   \mathbf{f_\theta}(\mathbf{\hat{x}_{t_n}^{\Phi_{ODE}}}, t_n) - \mathbf{f_\theta}(\mathbf{x_{t_{n}}}, t_{n}) + \mathbf{e_{n}}.
\end{align*}
\\
Recall that $\mathbf{f_{\theta}(\cdot, \cdot)}$ has Lipschitz constant $L_1$, 
\begin{align*}
  || e_{n+1} ||_2 &\leq ||e_{n} ||_2 + L_1||\mathbf{\hat{x}_{t_n}^{\Phi_{ODE}}} -  \mathbf{x_{t_{n}}} ||_2 \\
  & = ||e_{n} ||_2 + L_1\cdot O((t_{n+1} - t_{n} )^{p+1}) \\
  & = ||e_{n} ||_2 +  O((t_{n+1} - t_{n} )^{p+1}).
\end{align*}
\\
With the boundary condition where \[ 
\mathbf{e_{1}}  =  \mathbf{f_\theta}(\mathbf{x_{t_{1}}}, t_{1}) - \mathbf{f}(\mathbf{x_{t_{1}}}, t_{1}) = \mathbf{x_{t_{1}}} - \mathbf{x_{t_{1}}} = 0,
\]
\\
then 
\begin{align*}
    ||e_n || &\leq ||e_1|| +  \sum\limits_{i=1}^{i=n-1} O((t_{i+1} - t_{i})^{p+1}) \\
            & \leq O(({\Delta t})^p) (T - \tau) = O(({\Delta t})^p).
\end{align*}

\noindent With this, the proof for lemma \ref{lemma1} is completed. 
\label{proof lemma 1}
\end{proof}

Then we provide the proof for convergence of using SDE solvers in CD. 
\begin{proof}[Proof of \cref{proof sde}] 
Consider the reverse SDE,
\begin{equation}\label{proof re sde}
    \textbf{SDE: \qquad} \\
     \mathrm{d} {\vx_t} = [f_t \vx_t + \frac{g^2_t }{\sigma_t} \mathbf{\epsilon_{\theta}}(\vx_t, t)]  \mathrm{d}t  + g_t \mathrm{d}\bar{\vw}_t . 
\end{equation}

\noindent Following \cite{lu2022dpm++}, we reparameterize the SDE with $\lambda_t =  \log \frac{\alpha_t}{\sigma_t}$, $g^2_t := \frac{\mathrm{d}\sigma_t^2}{\mathrm{d}t} - 2\frac{\mathrm{d log}\alpha_t}{\mathrm{d}t}\sigma_t^2 = -2\sigma^2_t \frac{\mathrm{d}\lambda_t}{\mathrm{d}t} $ and $g_t = \sigma_t \sqrt{-2\frac{\mathrm{d}\lambda_t}{\mathrm{d}t}}$, $\mathrm{d}\mathbf{w_{\lambda}} := \sqrt{-\frac{\mathrm{d}\lambda_t}{\mathrm{d}t}}  \mathrm{d}\bar{\mathbf{w}}_t $.

\noindent Derive the integration of \cref{proof re sde} from $\vx_{t_n}$ to $\vx_{t_m}$, which is

\begin{equation} \label{exact sde}
    \mathbf{x_{t_m}} = \underbrace{\frac{\alpha_{t_m}}{\alpha_{t_n}}\mathbf{x_{t_n}}}_{(1)} - 
    \underbrace{2\alpha_{t_m}\int_{\lambda_{t_n}}^{\lambda_{t_m}} e^{-\lambda}\mathbf{\hat{\epsilon}_{\theta}}(\mathbf{\hat{x}_{\lambda}}, \lambda) \mathrm{d}\lambda}_{(2)} +  
    \underbrace{\sigma_{t_n}\sqrt{e^{2(\lambda_{t_m} - \lambda_{t_n})} - 1} \vz_{t_n}}_{(3)} ,
\end{equation}
where $\vz_{t_n} \sim \mathcal{N}(\mathbf{0}, \mathbf{I}) $.
The first-order DPM SDE solver takes the form of
\begin{equation}\label{dpm sde}
    \mathbf{\hat{x}^{\Phi_{SDE}}_{t_m}} = \underbrace{\frac{\alpha_{t_m}}{\alpha_{t_n}} \mathbf{x_{t_n}}}_{(1)} - \underbrace{2\sigma_{t_m} (e^{\lambda_{t_m} - \lambda_{t_n}} - 1) \mathbf{\epsilon_{\theta}}(\mathbf{x_{t_n}}, t_n)}_{(2)} + \underbrace{\sigma_{t_n}\sqrt{e^{2(\lambda_{t_m} - \lambda_{t_n}}) -1} \mathbf{z_{t_m}}}_{(3)}.
\end{equation}

\noindent Thus we can bound the local error between \cref{dpm sde} and \cref{exact sde}.
Since the (1) term can be calculated exactly, the total errors $\Delta$ between \cref{dpm sde} and \cref{exact sde} can be divided into the errors $\Delta_2$ in (2) term and the errors $\Delta_3$ in (3) term

\begin{equation} \label{total error}
|| \vx_{t_m} - {\hat{\vx}^{\Phi_{SDE}}_{t_m}} ||_2 = ||\Delta ||_2 = || \Delta_2 + \Delta_3 ||_2 \leq ||\Delta_2 ||_2 + || \Delta_3 ||_2 .
\end{equation}

\noindent Denote $h_{t_n}:=\lambda_{t_m} - \lambda_{t_n}$ and $\Delta_{h}:= \max\limits_{n\in [1, N]} h_{t_n}$. 
Consider the $|| \Delta_2||$ term, we can estimate the error $\Delta_2$ by  the Peano remainder term of the Taylor  expansion,
\begin{equation}
    \int_{\lambda_{t_n}}^{\lambda_{t_m}} e^{-\lambda}\mathbf{\hat{\epsilon}_{\theta}}(\mathbf{\hat{x}_{\lambda}}, \lambda) \mathrm{d}\lambda = (e^{ -\lambda_{t_n}} - e^{-\lambda_{t_m}}) \mathbf{\epsilon_{\theta}}(\mathbf{x_{t_n}}, t_n) + O(({h_{t_n}})^2).
\end{equation}
With this, $|| \Delta_2||_2 $ is bounded by $O((\Delta_{h})^2)$.

\noindent Consider the $|| \Delta_3||$ term, we can rewrite 
\begin{align}
        \Delta_3 &= \sigma_{t_n}\sqrt{e^{2h_{t_n}} -1} \mathbf{z_{t_n}} -  \sigma_{t_n}\sqrt{e^{2h_{t_n}} -1} \mathbf{z_{t_m}} \\
        &= \frac{\sigma_{t_n}}{t_n}\sqrt{e^{2h_{t_n}} -1} t_n \mathbf{z_{t_n}} - \frac{\sigma_{t_n}}{t_m}\sqrt{e^{2h_{t_n}} -1} t_m \mathbf{z_{t_m}} .
\end{align}

\noindent Since $  \lim\limits_{h \rightarrow  0}  \frac{\sqrt{e^{2h -1}}}{h} = 1$,
we can derive that
\begin{equation} \label{delta3 1}
\Big|\frac{\sigma_{t_n}}{t_n}\sqrt{e^{2h_{t_n}} -1}  - \frac{\sigma_{t_n}}{t_m}\sqrt{e^{2h_{t_n}} -1} \Big| 
    = \Big|\frac{\sigma_{t_n}}{t_n t_m} \underbrace{\sqrt{e^{2h_{t_n}} -1}}_{O({h_{t_n}})} \underbrace{(t_n -t_m)}_{O({h_{t_n}})} \Big|,
\end{equation}

\noindent which is bounded by $O((\Delta_h)^2)$.

\noindent Thus, we can derive the bound of the $|| \Delta_3||$ term,

\begin{align}
        |\Delta_3| &= \Big|\frac{\sigma_{t_n}}{t_n}\sqrt{e^{2h_{t_n}} -1} t_n \mathbf{z_{t_n}} - \frac{\sigma_{t_n}}{t_m}\sqrt{e^{2h_{t_n}} -1} t_m \mathbf{z_{t_m}} \Big| \\
        & = \Big|\frac{\sigma_{t_n}}{t_n}\sqrt{e^{2h_{t_n}} -1} t_n \mathbf{z_{t_n}} - 
         \frac{\sigma_{t_n}}{t_m}\sqrt{e^{2h_{t_n}} -1} t_n \mathbf{z_{t_n}}
        + \frac{\sigma_{t_n}}{t_m}\sqrt{e^{2h_{t_n}} -1} t_n \mathbf{z_{t_n}}
        -\frac{\sigma_{t_n}}{t_m}\sqrt{e^{2h_{t_n}} -1} t_m \mathbf{z_{t_m}} \Big| \\
        &   \leq
       \Big|\frac{\sigma_{t_n}}{t_n}\sqrt{e^{2h_{t_n}} -1} -\frac{\sigma_{t_n}}{t_m}\sqrt{e^{2h_{t_n}} -1} \Big| \cdot \Big|t_n \mathbf{z_{t_n}} \Big| + \Big|\frac{\sigma_{t_n}}{t_m}\sqrt{e^{2h_{t_n}} -1} \Big| \cdot  \Big| t_n \mathbf{z_{t_n}} - t_m \mathbf{z_{t_m}} \Big|.
\end{align}

\noindent Notice that $t\vz$ is a Wiener process, therefore, we have $| t_n \mathbf{z_{t_n}} - t_m \mathbf{z_{t_m}} \Big| $  bounded by $|O(h_{t_n})|$.

\noindent Thus, we have
\begin{align}
    |\Delta_3| &\leq \underbrace{
       \Big|\frac{\sigma_{t_n}}{t_n}\sqrt{e^{2h_{t_n}} -1} -\frac{\sigma_{t_n}}{t_m}\sqrt{e^{2h_{t_n}} -1} \Big|}_{ O((h_{t_n})^2) \text{by } \cref{delta3 1} } \cdot \underbrace{\Big|t_n \mathbf{z_{t_n}} \Big|}_{\text{Bounded Random Variable}} + \underbrace{\Big|\frac{\sigma_{t_n}}{t_m}\sqrt{e^{2h_{t_n}} -1} \Big|}_{O(h_{t_n})} \cdot  \underbrace{\Big| t_n \mathbf{z_{t_n}} - t_m \mathbf{z_{t_m}} \Big|}_{O(h_{t_n})} \\
        & \leq O((h_{t_n})^2) .
\end{align}

\noindent With this, we prove that both $||\Delta_2||_2 $ and $||\Delta_3||_2$ are bounded by $O((\Delta_{h})^2) $. Thus we can verify the bound $O((\Delta_h)^2) $ of the total local errors $||\Delta||_2$ in \cref{total error}.

By lemma \ref{lemma1}, the SDE solver has a local error bound $O((\Delta_{h})^2) $, and $O((\Delta_{h})^2)= O((\Delta_t)^2)$,
we have $  \sup\limits_{n,\vx} ||\vf_\theta(\vx_{t_n}, t_n) -   \vf(\vx_{t_n}, t_n) ||_2 = O((\Delta t)^p) $, p=1, which completes the proof.

\end{proof}

\section{Multi-step Sampling of Stochastic Consistency Distillation Model}
We present the multi-step sampling algorithm of SCott. 
SCott can generate samples from initial Gaussian noise with one step.
As a consistency model, SCott enjoys 
improving outcomes by alternating denoising and noise injection at inference time. Specifically, at the $n$-th iteration, we first inject noise to the previous predicting sample $\vz$ according to the forward diffusion $\hat{\vz}_{t_n} \sim \mathcal{N}(\alpha(t_n)\vz, \sigma(t_n)\mathbf{I})$. Then we predict the next $\vz$ using SCott. Such a procedure can improve sample quality. Note that $\bar{\vf}_{\vtheta}$ denotes the function that produces Gaussian mean with the model $\vf_{\vtheta}$. The pseudo-code is provided in \cref{alg:sampling}.

\begin{algorithm}[htb] 
   \caption{Multi-step Stochastic Consistency Model Sampling.}
\begin{algorithmic}
    \STATE{\bfseries Input:} SCott model $\vf_\vtheta$, sequence of timesteps $t_N >  t_{N-1} >  \cdot\cdot\cdot > t_2 > t_1$, noise schedule $\alpha_t, \sigma_t$, CFG scale $\omega$, text condition $\vc$ , initial noise $\hat{\vz}_T$\\ 
    $\vz \leftarrow \bar{\vf}_\vtheta(\hat{\vz}_T, c, \omega ,T) $
   \FOR{$n=N$ to $1$}
    \STATE Sample  $ \hat{\vz}_{t_n} \sim \mathcal{N}(\alpha(t_n)\vz, \sigma(t_n)\mathbf{I}) $ \\
   Sample $\vz \leftarrow \bar{\vf}_\vtheta(\hat{\vz}_{t_n}, c, \omega, t_n)$ \\
    \ENDFOR\\
    \STATE{\bfseries Output:} $\vz$

\end{algorithmic}
    \label{alg:sampling}
\end{algorithm}

\section{Empirical investigation of using SDE solvers for CD.} 
For empirical investifation of using SDE solvers for CD,
we experiment on a simple 1D task---characterizing a Gaussian mixture distribution with 3 uniform components $\pdata(\vx) := \frac{1}{3} \mathcal{N}(\vx; -1.5, 0.2^2\mathbf{I}) + \frac{1}{3}\mathcal{N}(\vx; 0, 0.2^2\mathbf{I})  + \frac{1}{3}\mathcal{N}(\vx; 1.5, 0.2^2\mathbf{I})$. 
We first train a DM, instantiated as a 4-layer MLP, with a standard Gaussian prior on it, and then perform CD with SDE and ODE solvers respectively. 
In particular, we select the EDM solver~\cite{karras2022elucidating} with the noise coefficient equaling 0 as the ODE solver and the EDM solver with that equaling 2 as the SDE solver. 

We plot the one-step sampling results of the trained models in \cref{fig:toy}, where the Gaussian noise and sampling trajectory are also displayed. 
As shown, the stochastic CD can successfully map noise points to suitable target samples, even with higher accuracy than the vanilla CD. 

One possible explanation for this phenomenon is that, for a given $t_m$, the sampled states from the SDE solver are likely to surround the deterministic states sampled from the ODE solver. 
This situation can be viewed as a form of data augmentation, which helps the student model better align and rectify its predictions.
These results provide concrete evidence that we can indeed use SDE solvers for CD. 
\begin{figure}[h]
    \centering
    \includegraphics[width=0.6\linewidth]{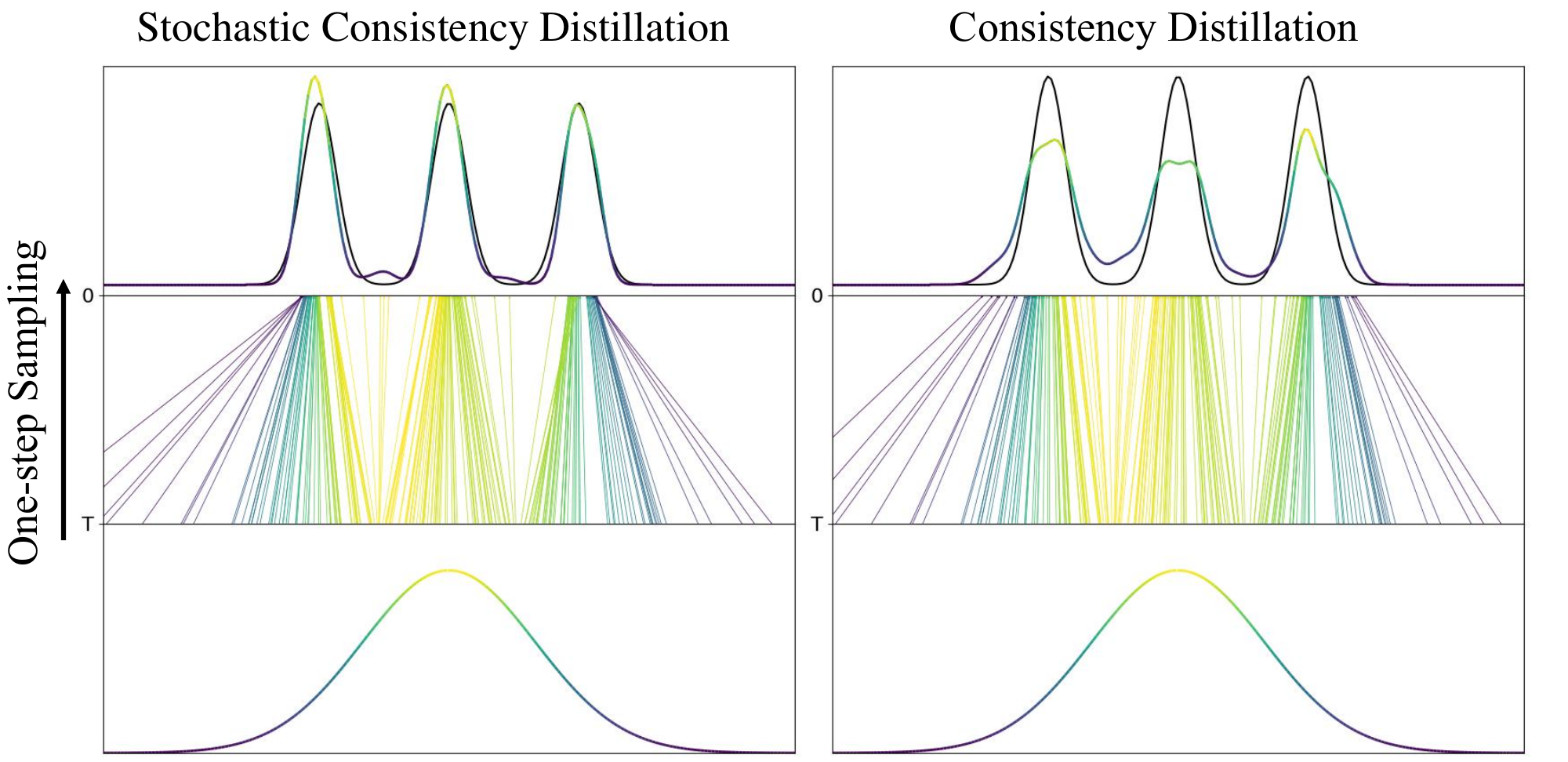}
    \caption{Comparison of stochastic CD (based on SDE solvers) and vanilla CD (based on ODE solvers) on a synthetic generation task. The other experimental settings for the two cases are identical. }
    \label{fig:toy}
\end{figure}

\section{Implementation Details}
\textbf{Trainig.}
We use LAION-Aesthetics-
6+ subset of LAION-5B~\cite{schuhmann2022laion} to train our model.  We train the model with 4 A100 GPUs and a batch size of 40 for 40,000 iterations. For SD training, the learning rate is 8e-6, and the learning rate is 2e-5 when training discriminator. 

\noindent\textbf{Model.}
We use the publicly available Realistic-Vision-v51 (RV5.1) as the teacher, which is obtained by fine-tuning the pre-trained Stable Diffusion-V1.5 (SD1.5)~\cite{rombach2022high}. Both the student and discriminator backbone are initialized by the teacher. When compared to other methods, we use SD1.5 as the teacher for fair comparison. In addition, we apply EMA with a coefficient of 0.995 for the student model. 

\noindent\textbf{Algorithm.} Algorithm \ref{alg:stoc} details our stochastic consistency distillation training procedure.
\begin{algorithm}[h] 
   \caption{Stochastic Consistency Distillation}
\begin{algorithmic}
\label{alg:stoc}
    \STATE{\bfseries Input:} dataset $\mathcal{D}$, initial consistency model parameter $\vtheta$, discriminator parameter $\vphi$, SDE solver $\Phi_{SDE}$, noise schedule $\alpha_t, \sigma_t$, sampling steps $h$, adversarial loss $\mathcal{L}_{adv}(\cdot, \cdot)$, 
     weight $\lambda(\cdot)$,  CFG scale $\omega$, 
     EMA rate $\mu$, loss coefficient $\lambda_{adv}$ \\
     $\vtheta^{-} \leftarrow \vtheta$
   \REPEAT
   \STATE Sample  $(\vx_0, \vc) \sim \mathcal{D} $ \\
   Sample $\vz_{t_n} \sim \mathcal{N} (\alpha_{t_n} \vx_0, \sigma^2_{t_n} \mathbf{I}) $ \\
    $\hat{\vz}_{t_{m}} \leftarrow  \Phi_{SDE}(\vepsilon_{\theta},\vz_{t_{n}}, t_{n},  t_{m}, h) 
   $
   \\
    Calculate $\mathcal{L}_{CD}(\vtheta,\vtheta^{-}; \Phi_{SDE} )$ with \cref{eq:cd} \\
    Calculate $ \mathcal{L}_{adv}(\vtheta, \vphi)$ with \cref{eq: gan} \\
    Calculate $\mathcal{L}_{SCott}(\vtheta, \vphi)$ with \cref{eq: scott} \\
 $\vtheta \leftarrow \vtheta -  \frac{\partial }{\partial \vtheta} \mathcal{L}_{SCott}(\vtheta, \vphi)$ \\
 $\vphi \leftarrow \vphi - \frac{\partial}{\partial \vphi}\mathcal{L}_{SCott}(\vtheta, \vphi)$ \\
 $\vtheta^{-} \leftarrow  $ stop\_grad$(\mu \vtheta^{-} + (1 - \mu) \vtheta)$
   \UNTIL{convergence}
    
\end{algorithmic}
\end{algorithm}

\newpage
\section{Additional Results }
\subsection{Qualitative Results For Table 2 in the main part of our paper} 


\begin{figure*}[th]
\centering
\textbf{DDIM (50 steps) \quad DPM++ (25 steps) \quad LCM (2 steps)\quad InstaFlow (1 step)\quad SCott (2 steps)}\\
\includegraphics[width=0.16\textwidth]{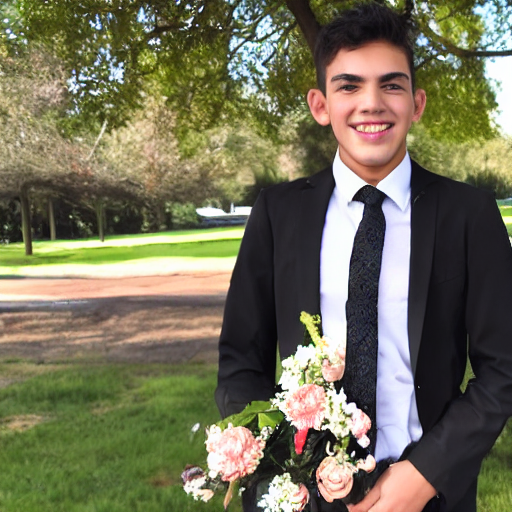}
\includegraphics[width=0.16\textwidth]{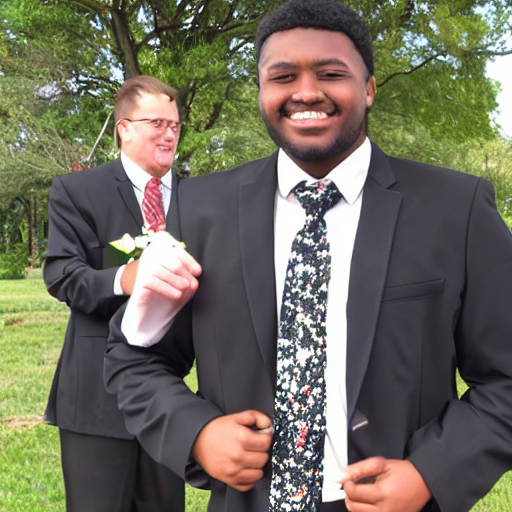}
\includegraphics[width=0.16\textwidth]{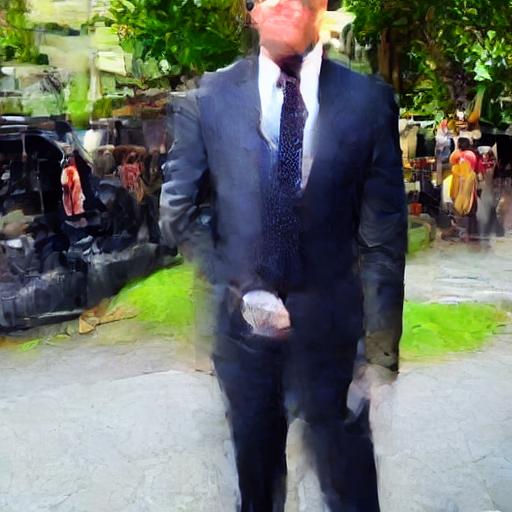}
\includegraphics[width=0.16\textwidth]{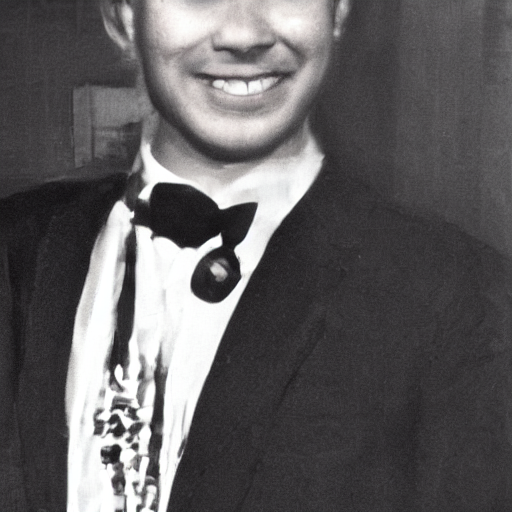}
\includegraphics[width=0.16\textwidth]{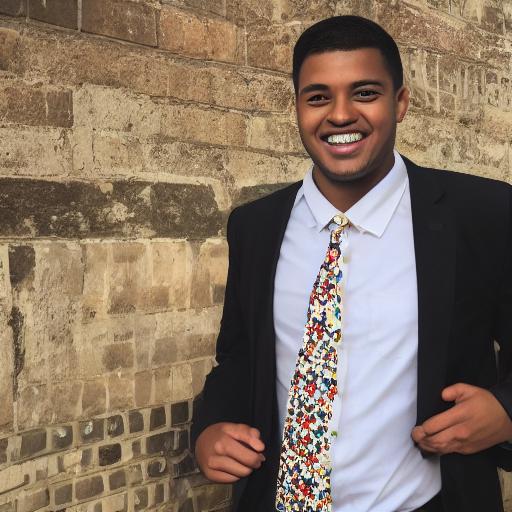}

\small{A young man wearing black attire and a flowered tie is standing and smiling.} \\
\includegraphics[width=0.16\textwidth]{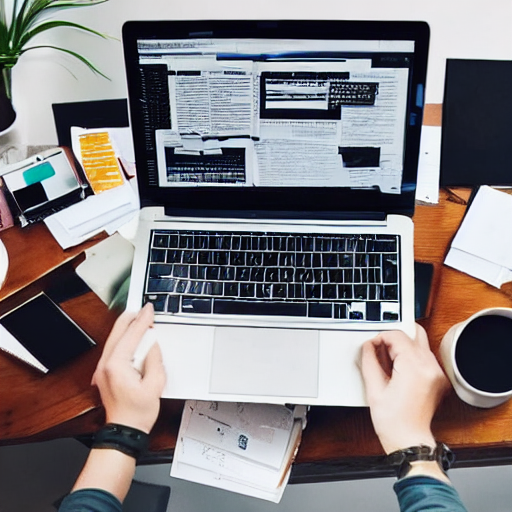}
\includegraphics[width=0.16\textwidth]{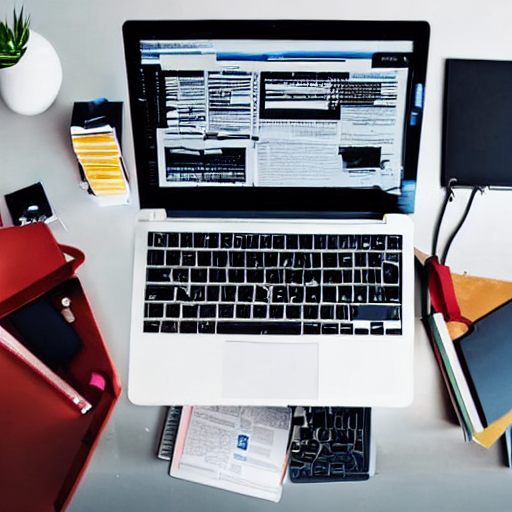}
\includegraphics[width=0.16\textwidth]{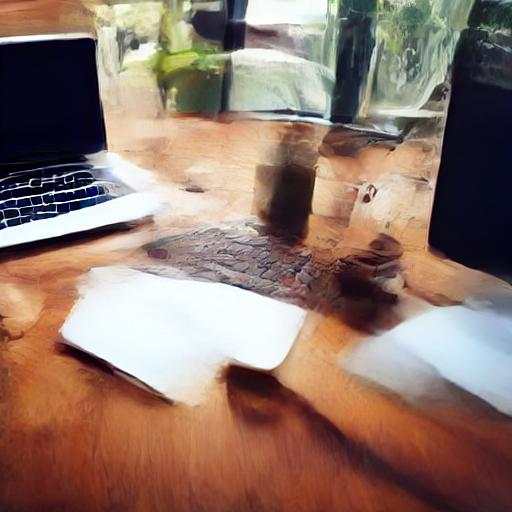}
\includegraphics[width=0.16\textwidth]{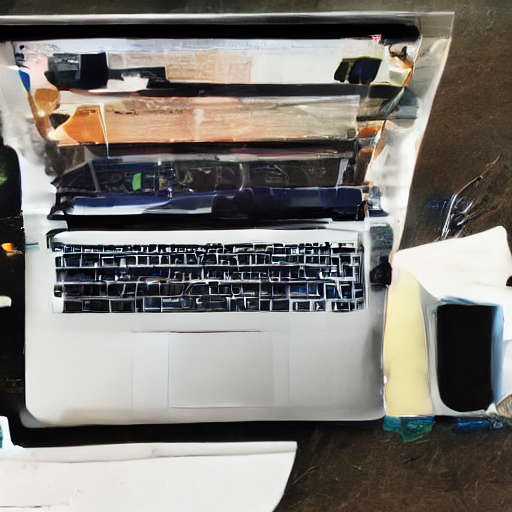}
\includegraphics[width=0.16\textwidth]{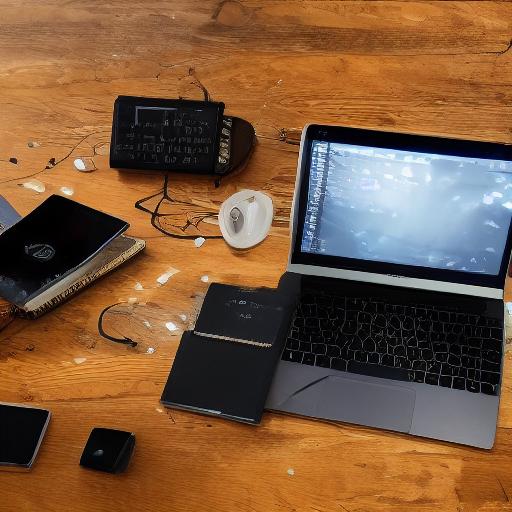}

\small{A picture of a messy desk with an open laptop.} \\
\includegraphics[width=0.16\textwidth]{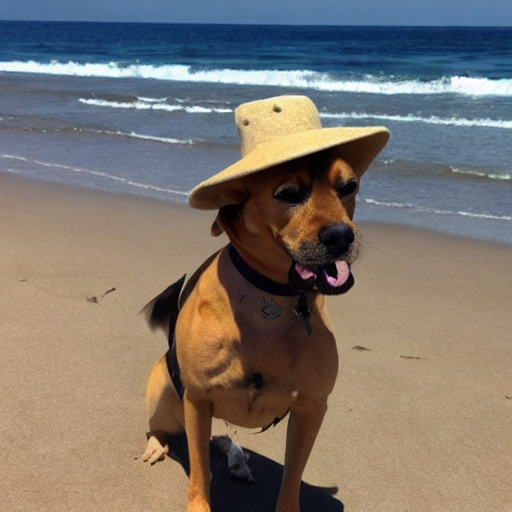}
\includegraphics[width=0.16\textwidth]{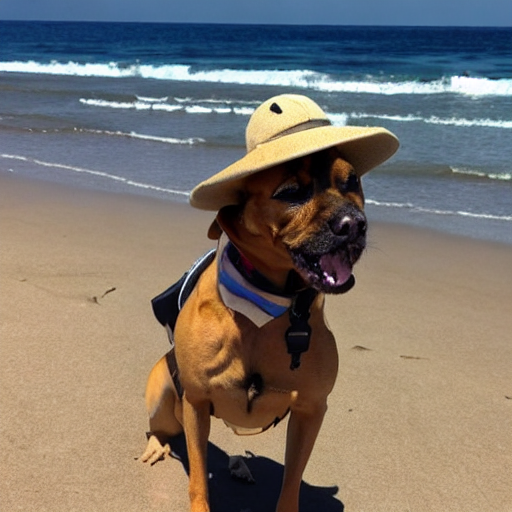}
\includegraphics[width=0.16\textwidth]{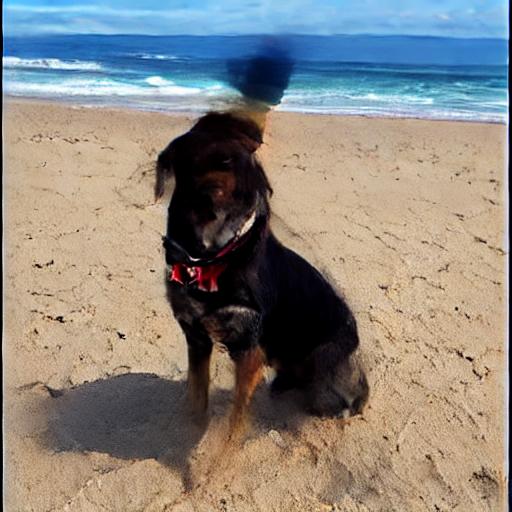}
\includegraphics[width=0.16\textwidth]{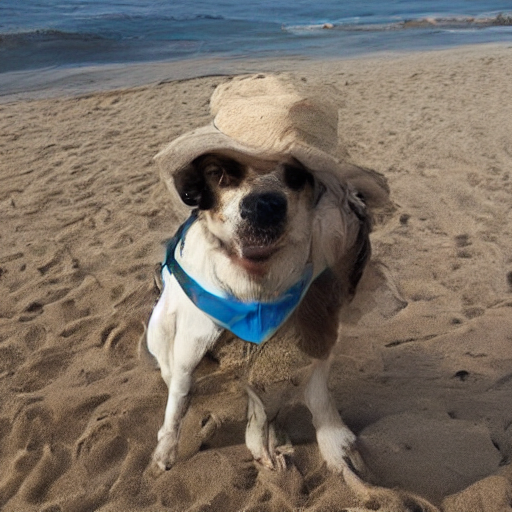}
\includegraphics[width=0.16\textwidth]{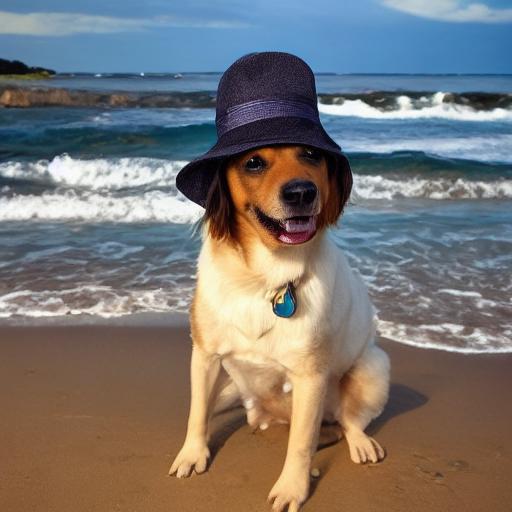}

\small{A dog wearing a hat on the beach.}\\
\includegraphics[width=0.16\textwidth]{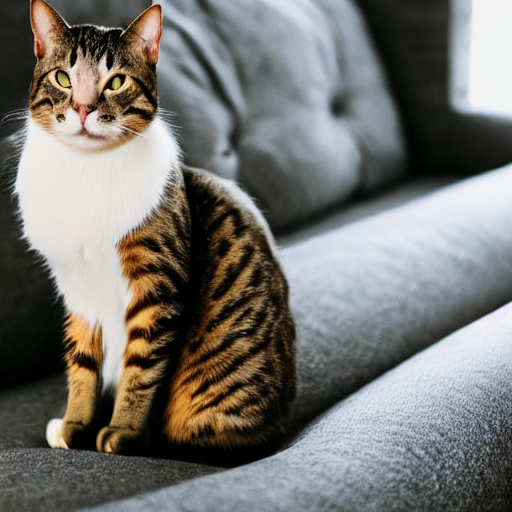}
\includegraphics[width=0.16\textwidth]{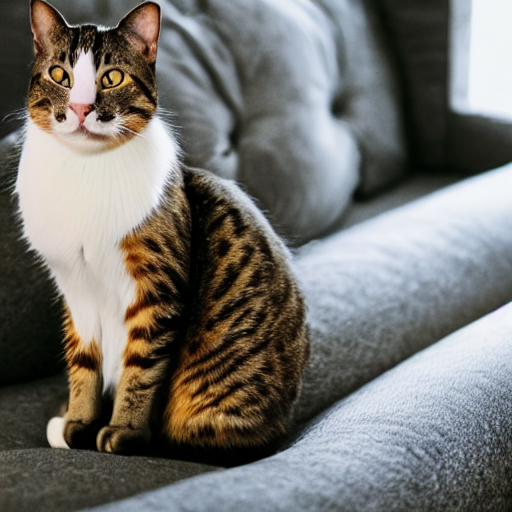}
\includegraphics[width=0.16\textwidth]{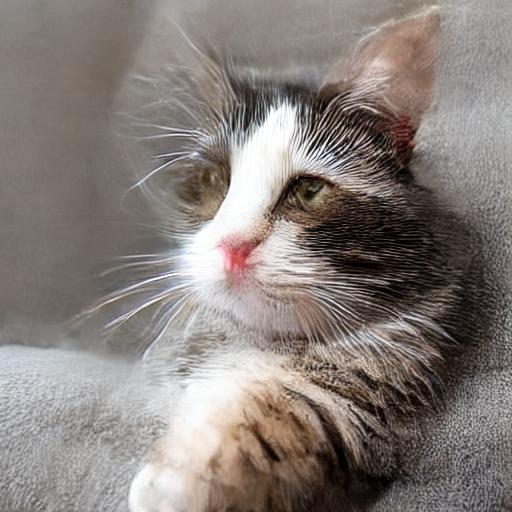}
\includegraphics[width=0.16\textwidth]{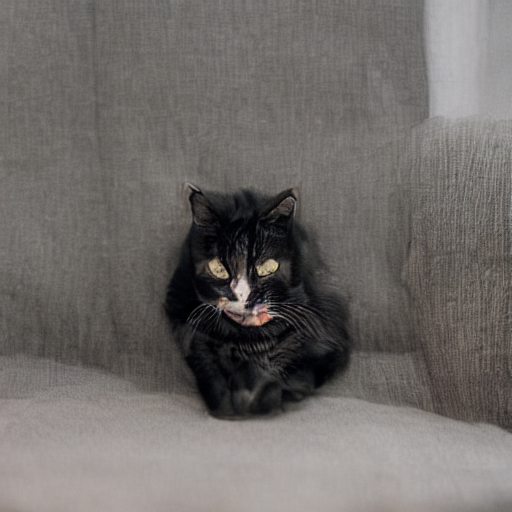}
\includegraphics[width=0.16\textwidth]{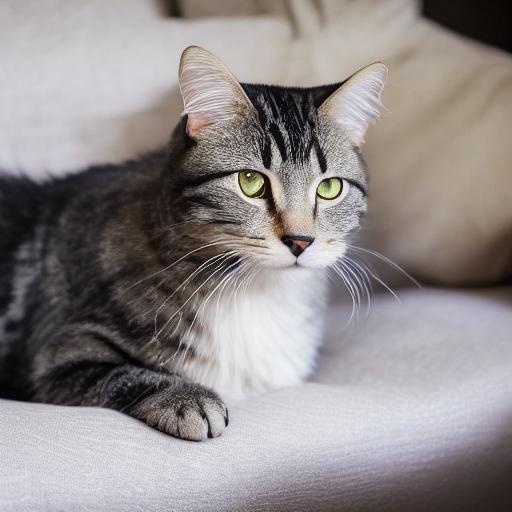}

\small{A cat sitting on couch.}\\
\includegraphics[width=0.16\textwidth]{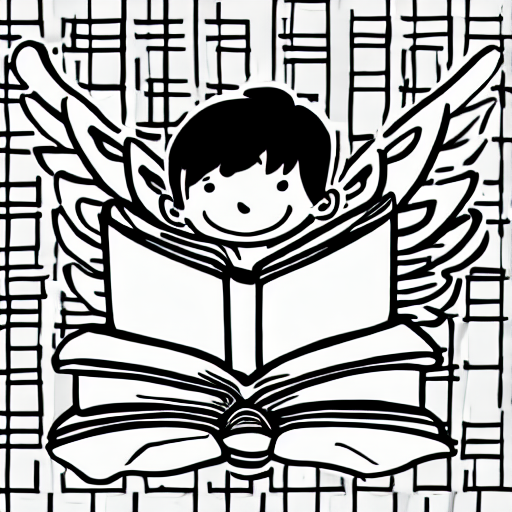}
\includegraphics[width=0.16\textwidth]{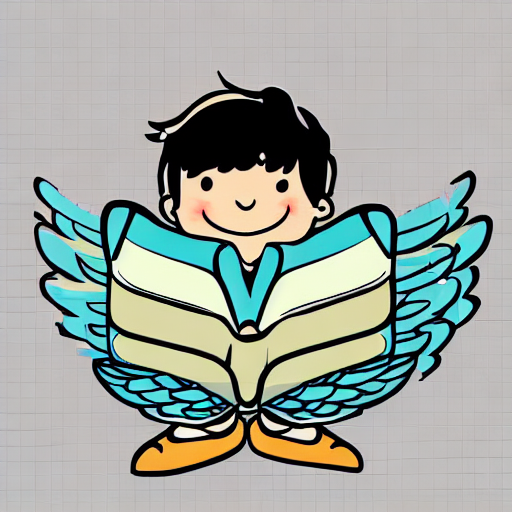}
\includegraphics[width=0.16\textwidth]{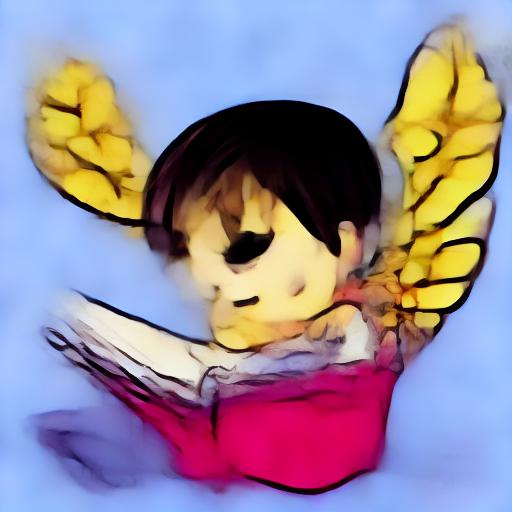}
\includegraphics[width=0.16\textwidth]{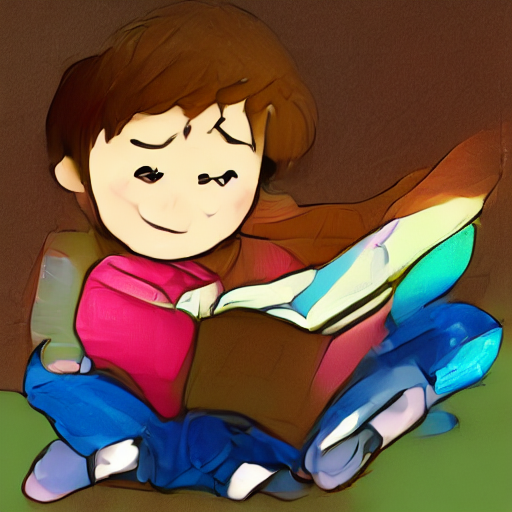}
\includegraphics[width=0.16\textwidth]{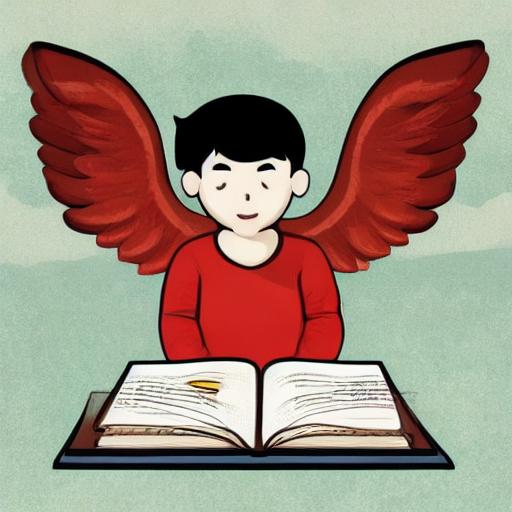}

\small{A cute boy with wings is reading book, Cartoon Drawings, high details.} 
\caption{Qualitative comparisons of SCott against competing methods and DDIM, DPM++ baselines. All models are initialized by SD1.5.}
  \label{fig:sd1.5}
\end{figure*} 

We provide qualitative examples to benchmark our SCott against the state-of-the-art fast sampling methods, LCM~\cite{luo2023latent}, InstaFlow~\cite{liu2023instaflow}, and DDIM~\cite{song2020ddim}, DPM++~\cite{lu2022dpm++} baselines in \cref{fig:sd1.5}, corresponding to Table 2 in the main part of our paper. 

\subsection{Comparison on MSCOCO-2017 5K with Dreamshaper-v7.}

InstaFlow and LCM teams release the checkpoints and inference codes of InstaFlow+dreamshaper-7\footnote{\url{https://github.com/gnobitab/InstaFlow?tab=readme-ov-file}.} and LCM-dreamshaper-7\footnote{\url{https://huggingface.co/latent-consistency/lcm-lora-sdv1-5}.}, both of which are build on dreamshaper-v7 (DS-V7)\footnote{\url{https://huggingface.co/stablediffusionapi/dreamshaper-v7}.}. Therefore, we conduct further comparative analysis with LCM and InstaFlow based on DS-V7. Table \ref{table:ds} lists the metrics of different methods evaluated on MSCOCO-2017 5K. We observe our method surpasses InstaFlow-0.9B and LCM with respect to CS and CR, which again indicates our method has better text-to-image consistency and diversity.

\begin{table*}[h]
\caption{Comparisons with the state-of-the-art methods on MSCOCO-2017 5K in terms of FID,  CS, CR. All models are based on  DS-V7.  }
\label{table:ds}
\vskip 0.15in
\begin{center}
\begin{small}
\begin{sc}
\begin{tabular}{lcccccr}
\toprule
Method& Step & Time (s) & FID  $\downarrow$ & CS $\uparrow$  & CR $\uparrow $  \\
\midrule
DPM++~\cite{lu2022dpm++}&25&0.88&31.1&0.325&0.8564 \\
DDIM~\cite{song2020ddim}&50&$-$&32.0&0.323&0.8435 \\
\hline 
LCM~\cite{luo2023latent}&4&0.21&41.1&0.300&0.7495\\
InstaFlow-0.9B~\cite{liu2023instaflow}&1&0.09&\textbf{24.9}&0.310&0.8630\\
SCott (Ours)&2&0.13&28.6&\textbf{0.318}&\textbf{0.8688}\\
\bottomrule
\end{tabular}
\end{sc}
\end{small}
\end{center}
\vskip -0.1in
\end{table*}

For FID, InstaFlow-0.9B~\cite{liu2023instaflow} shows lower value than DPM++~\cite{lu2022dpm++}, DDIM~\cite{song2020ddim} and our method. The probable reason is that zero-shot FID calculated on  MSCOCO-2017 5K  for evaluating visual
quality might not be the most reliable, as discussed in prior works~\cite{betzalel2022study,podell2023sdxl}. Therefore, we exploit qualitative assessment for intuitive comparisons. Figure \ref{fig:dsv7} demonstrates SCott's generated images are superior to InstaFlow and LCM by a substantial margin. More comparison results are shown in \cref{more comparison}.

\begin{figure}[htbp]
\centering
\textbf{DDIM (50 steps) \quad DPM++ (25 steps) \quad LCM (2 steps)\quad InstaFlow (1 step)\quad SCott (2 steps)}\\
\includegraphics[width=0.16\textwidth]{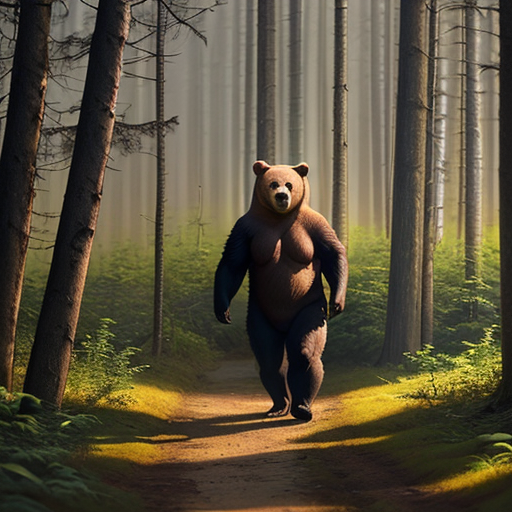}
\includegraphics[width=0.16\textwidth]{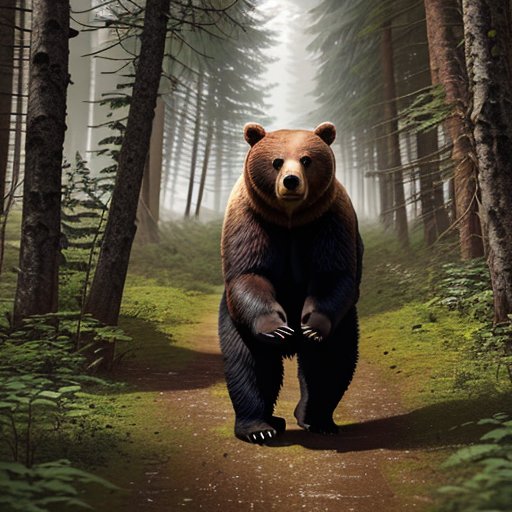}
\includegraphics[width=0.16\textwidth]{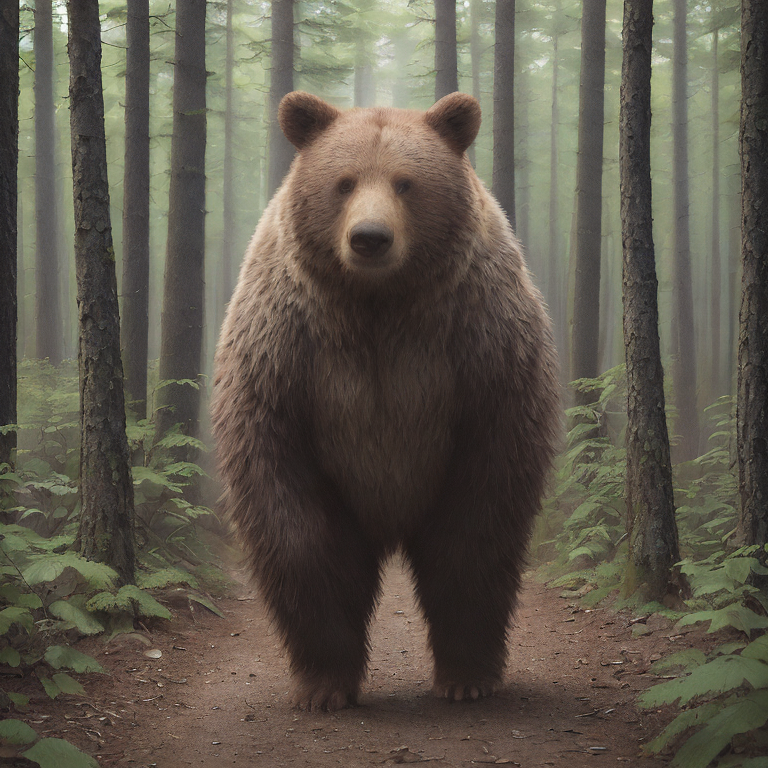}
\includegraphics[width=0.16\textwidth]{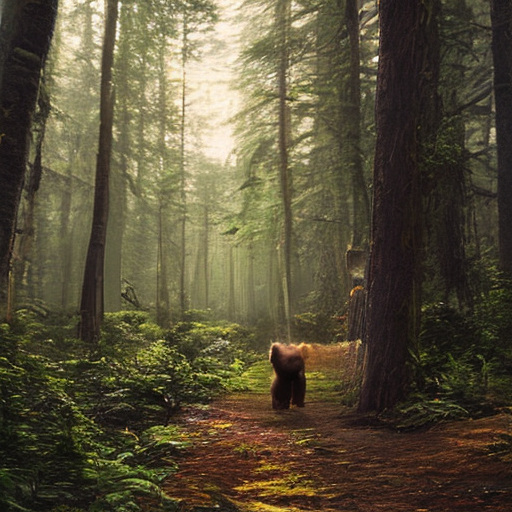}
\includegraphics[width=0.16\textwidth]{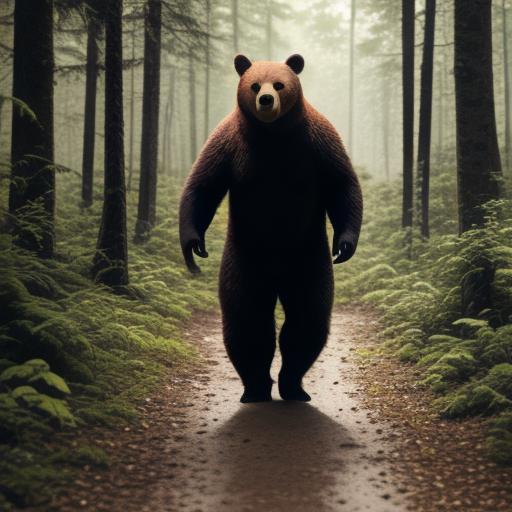}

\small{A bear walking in the forest in the early morning.} \\
\includegraphics[width=0.16\textwidth]{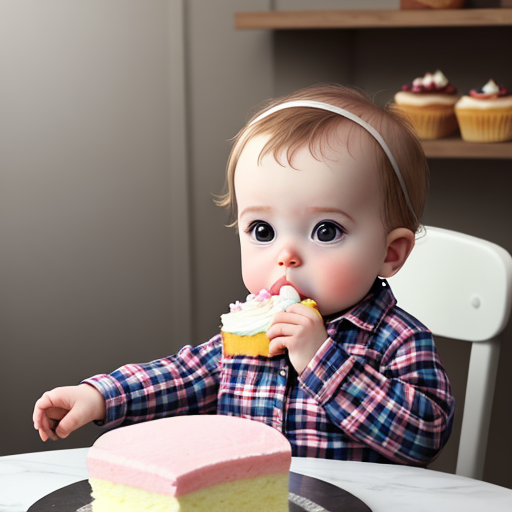}
\includegraphics[width=0.16\textwidth]{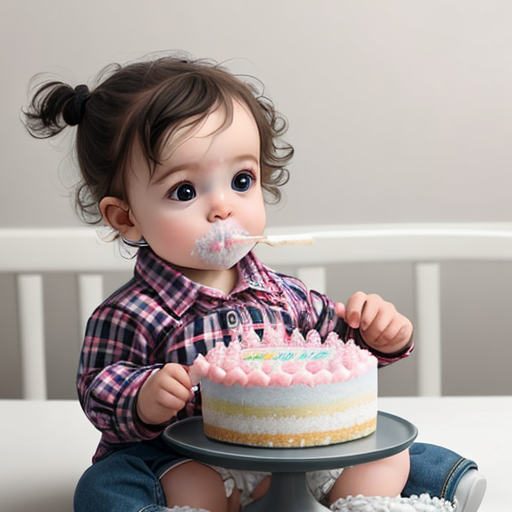}
\includegraphics[width=0.16\textwidth]{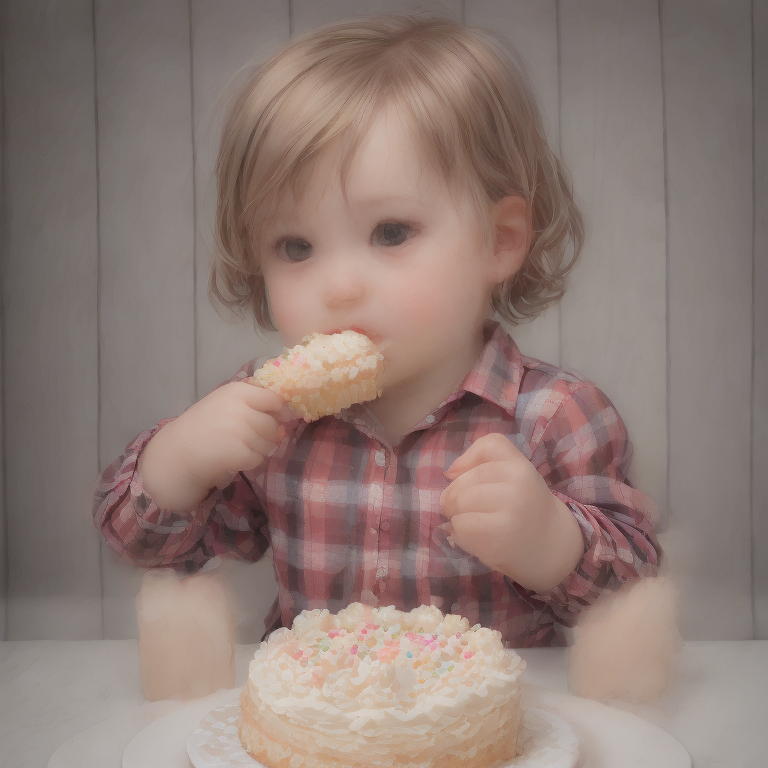}
\includegraphics[width=0.16\textwidth]{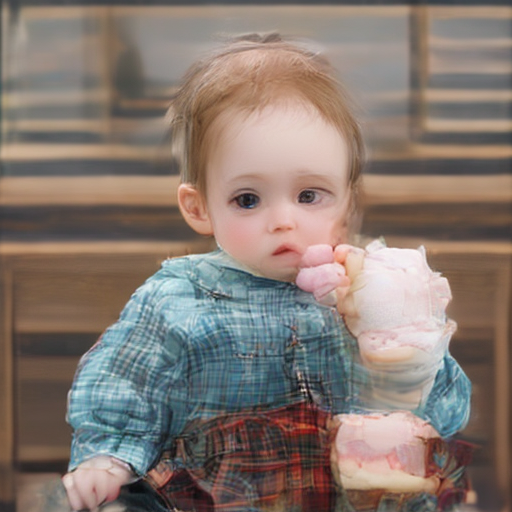}
\includegraphics[width=0.16\textwidth]{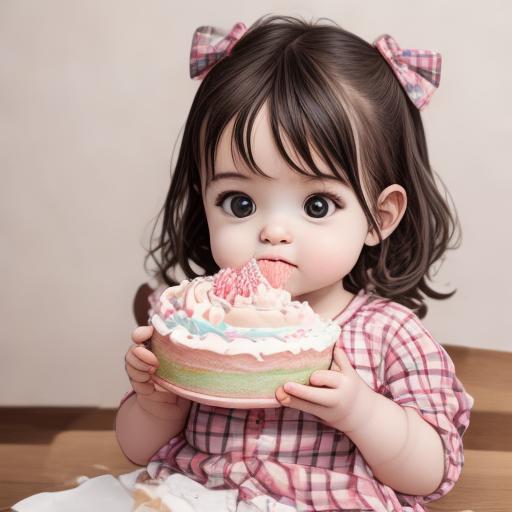}

\small{A Cute Puppy with wings, Cartoon Drawings, high details.} \\
\includegraphics[width=0.16\textwidth]{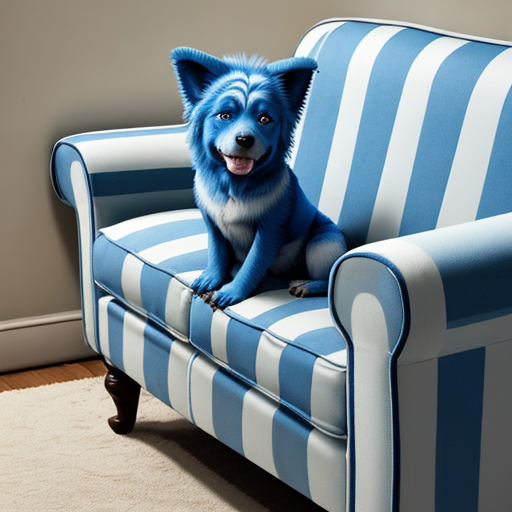}
\includegraphics[width=0.16\textwidth]{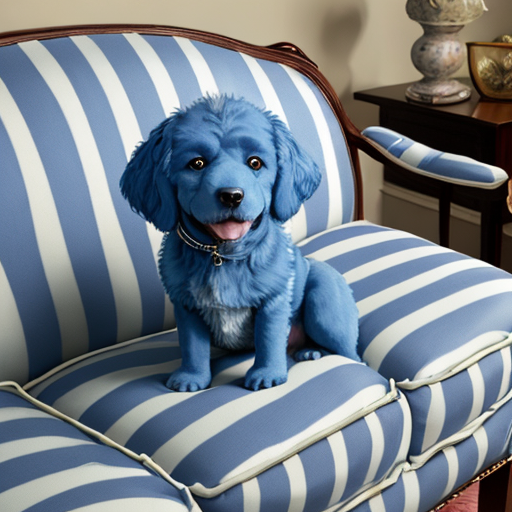}
\includegraphics[width=0.16\textwidth]{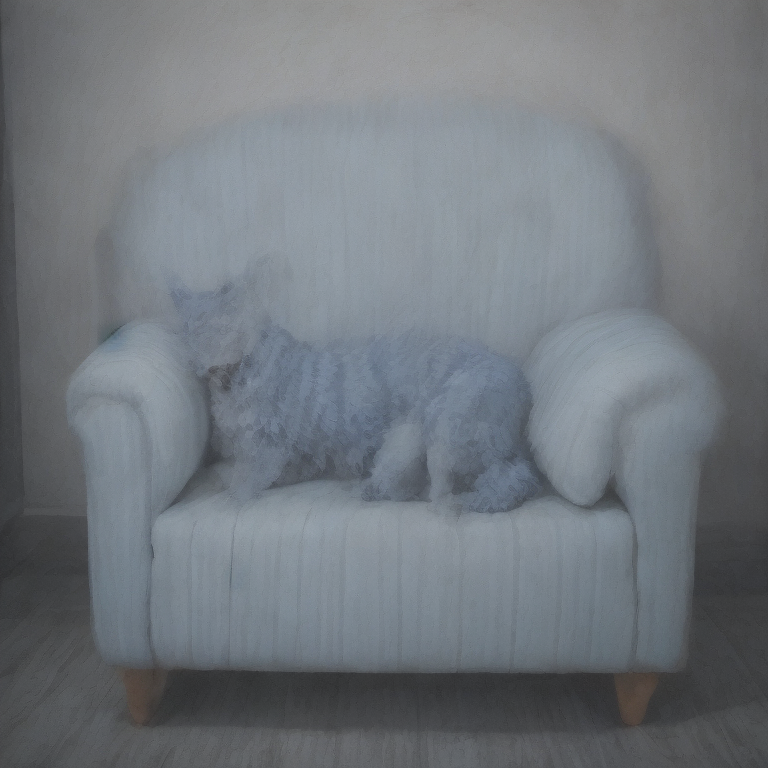}
\includegraphics[width=0.16\textwidth]{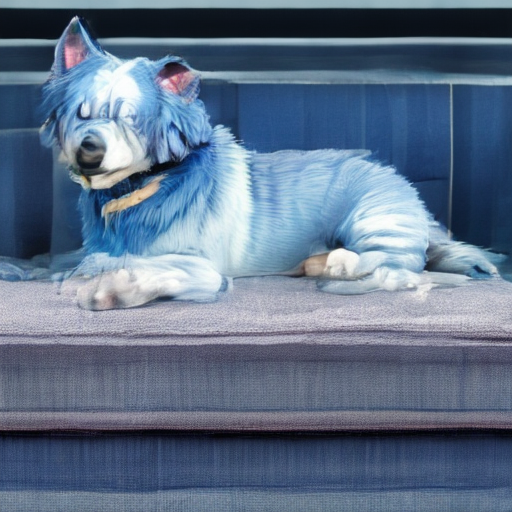}
\includegraphics[width=0.16\textwidth]{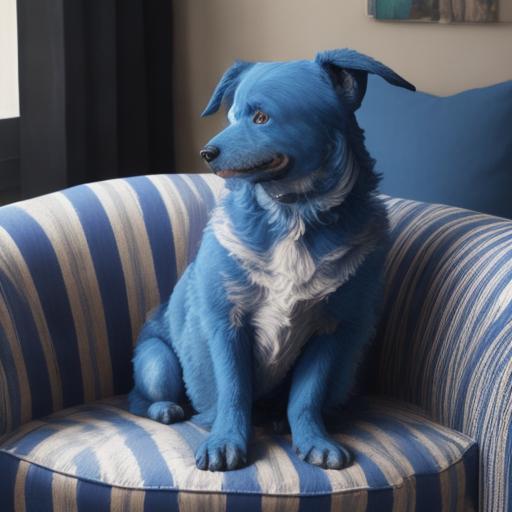}

\small{A blue dog sitting on a striped couch.}\\
\includegraphics[width=0.16\textwidth]{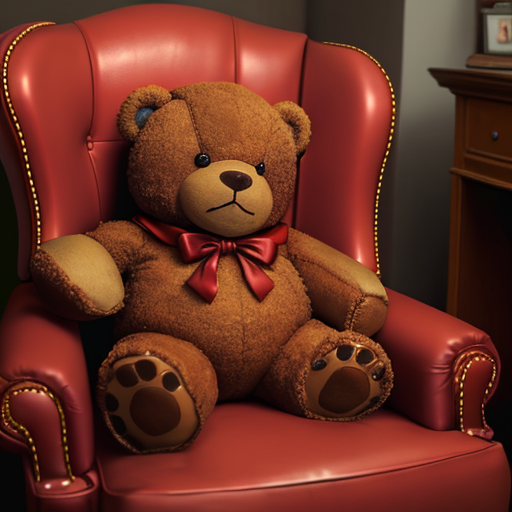}
\includegraphics[width=0.16\textwidth]{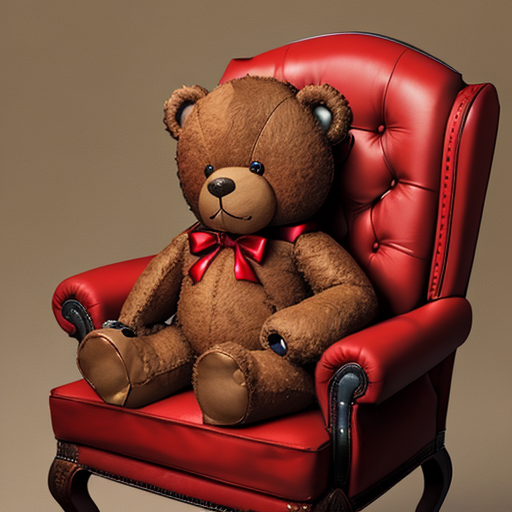}
\includegraphics[width=0.16\textwidth]{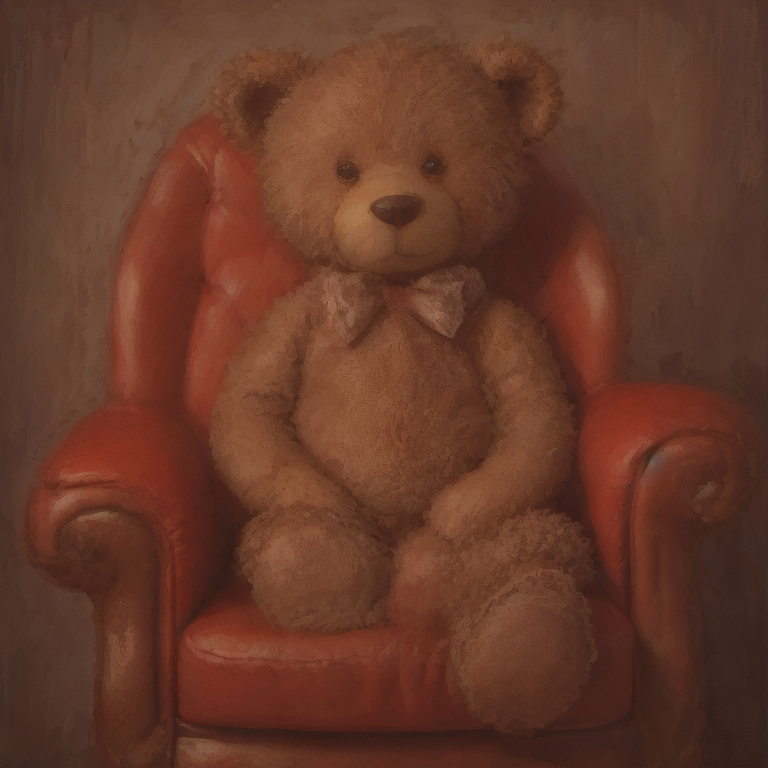}
\includegraphics[width=0.16\textwidth]{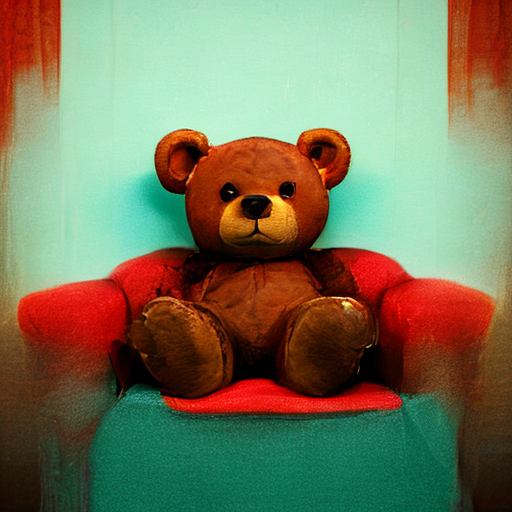}
\includegraphics[width=0.16\textwidth]{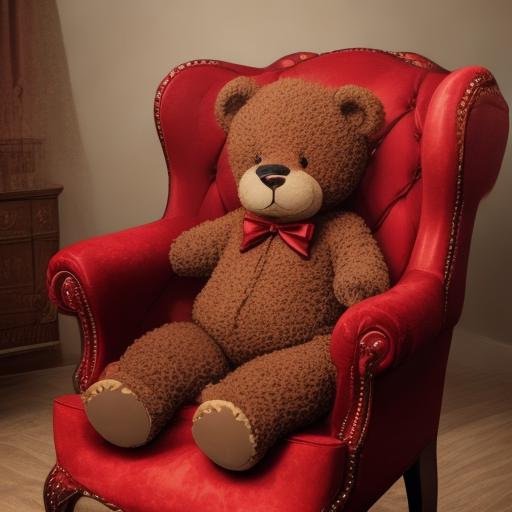}

\small{A brown teddy bear sitting in a red chair.}\\
\includegraphics[width=0.16\textwidth]{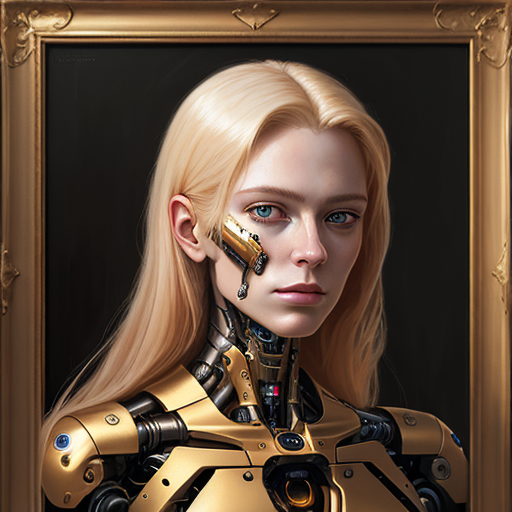}
\includegraphics[width=0.16\textwidth]{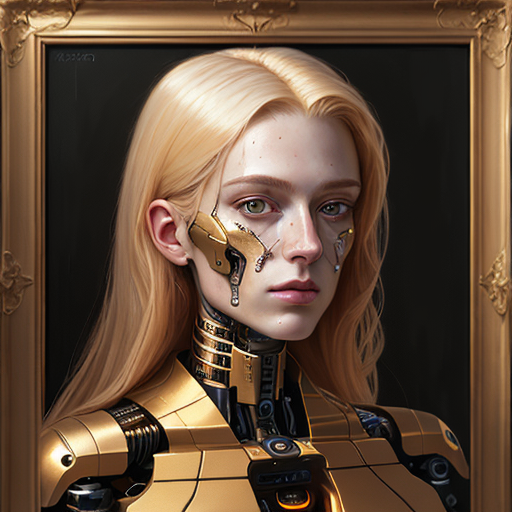}
\includegraphics[width=0.16\textwidth]{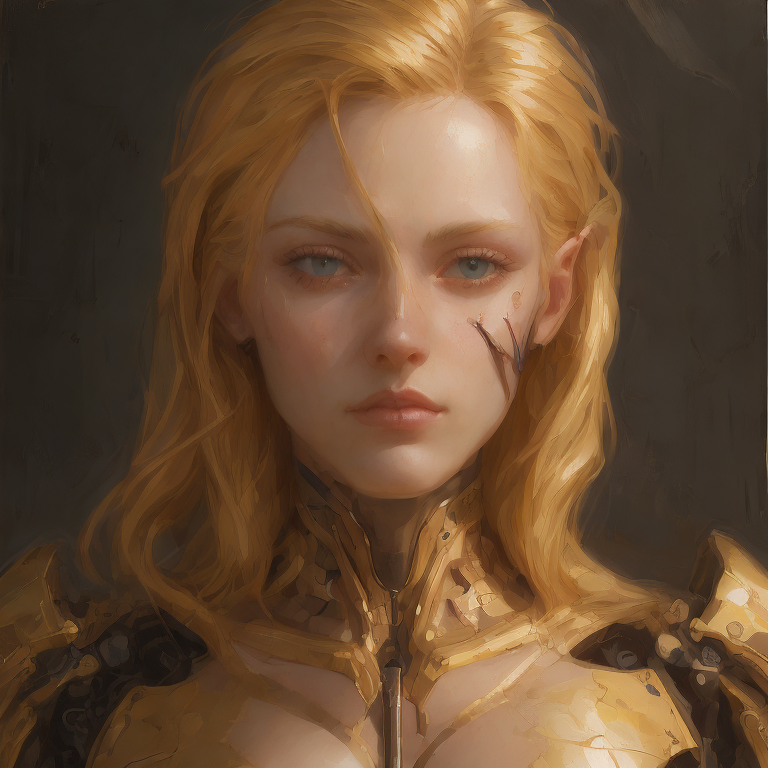}
\includegraphics[width=0.16\textwidth]{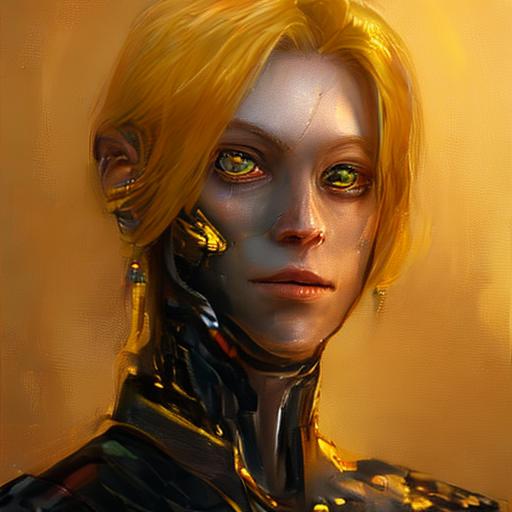}
\includegraphics[width=0.16\textwidth]{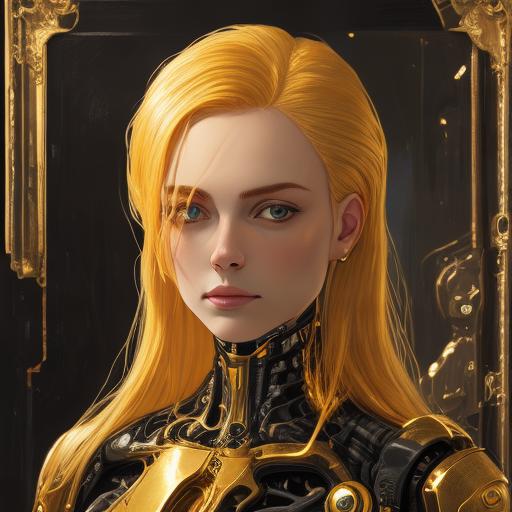}

\small{Self-portrait oil painting, a beautiful cyborg with golden hair, 8k.} 
\caption{Qualitative comparisons of SCott against competing methods and DDIM, DPM++ baselines. All models are initialized by DS-V7.}
\label{fig:dsv7}
\end{figure}


\subsection{More Ablation Studies}

\begin{table}[ht]
\caption{Performance comparison of our SCott on MSCOCO-2017 5K using different teachers with 2-step inference.} 
\label{table:teacher}
\vskip 0.15in
\begin{center}
\begin{small}
\begin{sc}
\begin{tabular}{lcccccr}
\toprule
Teacher & FID  $\downarrow$ & CS $\uparrow $  & CR $\uparrow$ \\
\midrule

SD1.5&\textbf{22.1}&0.308&\textbf{0.9169}\\
DS-V7&28.6&0.318&0.8688\\
RV5.1&26.2&\textbf{0.323}&0.8912\\
\bottomrule
\end{tabular}
\end{sc}
\end{small}
\end{center}
\vskip -0.1in
\end{table}

\textbf{Teacher type.} \cref{table:teacher} shows the ablation study on different teacher types. 
RV5.1 gets the highest CS value, indicating it achieves the highest text-to-image alignment. SD1.5 shows better FID than RV5.1 and DS-V7. This might be because RV5.1 and DS-V7 are fine-tuned from SD1.5 with high-quality images, losing some diversities.  

\subsection{Visual Results of 
Diversity} \label{diversity}
The diversity advantages of our SCott are visualized in \cref{fig: div dog,fig: div young,fig: div desk,fig: div high}. SCott (w/o GAN) denotes the model trained by loss $\mathcal{L_{KL}}$ while CD is obtained by replacing the SDE solver in SCott (w/o GAN) to ODE solver DDIM. Both the models are initialized by RV5.1.  All images are generated with 6 sampling steps. We can observe that SCott (w/o GAN) exhibits more inter-sample diversities.  Besides, the SCott (w/o GAN)'s synthesized images present richer details. Generally, SCott (w/o GAN) leads to better human preference than CD.

\begin{figure}[htbp]
\centering
\begin{minipage}[t]{0.49\textwidth}
\centering
\includegraphics[width=3cm]{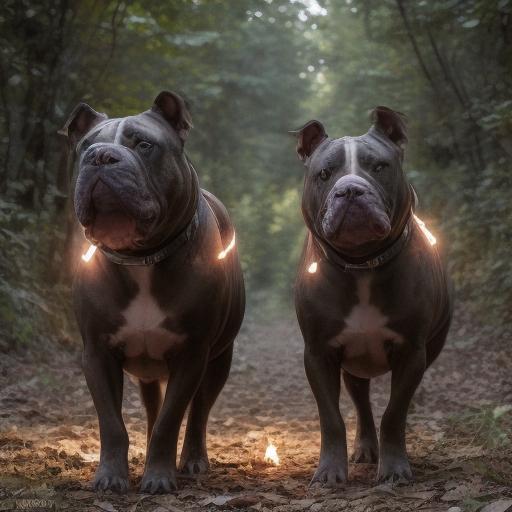}
\includegraphics[width=3cm]{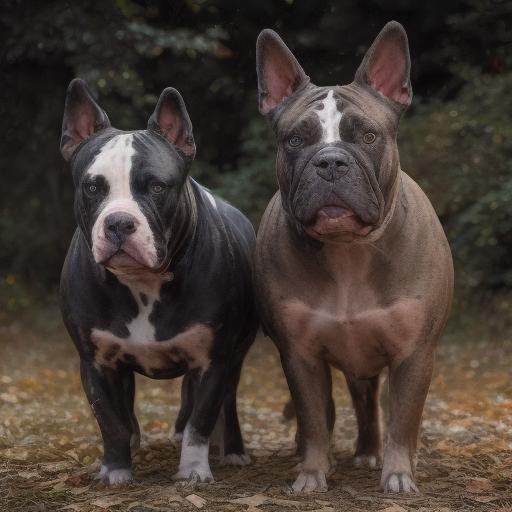}
\includegraphics[width=3cm]{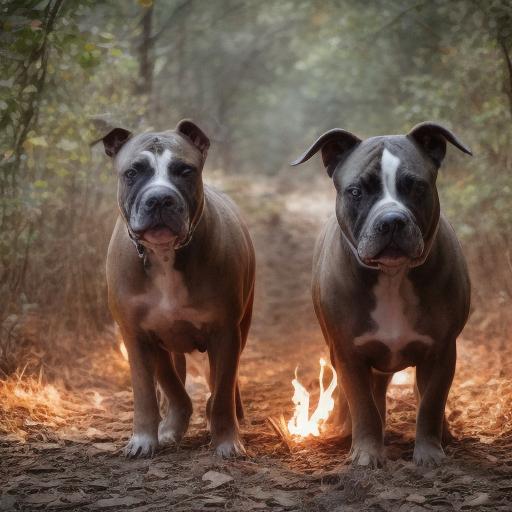}
\includegraphics[width=3cm]{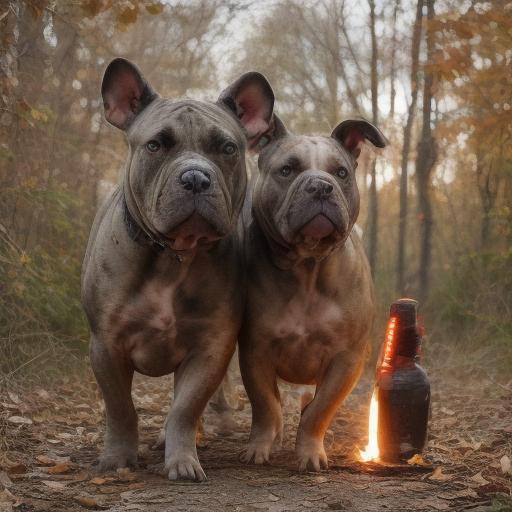}
\end{minipage}
\begin{minipage}[t]{0.49\textwidth}
\centering
\includegraphics[width=3cm]{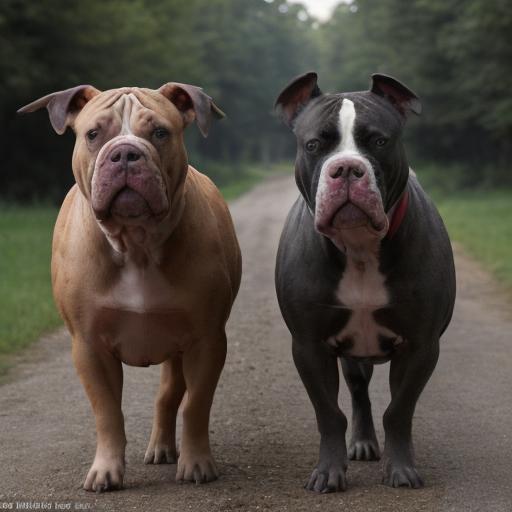}
\includegraphics[width=3cm]{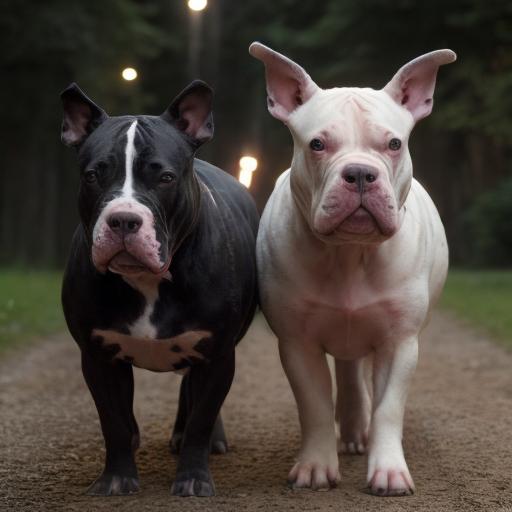}
\includegraphics[width=3cm]{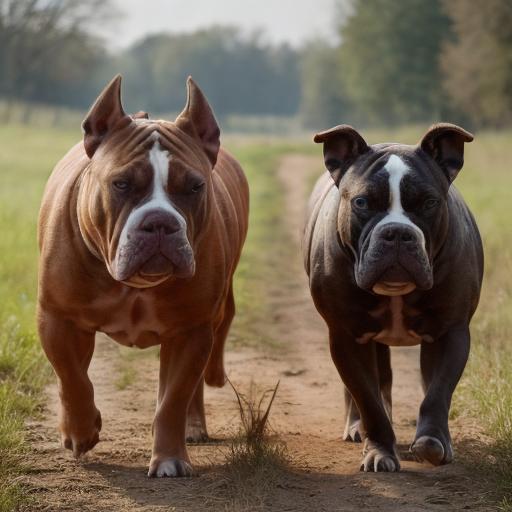}
\includegraphics[width=3cm]{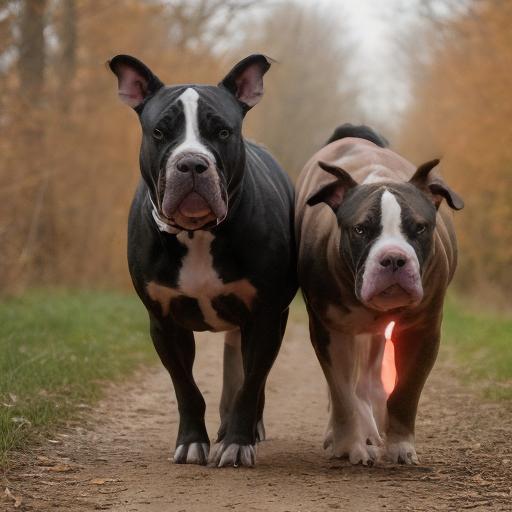}
\end{minipage}
 \begin{center}
\footnotesize CD\qquad\qquad\qquad\qquad\qquad\qquad\qquad\qquad\qquad\qquad\qquad\qquad\qquad\qquad\quad\quad SCott (w/o GAN)
\end{center} 
\caption{Prompt: 2 american bullys with ungre face red eyes crop ears big hed walking torch me.}
\label{fig: div dog}
\end{figure}

\begin{figure}[htbp]
\centering
\begin{minipage}[t]{0.49\textwidth}
\centering
\includegraphics[width=3cm]{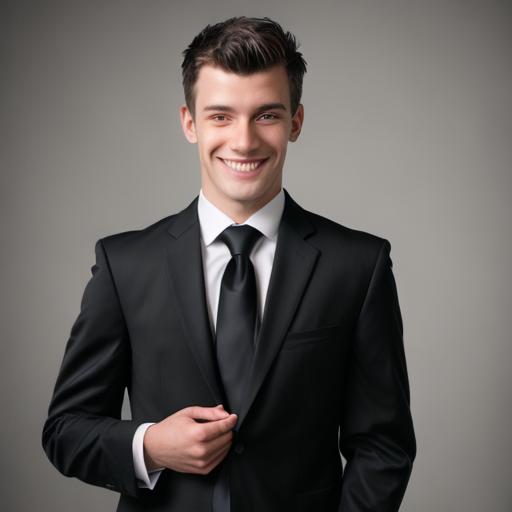}
\includegraphics[width=3cm]{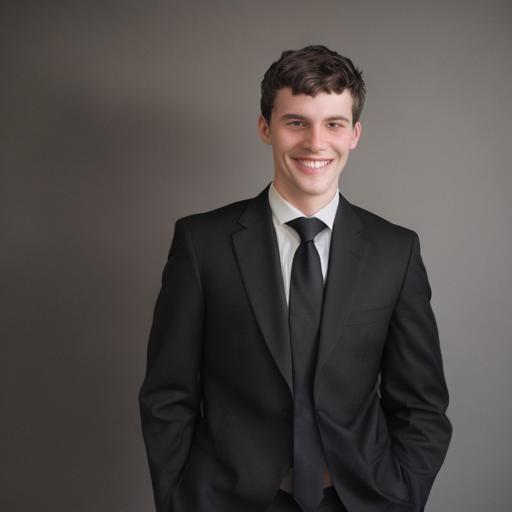}
\includegraphics[width=3cm]{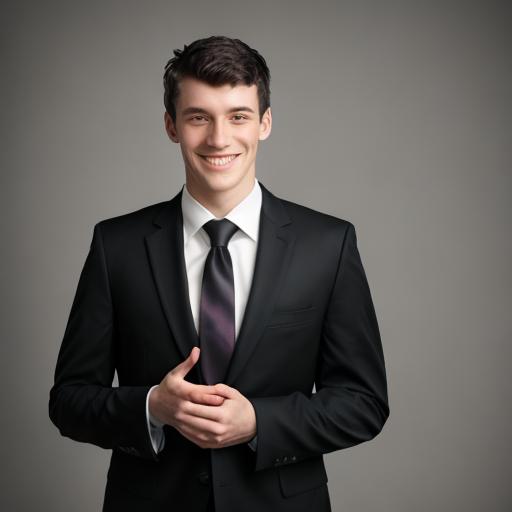}
\includegraphics[width=3cm]{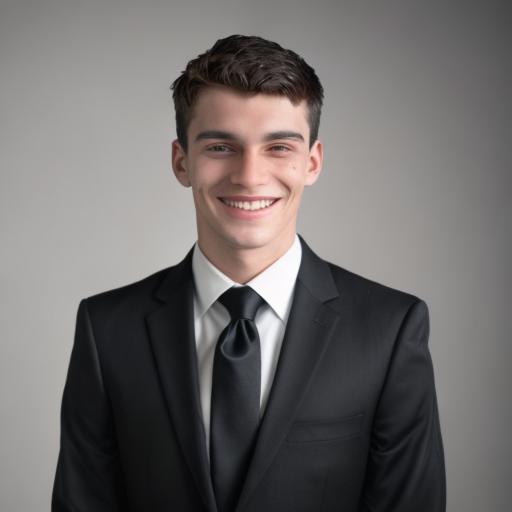}
\end{minipage}
\begin{minipage}[t]{0.49\textwidth}
\centering
\includegraphics[width=3cm]{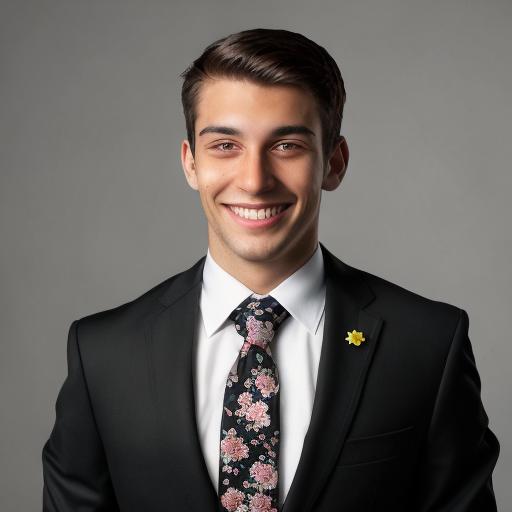}
\includegraphics[width=3cm]{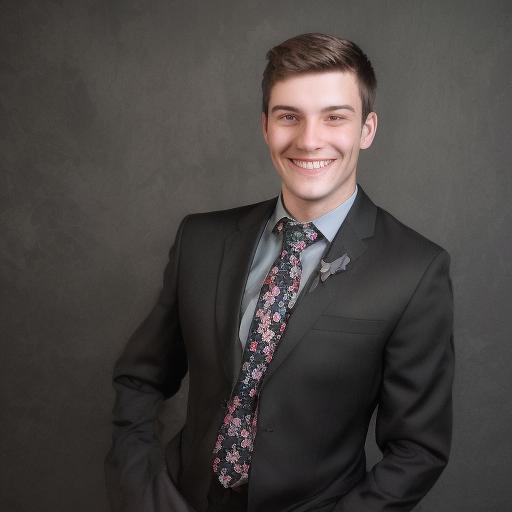}
\includegraphics[width=3cm]{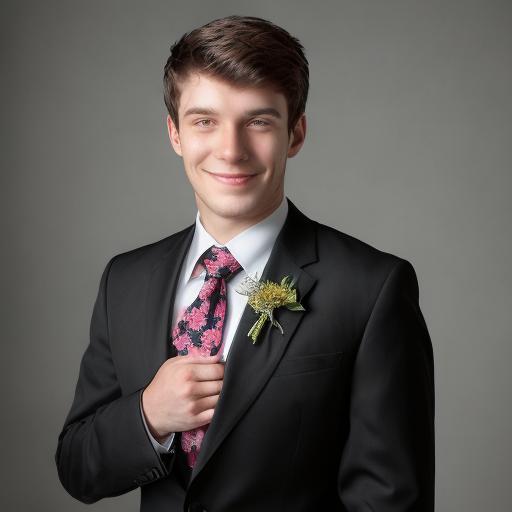}
\includegraphics[width=3cm]{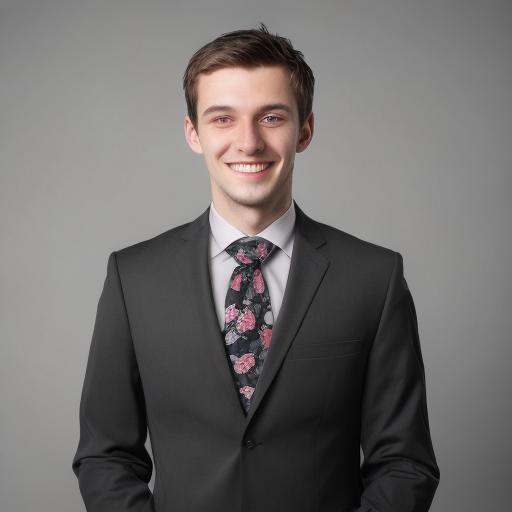}
\end{minipage}
 \begin{center}
\footnotesize CD\qquad\qquad\qquad\qquad\qquad\qquad\qquad\qquad\qquad\qquad\qquad\qquad\qquad\qquad\quad\quad SCott (w/o GAN)
\end{center} 
\caption{Prompt: A young man wearing black attire and a flowered tie is standing and smiling.}
\label{fig: div young}
\end{figure}

\begin{figure}[htbp]

\centering
\begin{minipage}[t]{0.49\textwidth}
\centering
\includegraphics[width=3cm]{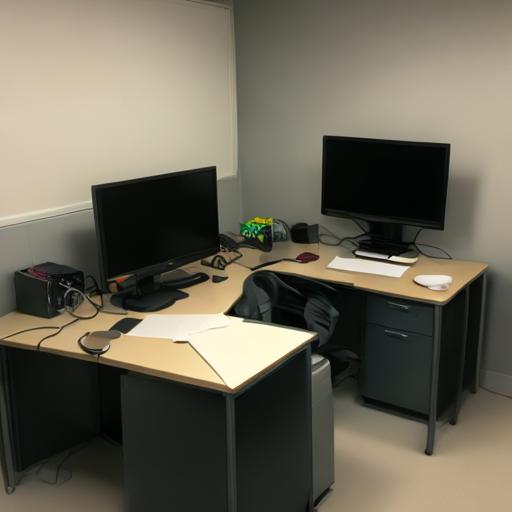}
\includegraphics[width=3cm]{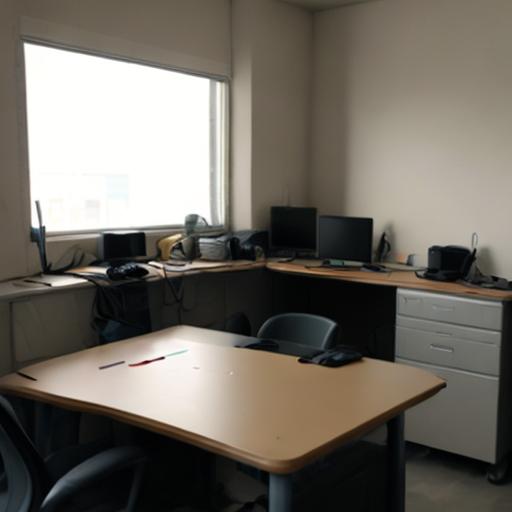}
\includegraphics[width=3cm]{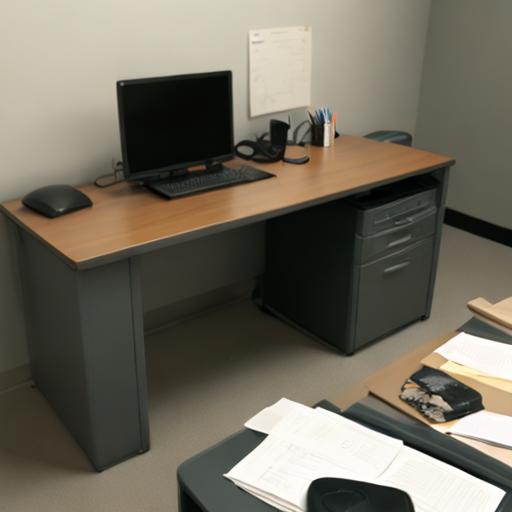}
\includegraphics[width=3cm]{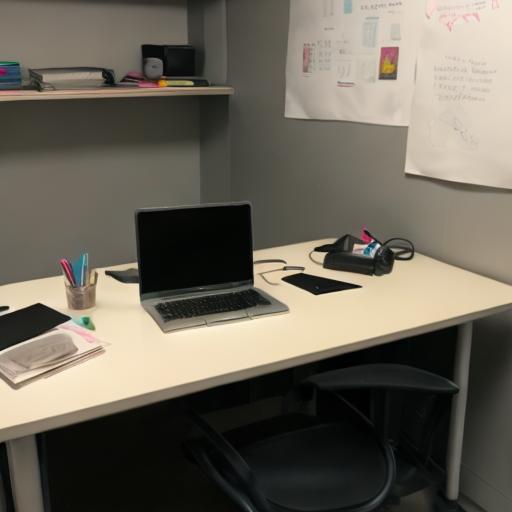}
\end{minipage}
\begin{minipage}[t]{0.49\textwidth}
\centering
\includegraphics[width=3cm]{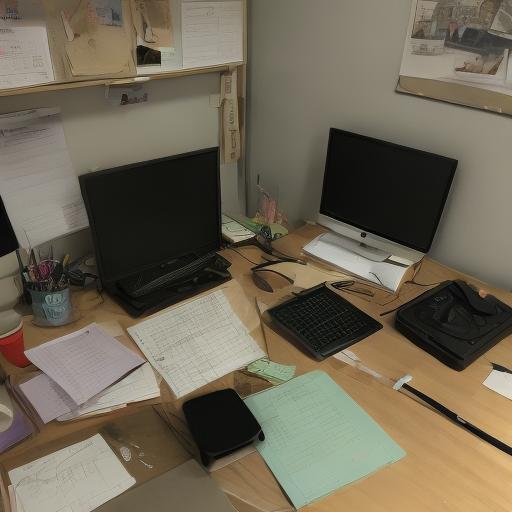}
\includegraphics[width=3cm]{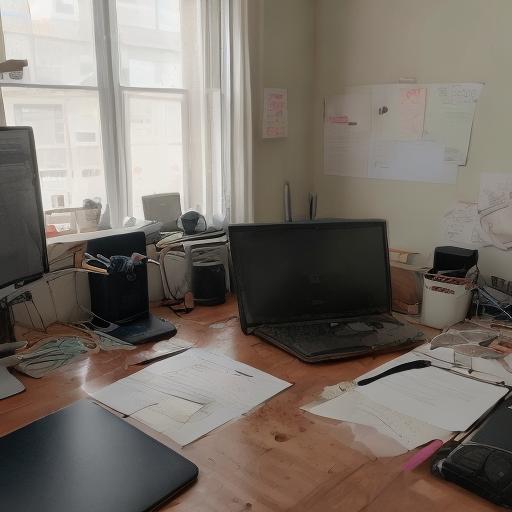}
\includegraphics[width=3cm]{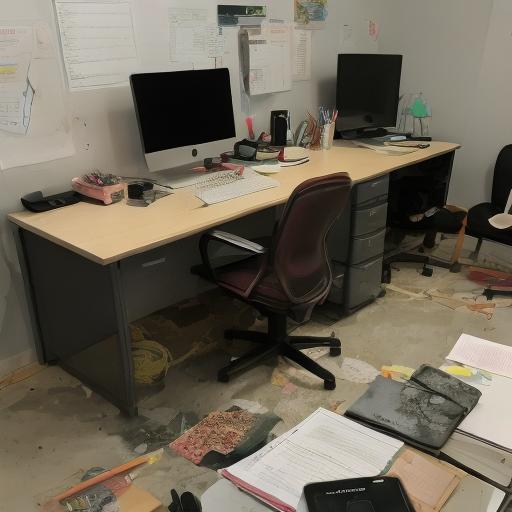}
\includegraphics[width=3cm]{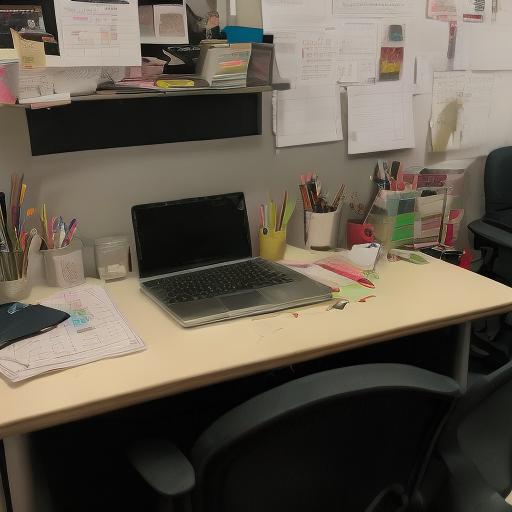}
\end{minipage}
 \begin{center}
\footnotesize CD\qquad\qquad\qquad\qquad\qquad\qquad\qquad\qquad\qquad\qquad\qquad\qquad\qquad\qquad\quad\quad SCott (w/o GAN)
\end{center} 
\caption{Prompt: A picture of a messy desk with an open laptop.}
\label{fig: div desk}
\end{figure}

\begin{figure}[htbp]
\centering
\flushleft 
\begin{minipage}[c]{0.49\textwidth}
\centering
\includegraphics[width=3cm]{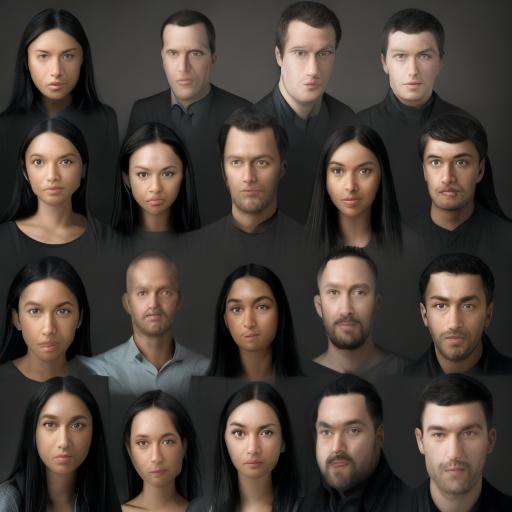}
\includegraphics[width=3cm]{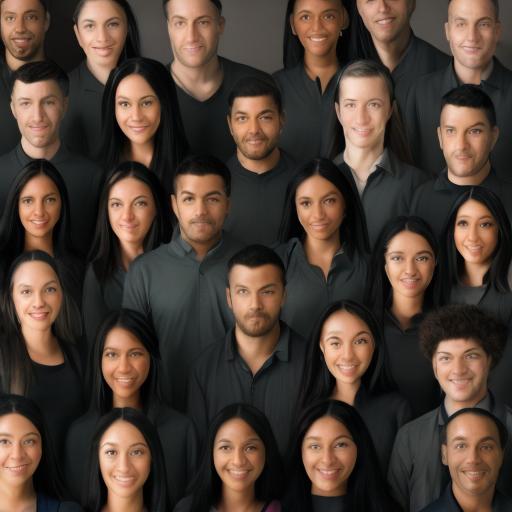} 
\includegraphics[width=3cm]{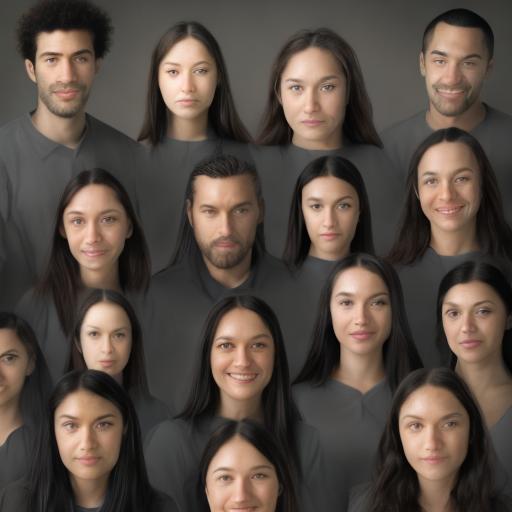}
\includegraphics[width=3cm]{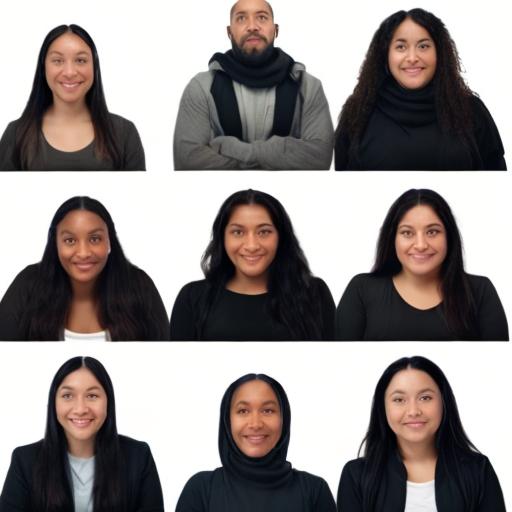}

\end{minipage}
\begin{minipage}[c]{0.49\textwidth}
\centering
\includegraphics[width=3cm]{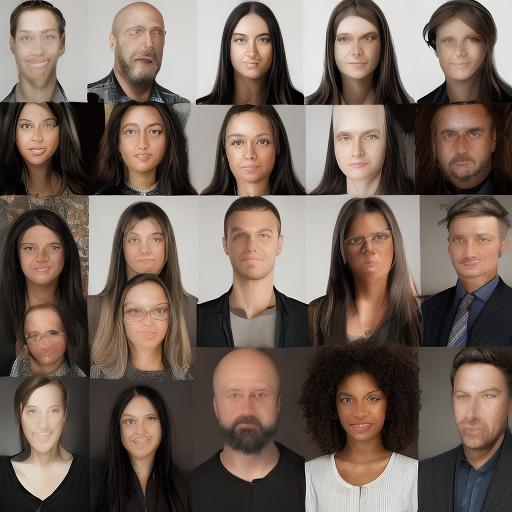}
\includegraphics[width=3cm]{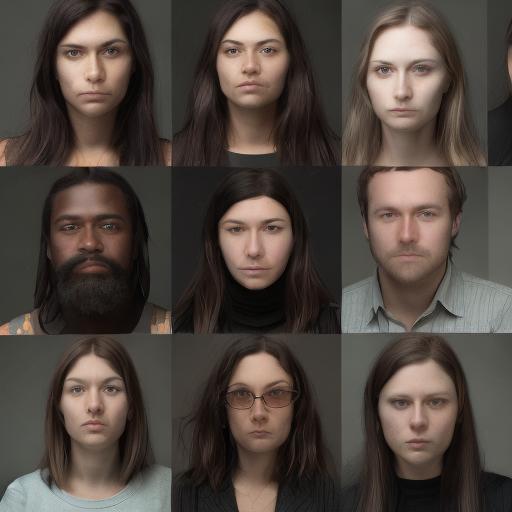} 
\includegraphics[width=3cm]{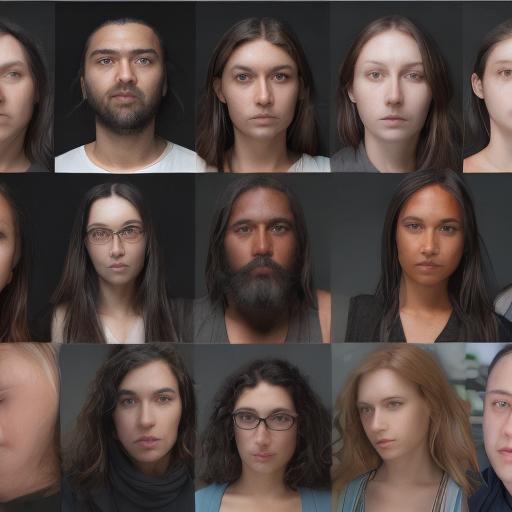}
\includegraphics[width=3cm]{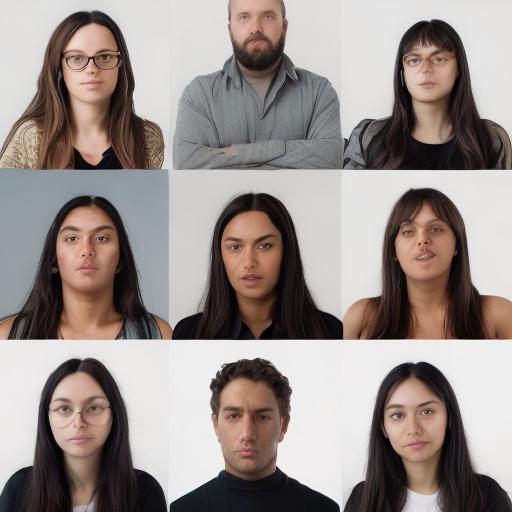}
\end{minipage}
 \begin{center}
\footnotesize CD\qquad\qquad\qquad\qquad\qquad\qquad\qquad\qquad\qquad\qquad\qquad\qquad\qquad\qquad\quad\quad SCott (w/o GAN)
\end{center} 
\caption{Prompt: A high-resolution image or illustration of a diverse group of people facing me, each displaying a distinct range of emotions. The image should be in 8K resolution or provide the highest quality available.}
\label{fig: div high}
\end{figure}

\subsection{Visual Results of Inference Steps}
\label{infer steps results}
As a CM, SCott can improve sample quality as NFE increases. We provide visual results of SCott at inference steps 2, 4, and 8 in \cref{fig: multi infer}. The seeds are the same within the columns. SCott can already generate high-quality samples with 2-step inference. With additional steps, SCott can consistently refine the details of the samples. 
\begin{figure}[htbp]
\centering
\begin{minipage}[t]{0.49\textwidth}
\centering
\rotatebox{90}{\scriptsize{~~~~~~~~~~~~~2-step}}
\includegraphics[width=2cm]{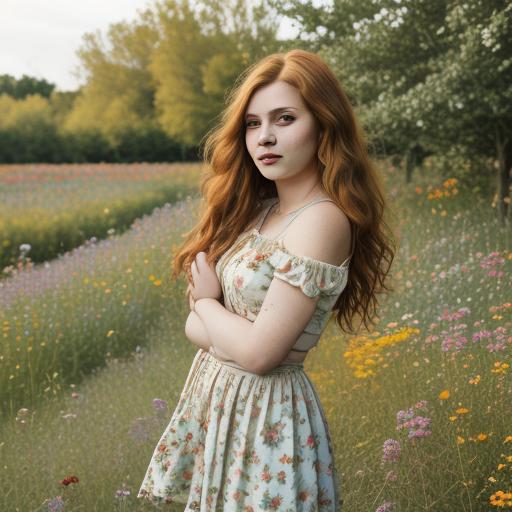}
\includegraphics[width=2cm]{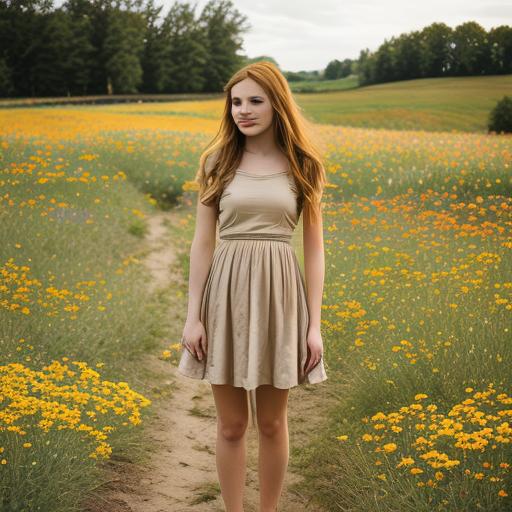}
\includegraphics[width=2cm]{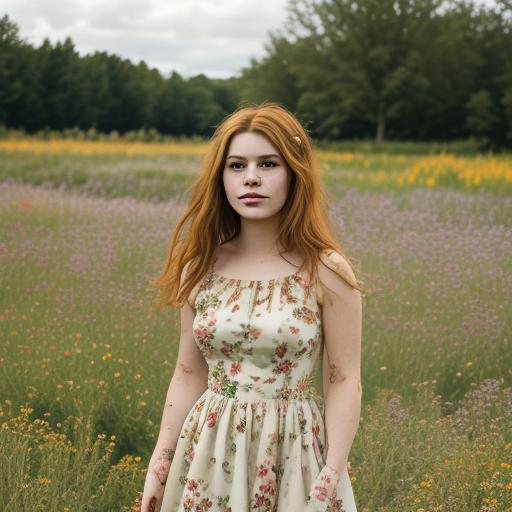}

\rotatebox{90}{\scriptsize{~~~~~~~~~~~~~4-step}}
\includegraphics[width=2cm]{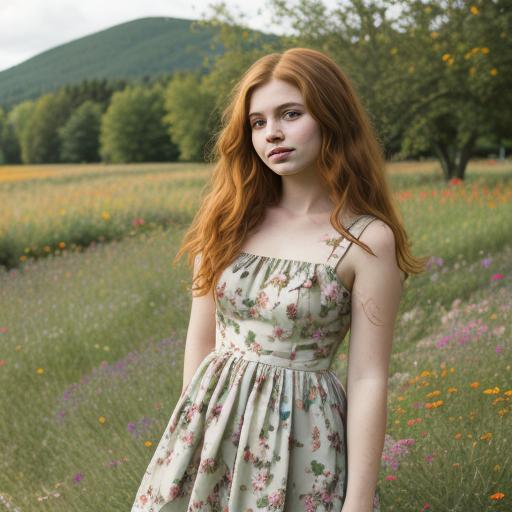}
\includegraphics[width=2cm]{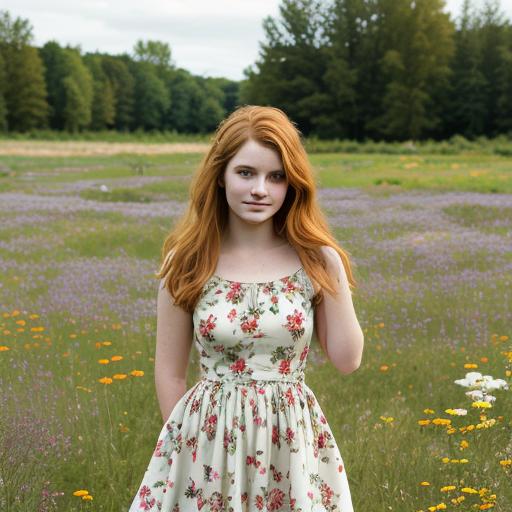}
\includegraphics[width=2cm]{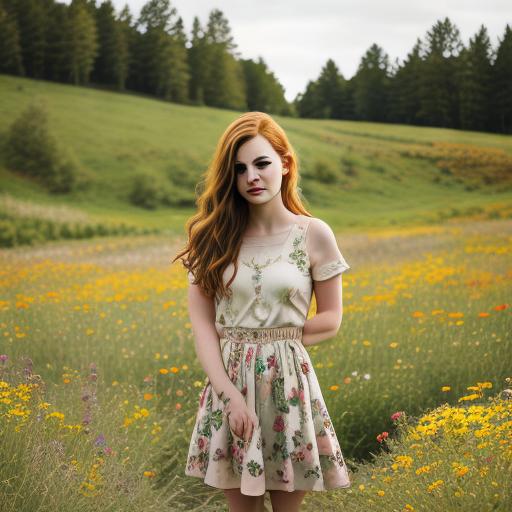}

\rotatebox{90}{\scriptsize{~~~~~~~~~~~~~8-step}}
\includegraphics[width=2cm]{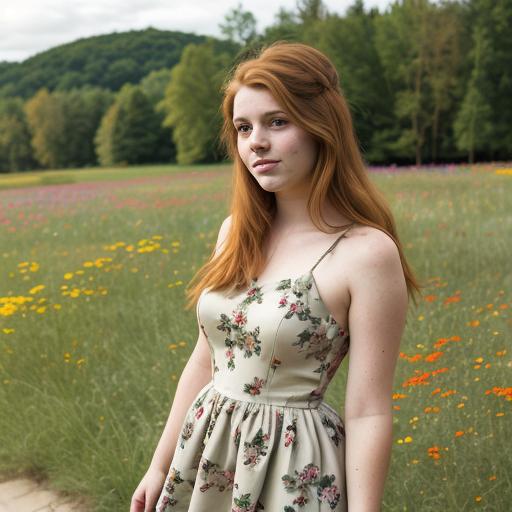}
\includegraphics[width=2cm]{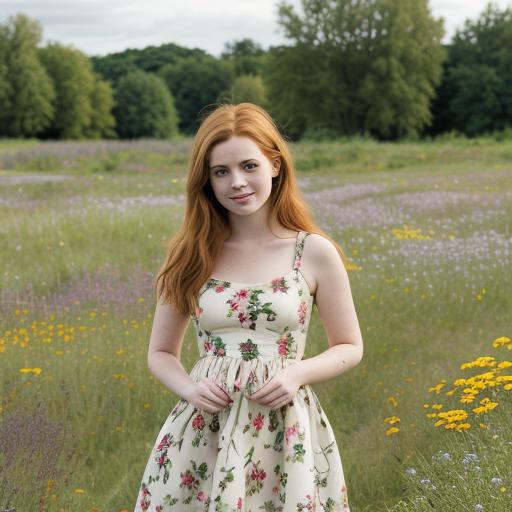}
\includegraphics[width=2cm]{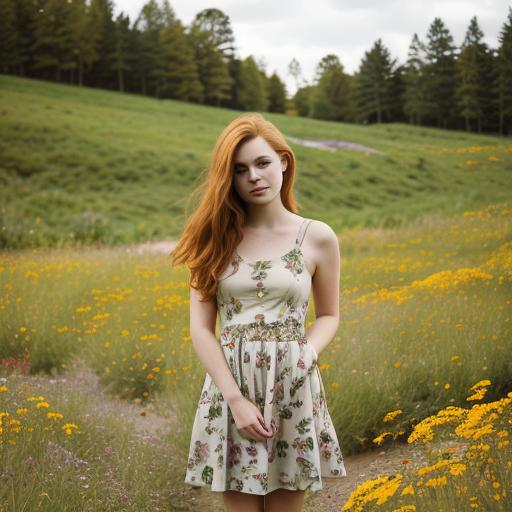}
\begin{center}
\footnotesize "22 year old girl with auburn hair, wearing a skirted dress, very detailed skin texture, drawing on the picture, flowers in the landscape, natural, gentle soul, photojournalism, bokeh."
\end{center}
\end{minipage}
\begin{minipage}[t]{0.49\textwidth}
\centering
\includegraphics[width=2cm]{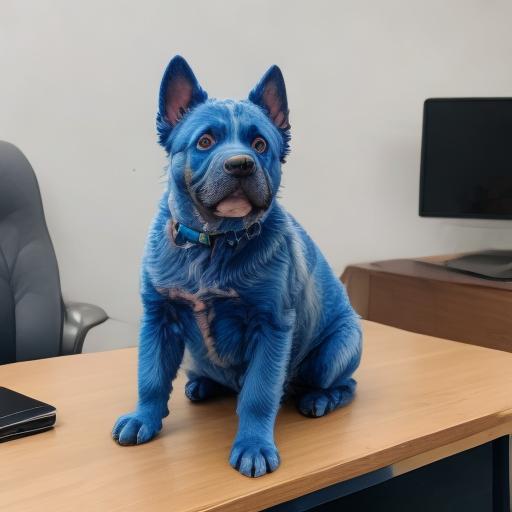}
\includegraphics[width=2cm]{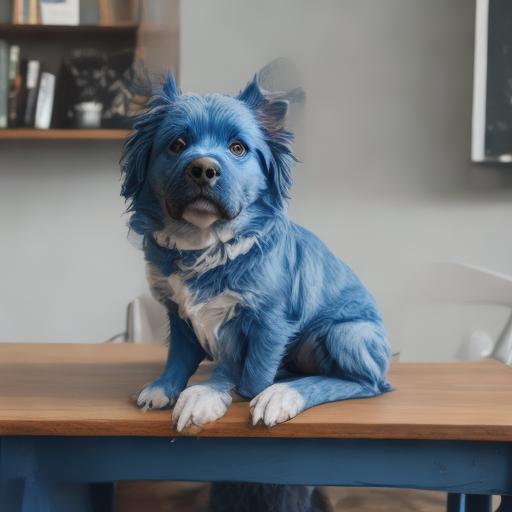}
\includegraphics[width=2cm]{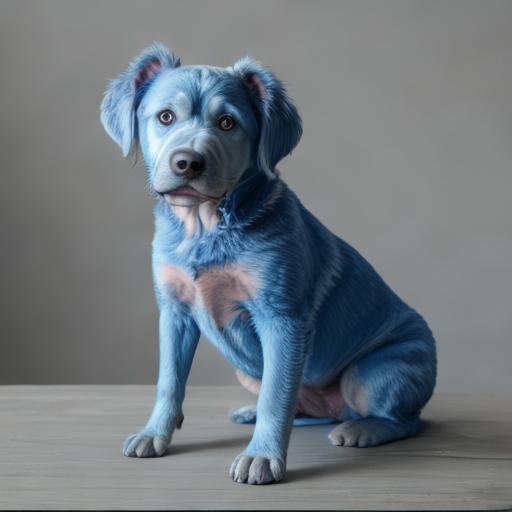}

\includegraphics[width=2cm]{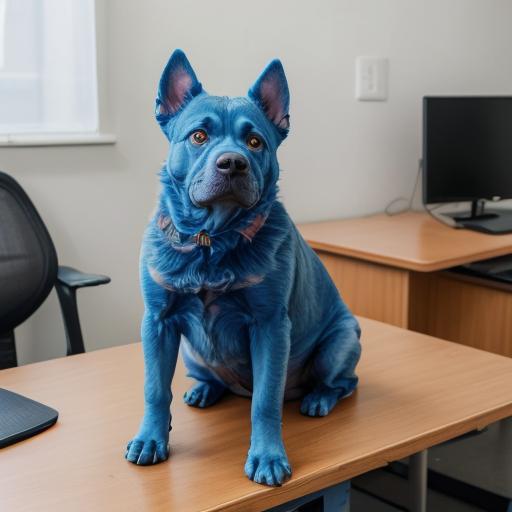}
\includegraphics[width=2cm]{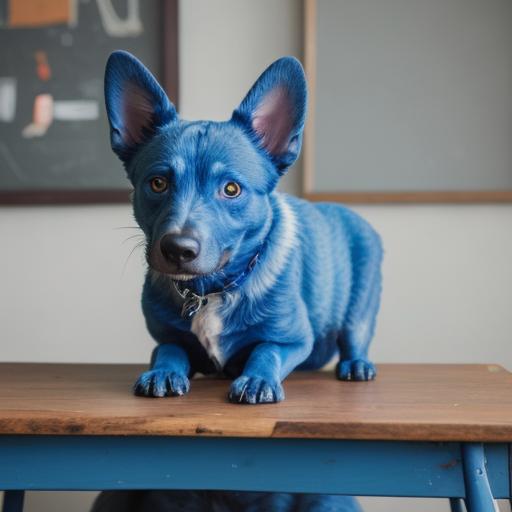}
\includegraphics[width=2cm]{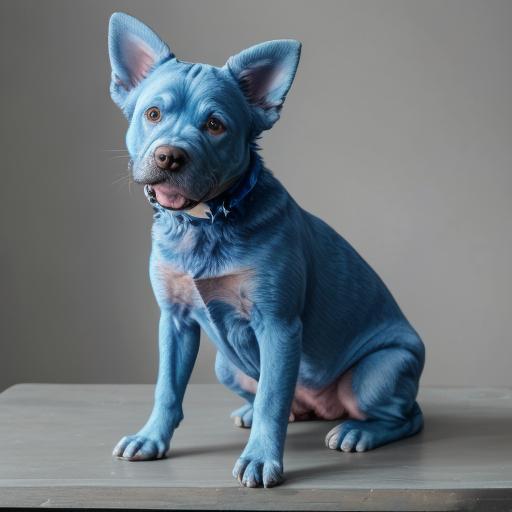}

\includegraphics[width=2cm]{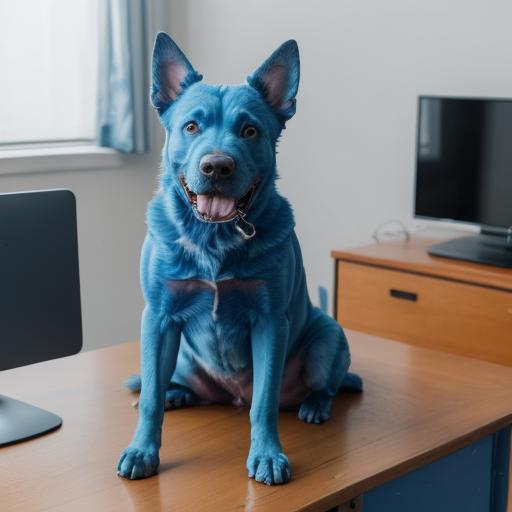}
\includegraphics[width=2cm]{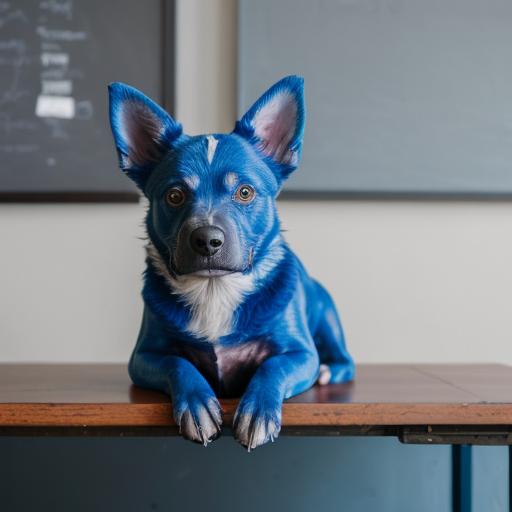}
\includegraphics[width=2cm]{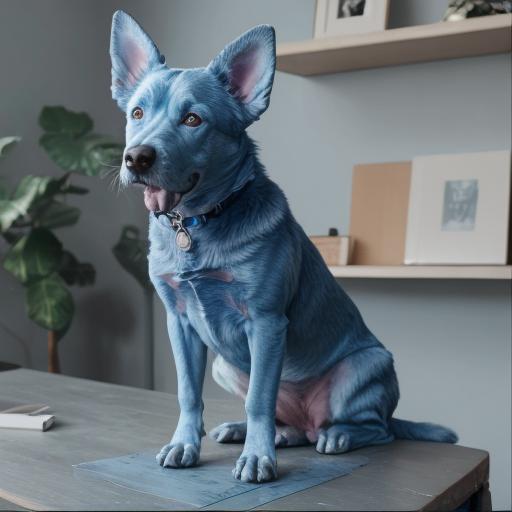}
\begin{center}
\footnotesize "A blue colored dog sitting on a desk."
\end{center}
\end{minipage}
\begin{center}
\end{center} 
\end{figure}

\subsection{Additional Visual Results}\label{more comparison}
\cref{fig: dog,fig: girl,fig: robot,fig: fox,fig: man} provide more visual results generated by InstaFlow-0.9B, LCM, and our SCott. 
\begin{figure}[htbp]
\centering
\begin{minipage}[t]{0.49\textwidth}
\centering
\includegraphics[width=3cm]{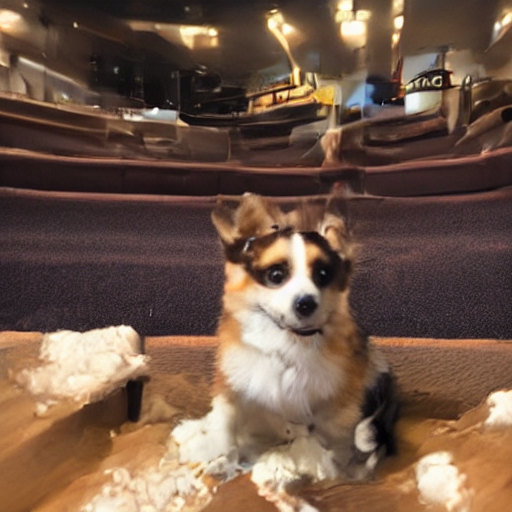}
\includegraphics[width=3cm]{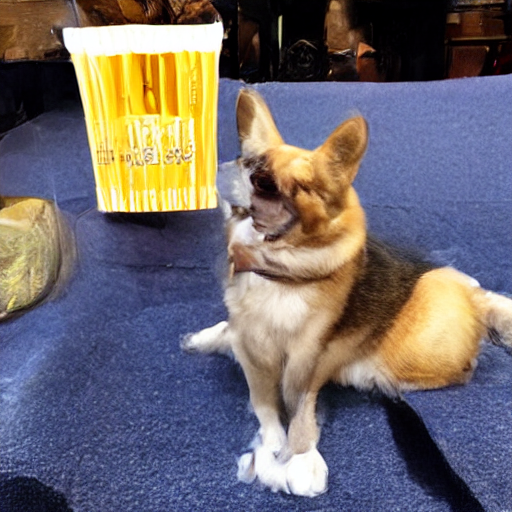}
\includegraphics[width=3cm]{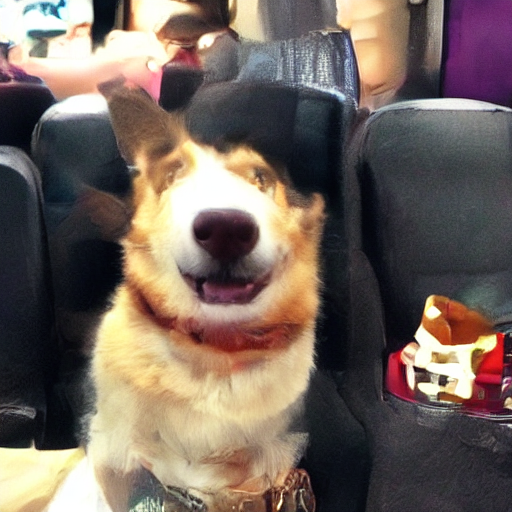}
\includegraphics[width=3cm]{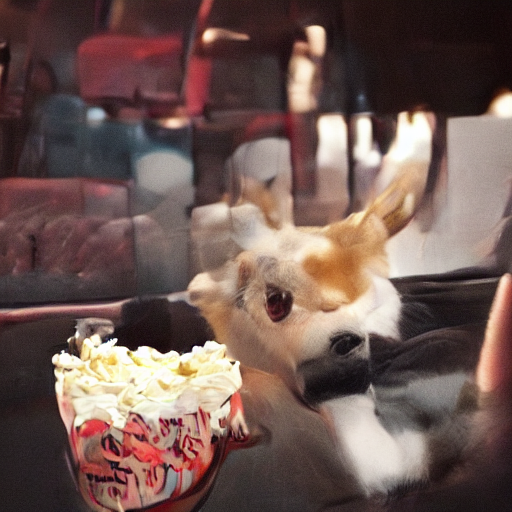}
\end{minipage}
\begin{minipage}[t]{0.49\textwidth}
\centering
\includegraphics[width=3cm]{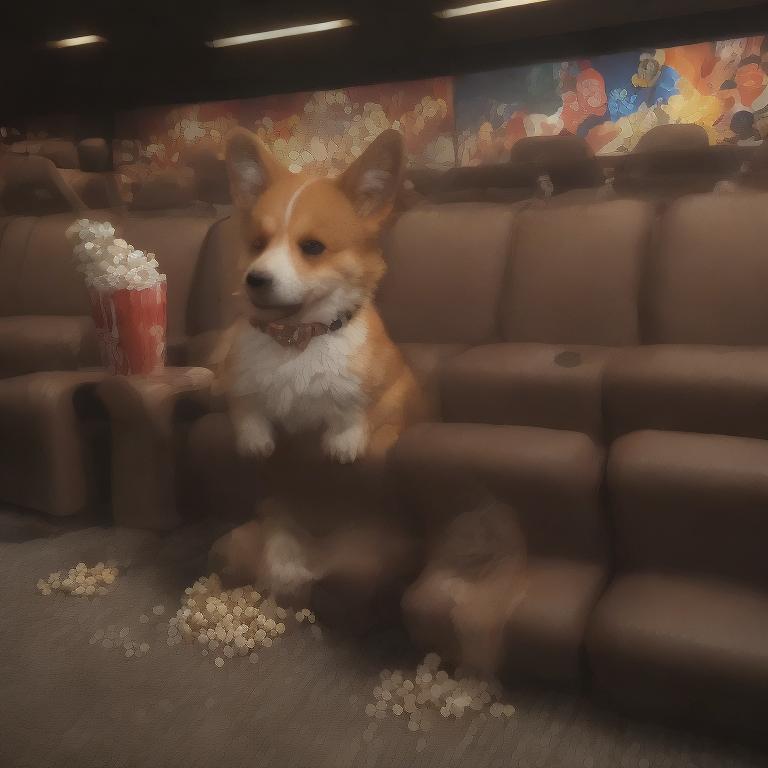}
\includegraphics[width=3cm]{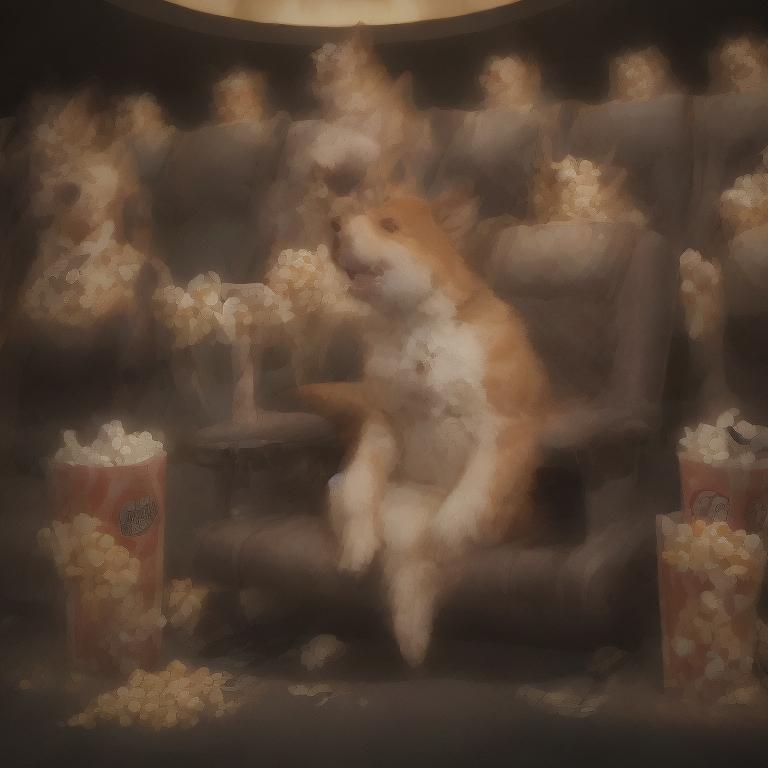}
\includegraphics[width=3cm]{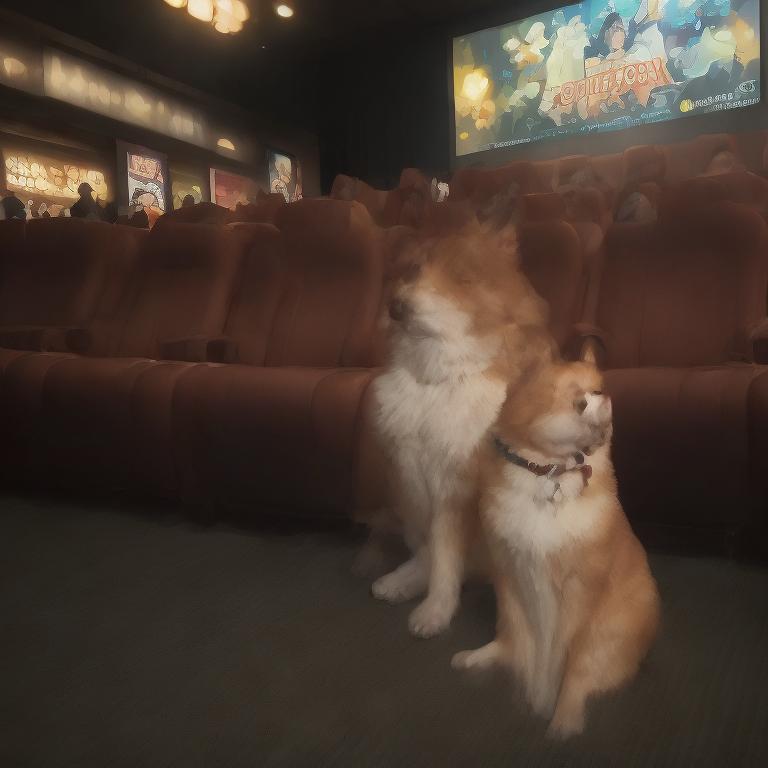}
\includegraphics[width=3cm]{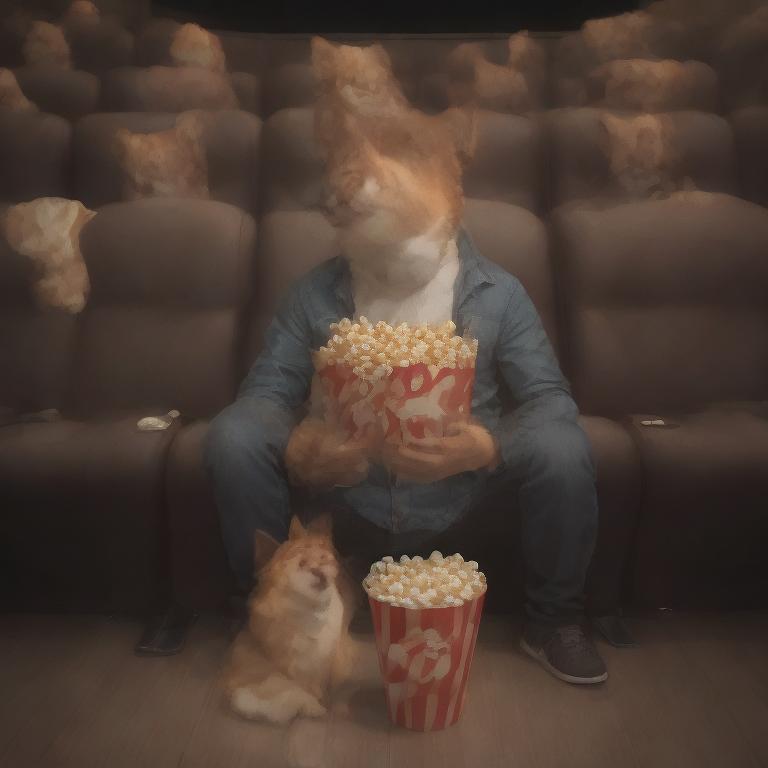}
\end{minipage}
 \begin{center}
\footnotesize \qquad InstaFlow (1 step)\qquad\qquad\qquad\qquad\qquad\qquad\qquad\qquad\qquad\qquad\qquad\qquad\qquad\qquad LCM (4 steps)
\end{center} 
\centering
\flushleft 
\begin{minipage}[c]{0.49\textwidth}
\centering
\includegraphics[width=3cm]{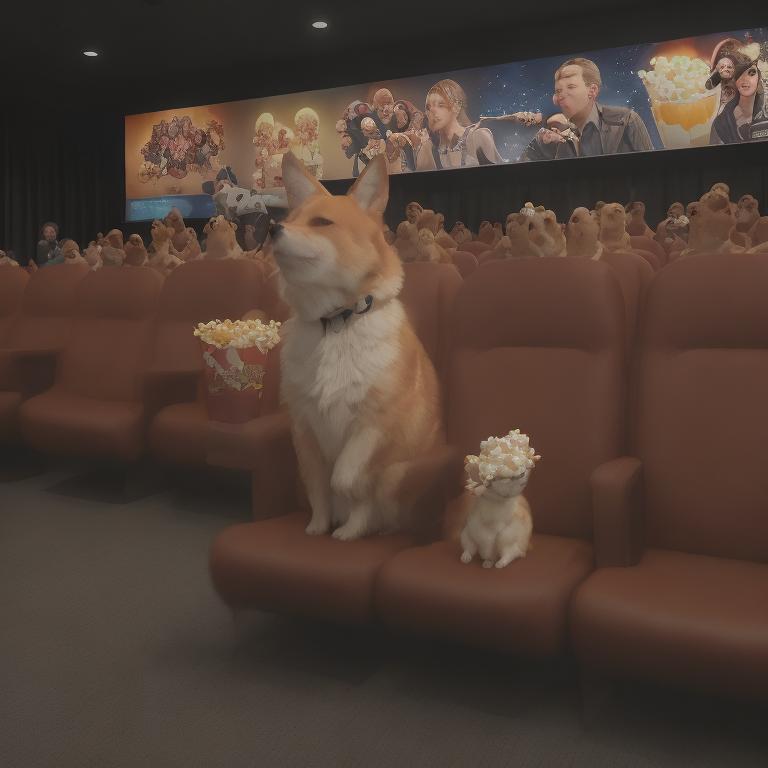}
\includegraphics[width=3cm]{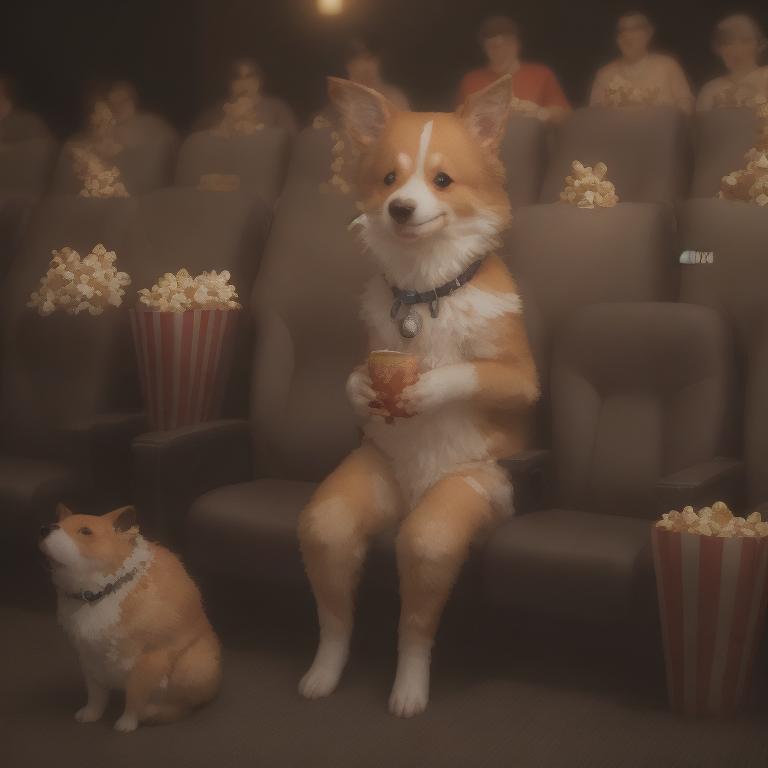} 
\includegraphics[width=3cm]{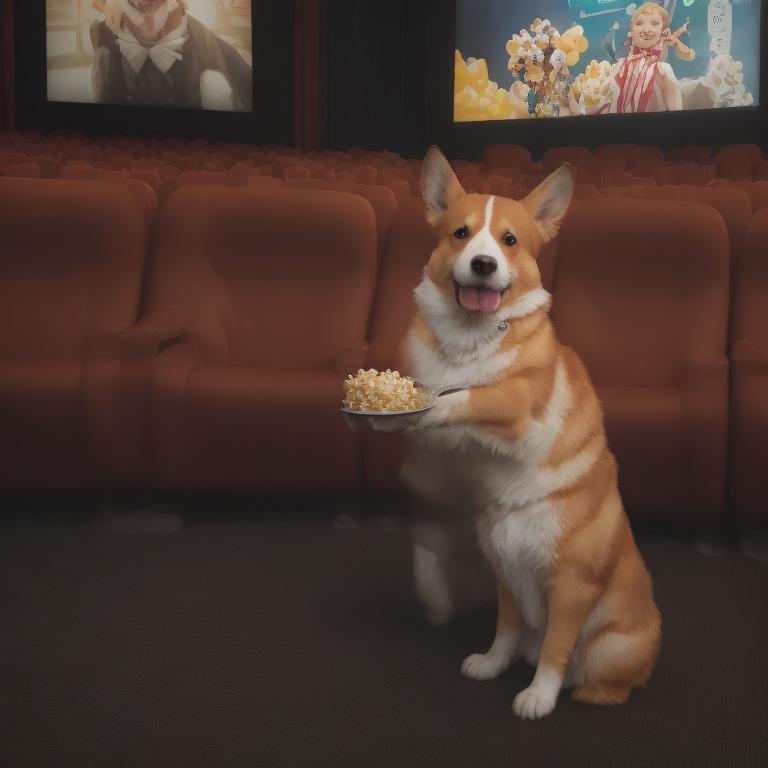}
\includegraphics[width=3cm]{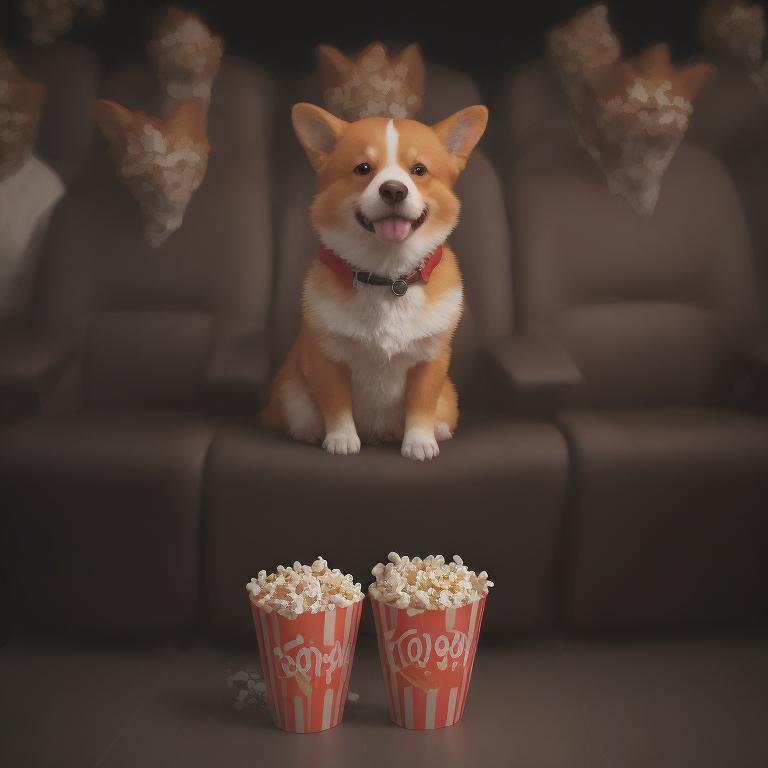}

\end{minipage}
\begin{minipage}[c]{0.49\textwidth}
\centering
\includegraphics[width=3cm]{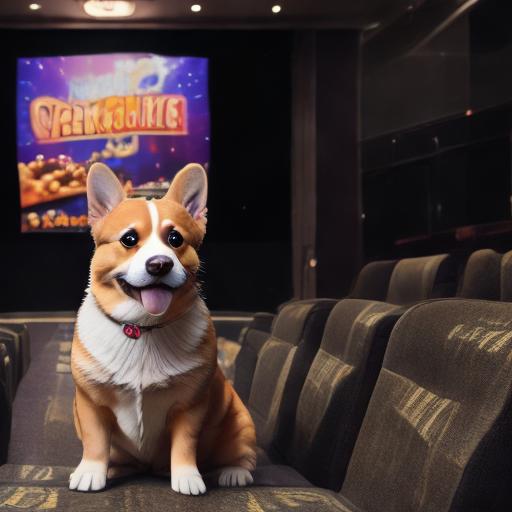}
\includegraphics[width=3cm]{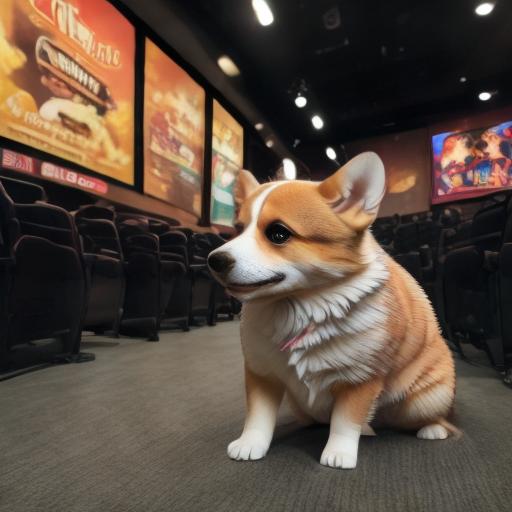} 
\includegraphics[width=3cm]{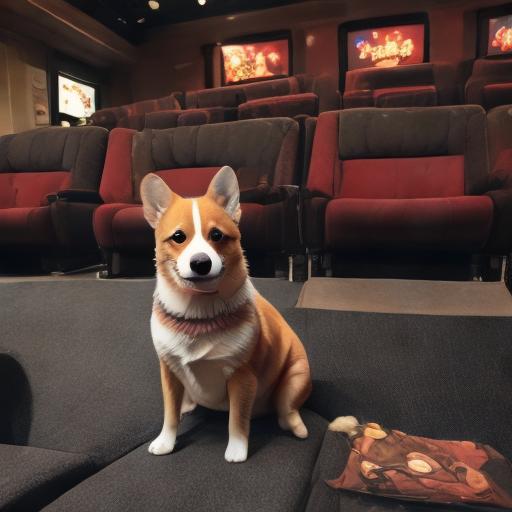}
\includegraphics[width=3cm]{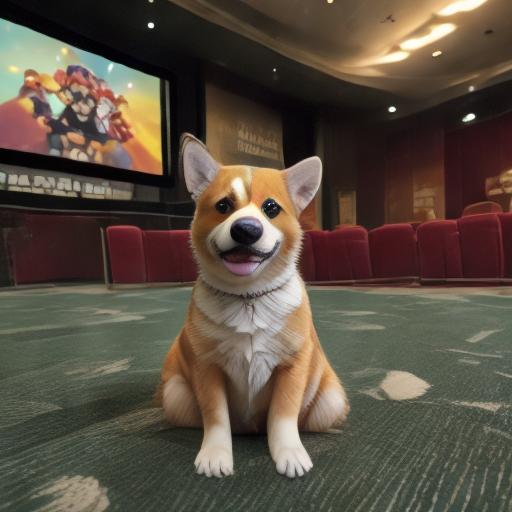}
\end{minipage}
 \begin{center}
\footnotesize \qquad LCM (8-steps)\qquad\qquad\qquad\qquad\qquad\qquad\qquad\qquad\qquad\qquad\qquad\qquad\qquad\qquad SCott (2 steps)
\end{center} 
\caption{Prompt: A small corgi sitting in a movie theater eating popcorn, unreal engine.}
\label{fig: dog}
\end{figure}

\begin{figure}[htbp]
\centering
\begin{minipage}[t]{0.49\textwidth}
\centering
\includegraphics[width=3cm]{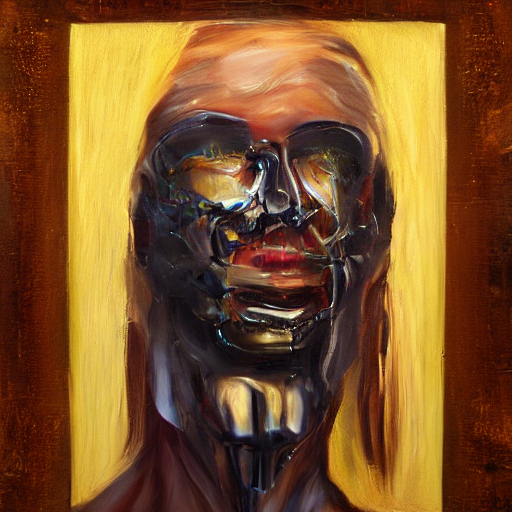}
\includegraphics[width=3cm]{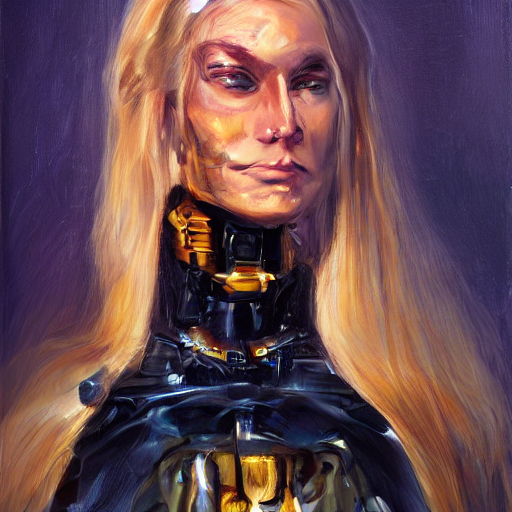}
\includegraphics[width=3cm]{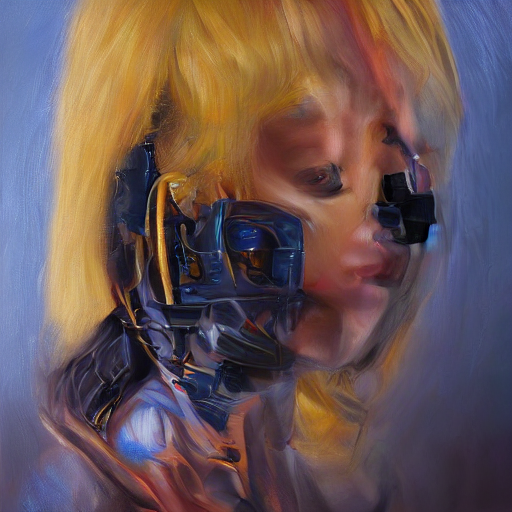}
\includegraphics[width=3cm]{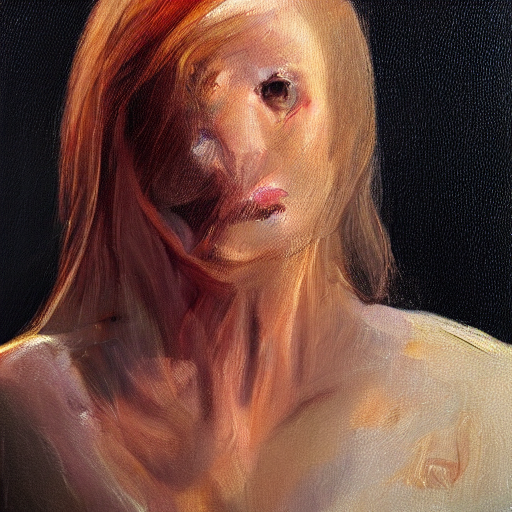}
\end{minipage}
\begin{minipage}[t]{0.49\textwidth}
\centering
\includegraphics[width=3cm]{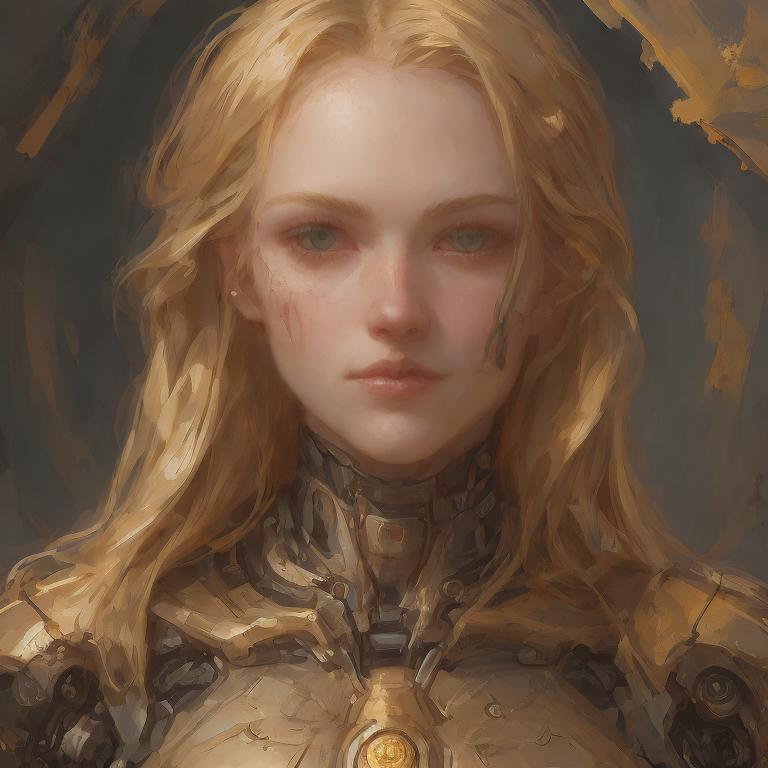}
\includegraphics[width=3cm]{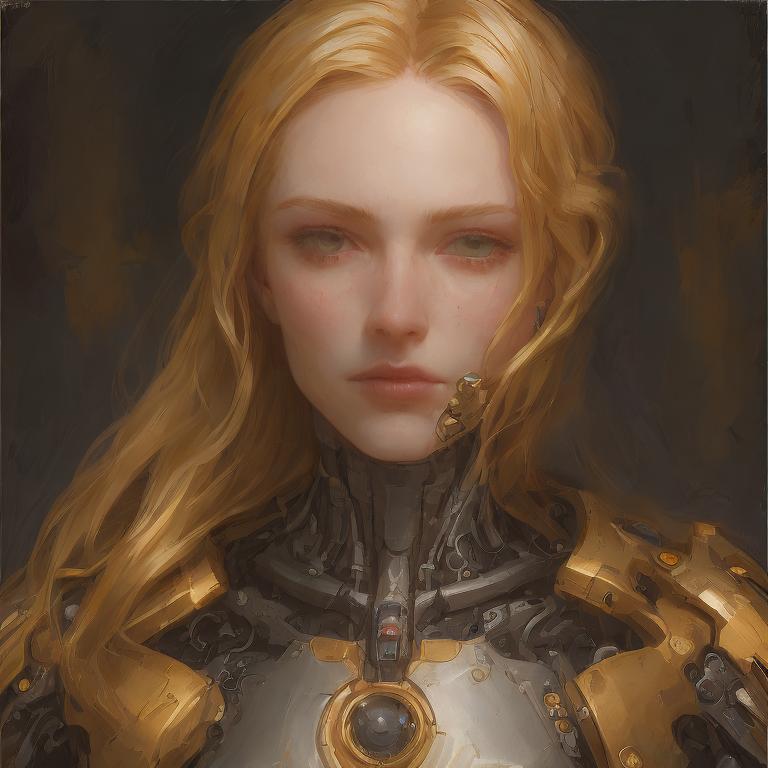}
\includegraphics[width=3cm]{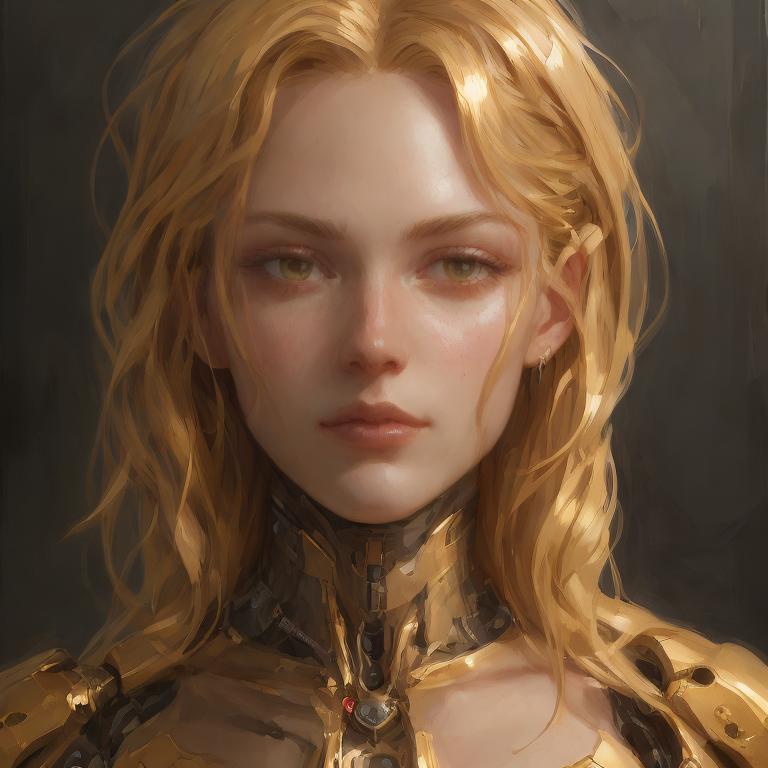}
\includegraphics[width=3cm]{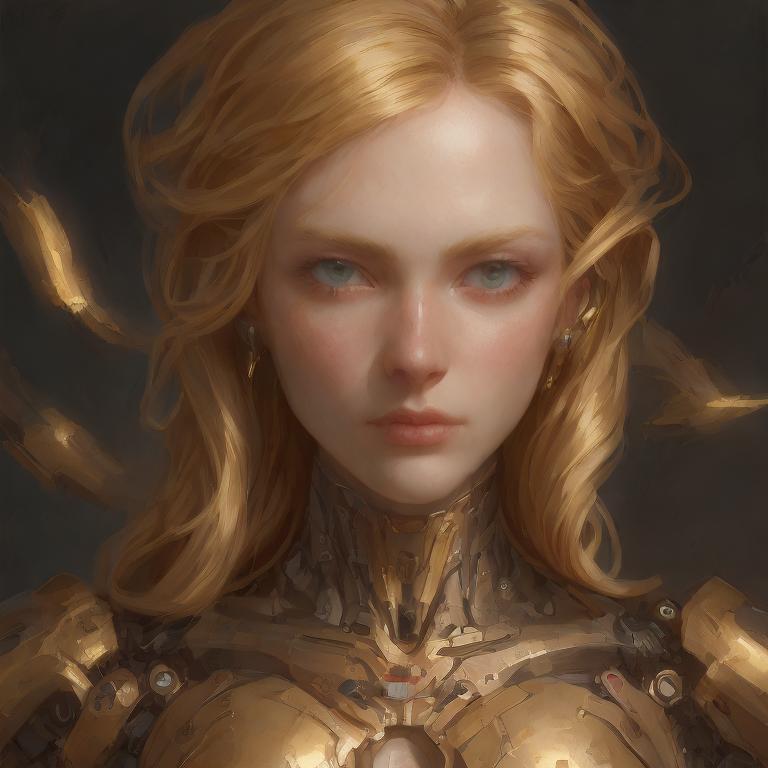}
\end{minipage}
 \begin{center}
\footnotesize \qquad InstaFlow (1 step)\qquad\qquad\qquad\qquad\qquad\qquad\qquad\qquad\qquad\qquad\qquad\qquad\qquad\qquad LCM (4 steps)
\end{center} 
\centering
\flushleft 
\begin{minipage}[c]{0.49\textwidth}
\centering
\includegraphics[width=3cm]{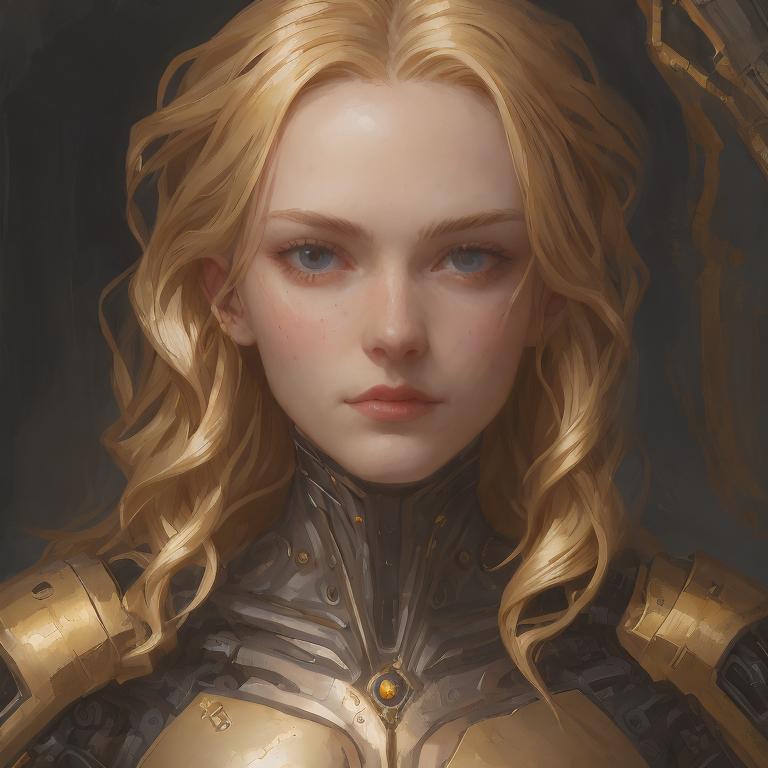}
\includegraphics[width=3cm]{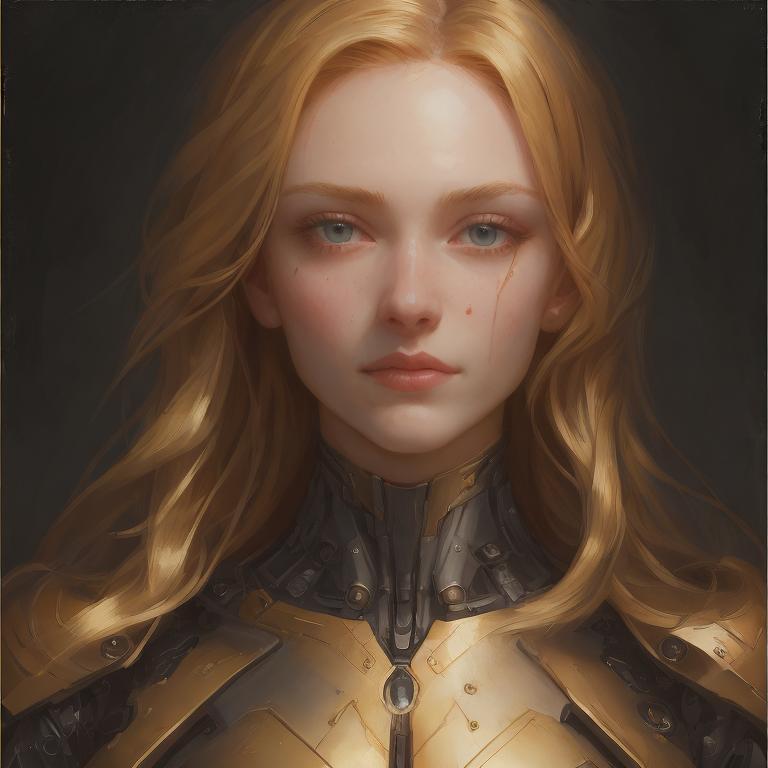} 
\includegraphics[width=3cm]{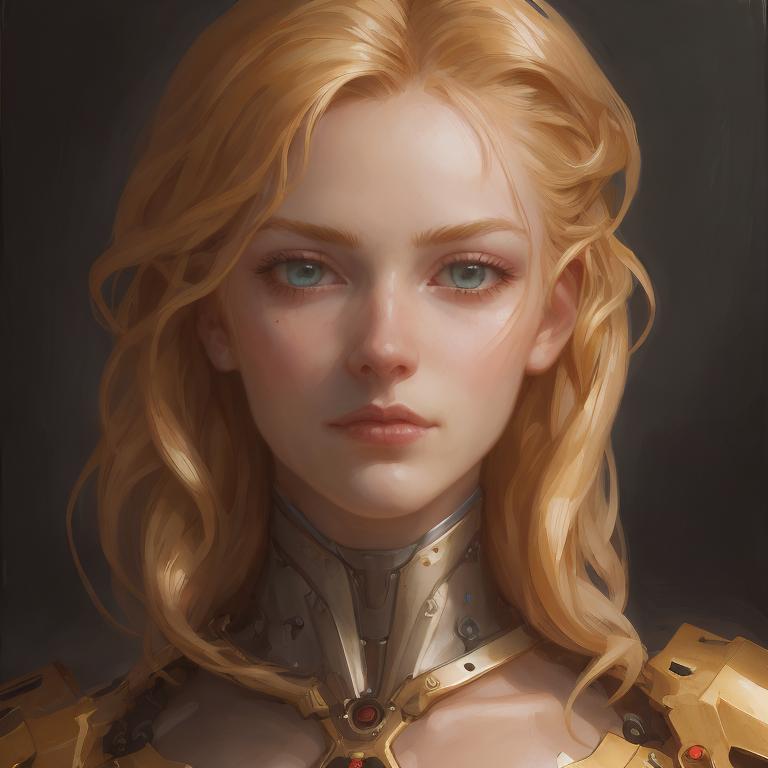}
\includegraphics[width=3cm]{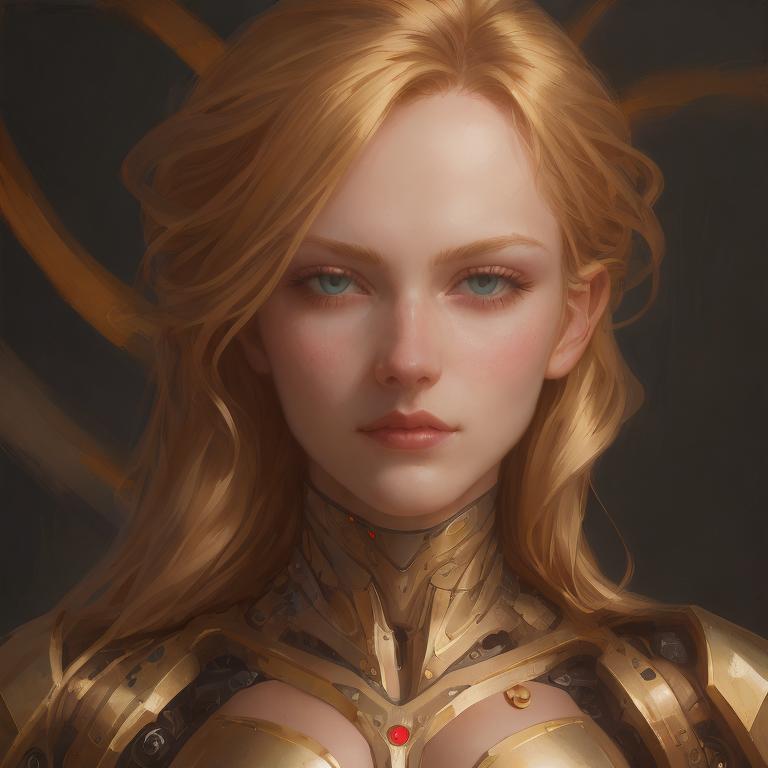}

\end{minipage}
\begin{minipage}[c]{0.49\textwidth}
\centering
\includegraphics[width=3cm]{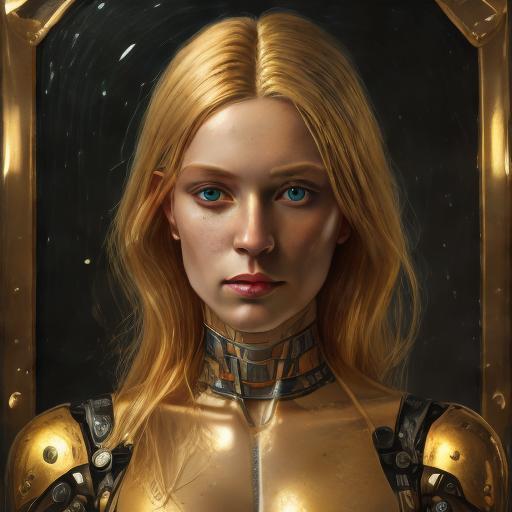}
\includegraphics[width=3cm]{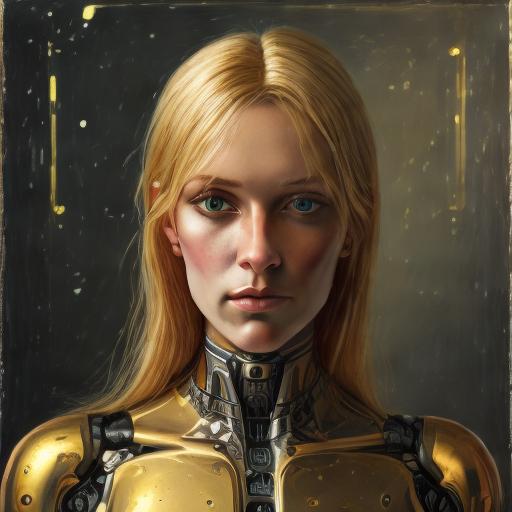} 
\includegraphics[width=3cm]{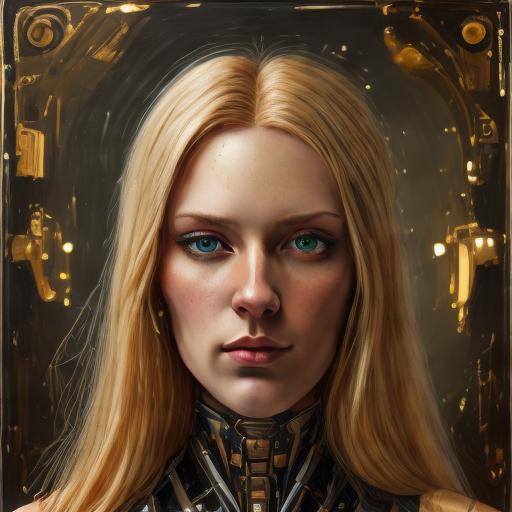}
\includegraphics[width=3cm]{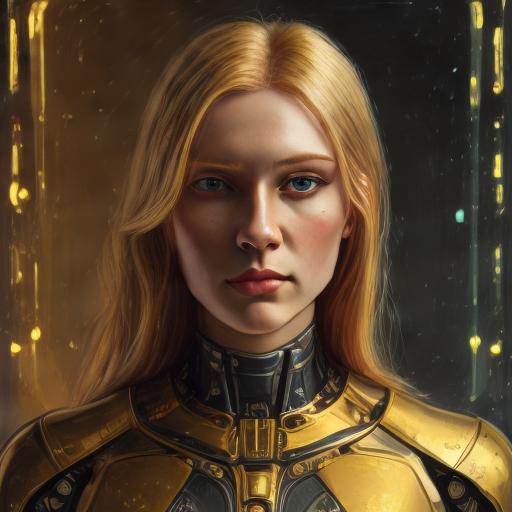}
\end{minipage}
 \begin{center}
\footnotesize \qquad LCM (8-steps)\qquad\qquad\qquad\qquad\qquad\qquad\qquad\qquad\qquad\qquad\qquad\qquad\qquad\qquad SCott (2 steps)
\end{center} 
\caption{Prompt: Self-portrait oil painting, a beautiful cyborg with golden hair, 8k.}
\label{fig: girl}
\end{figure}

\begin{figure}[htbp]
\centering
\begin{minipage}[t]{0.49\textwidth}
\centering
\includegraphics[width=3cm]{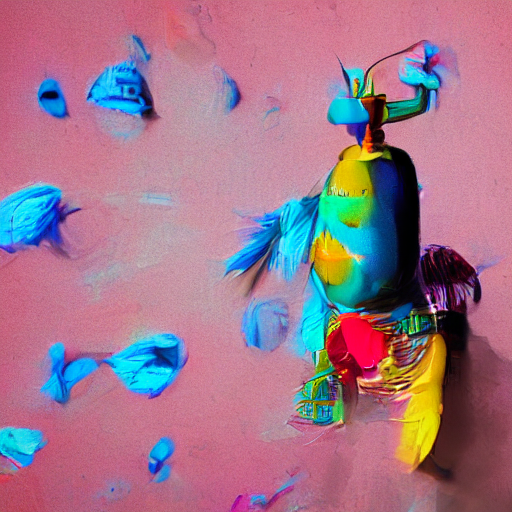}
\includegraphics[width=3cm]{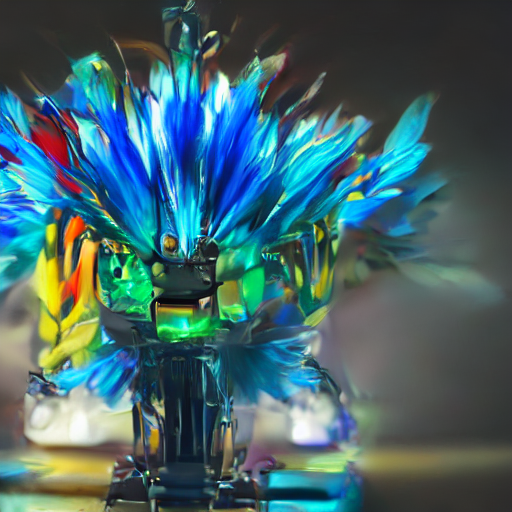}
\includegraphics[width=3cm]{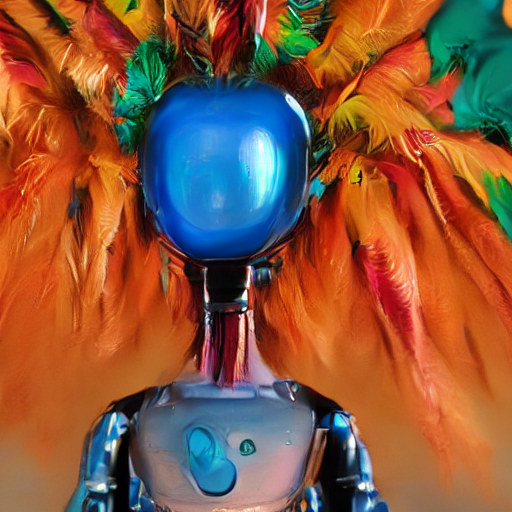}
\includegraphics[width=3cm]{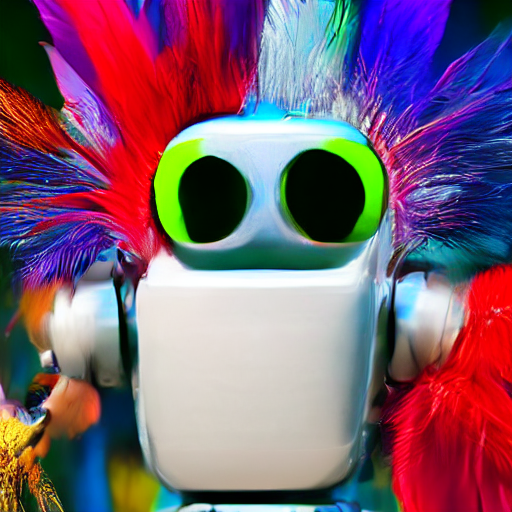}
\end{minipage}
\begin{minipage}[t]{0.49\textwidth}
\centering
\includegraphics[width=3cm]{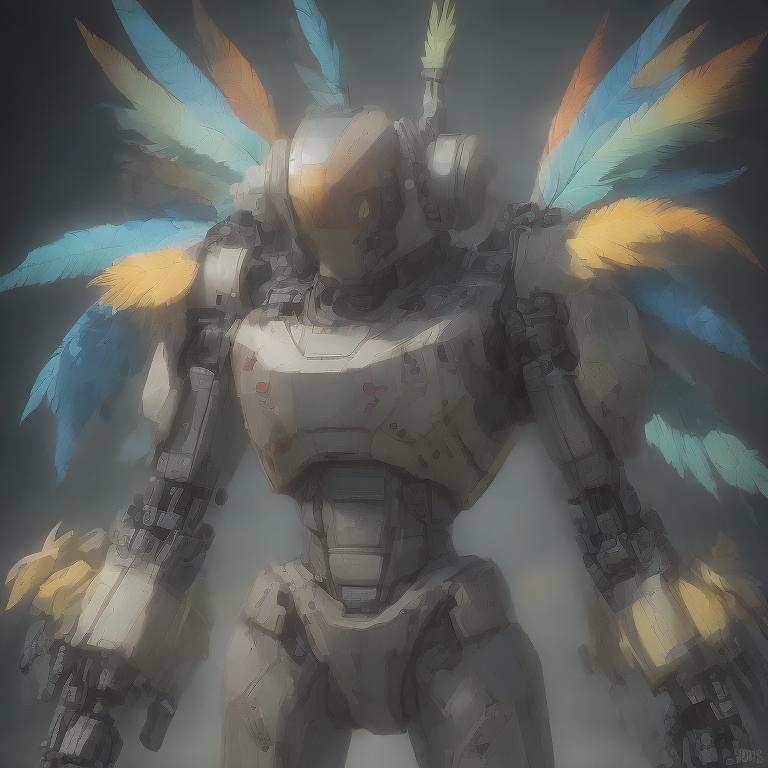}
\includegraphics[width=3cm]{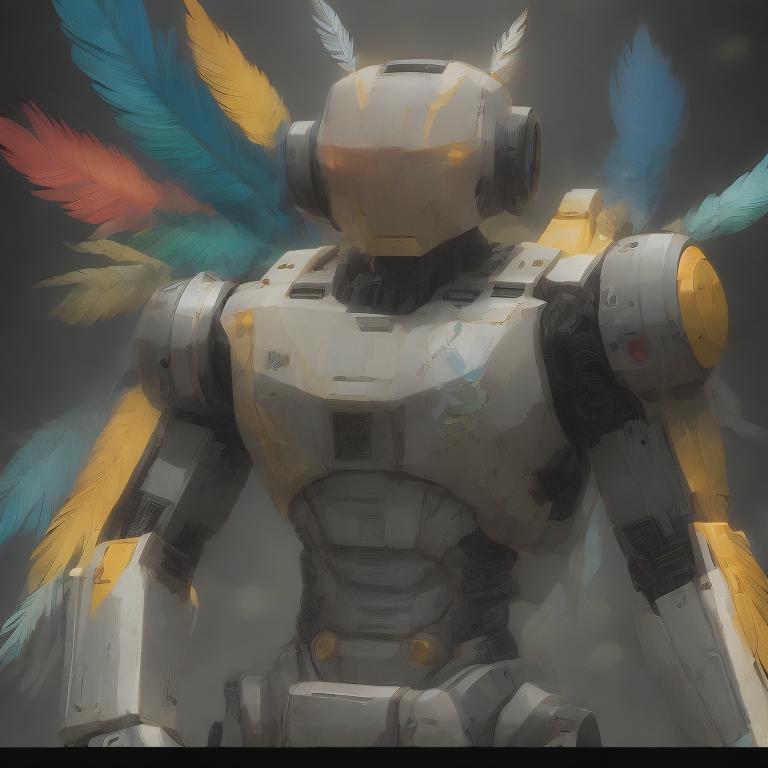}
\includegraphics[width=3cm]{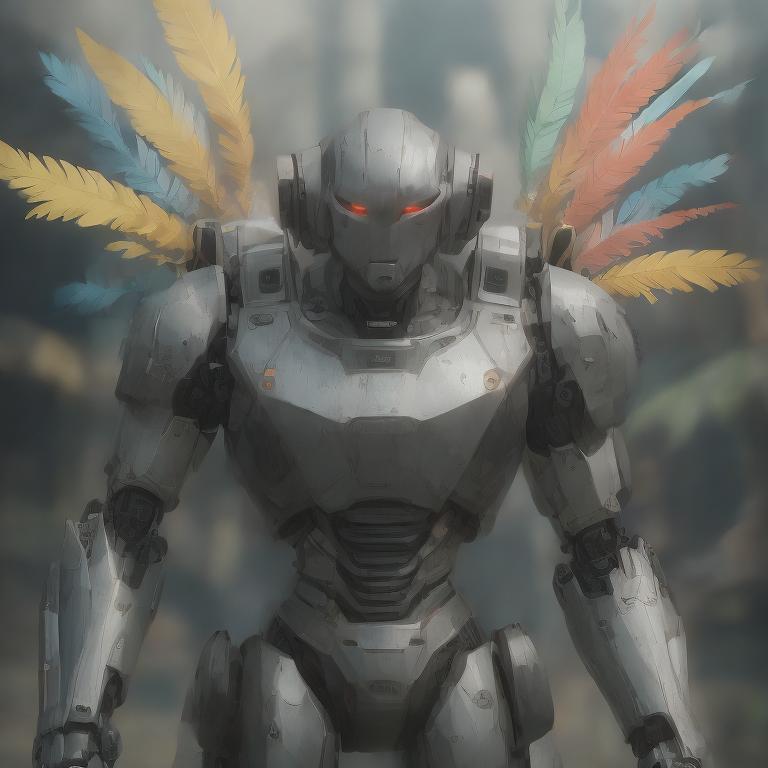}
\includegraphics[width=3cm]{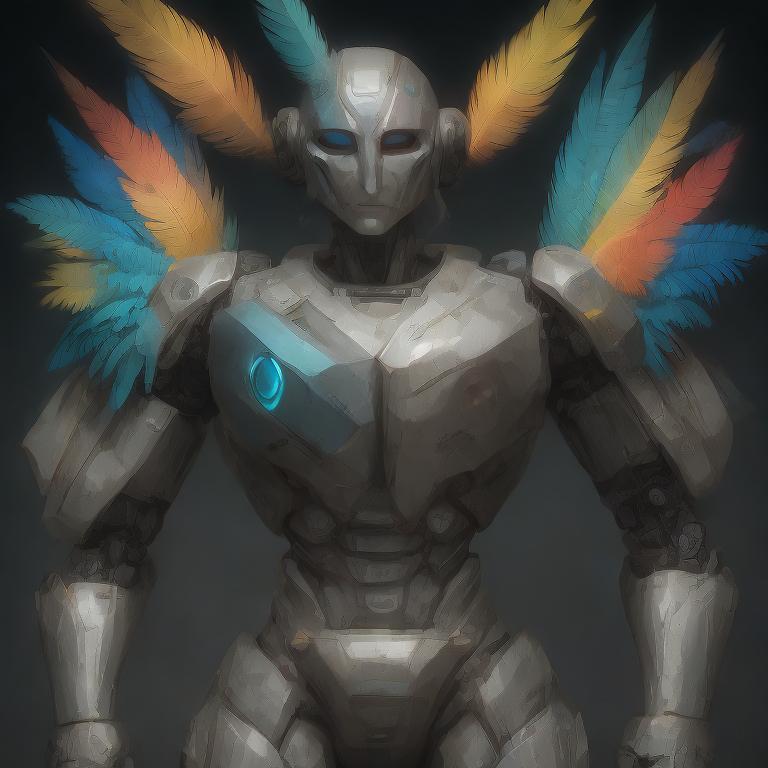}
\end{minipage}
 \begin{center}
\footnotesize \qquad InstaFlow (1 step)\qquad\qquad\qquad\qquad\qquad\qquad\qquad\qquad\qquad\qquad\qquad\qquad\qquad\qquad LCM (4 steps)
\end{center} 
\centering
\flushleft 
\begin{minipage}[c]{0.49\textwidth}
\centering
\includegraphics[width=3cm]{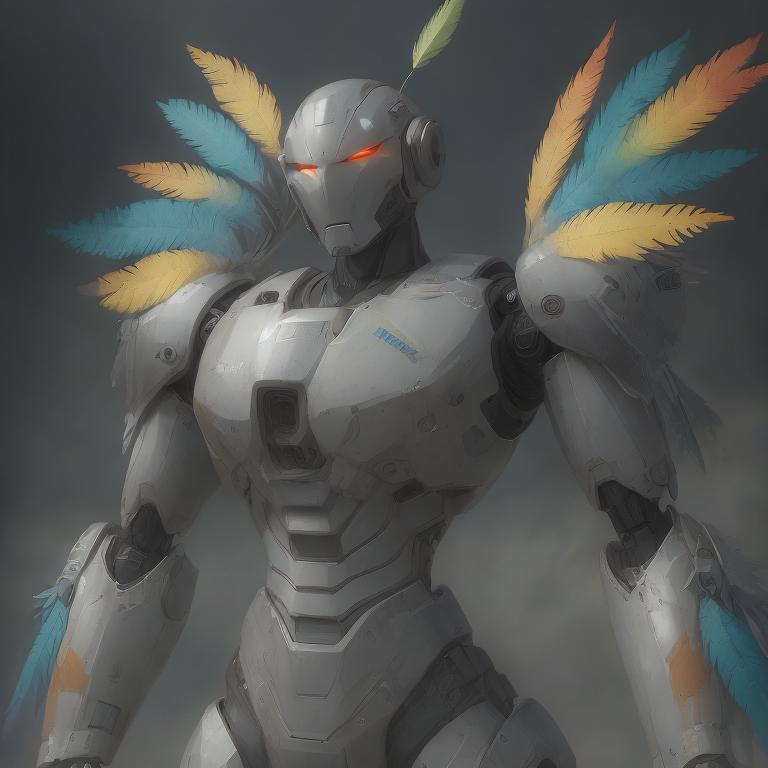}
\includegraphics[width=3cm]{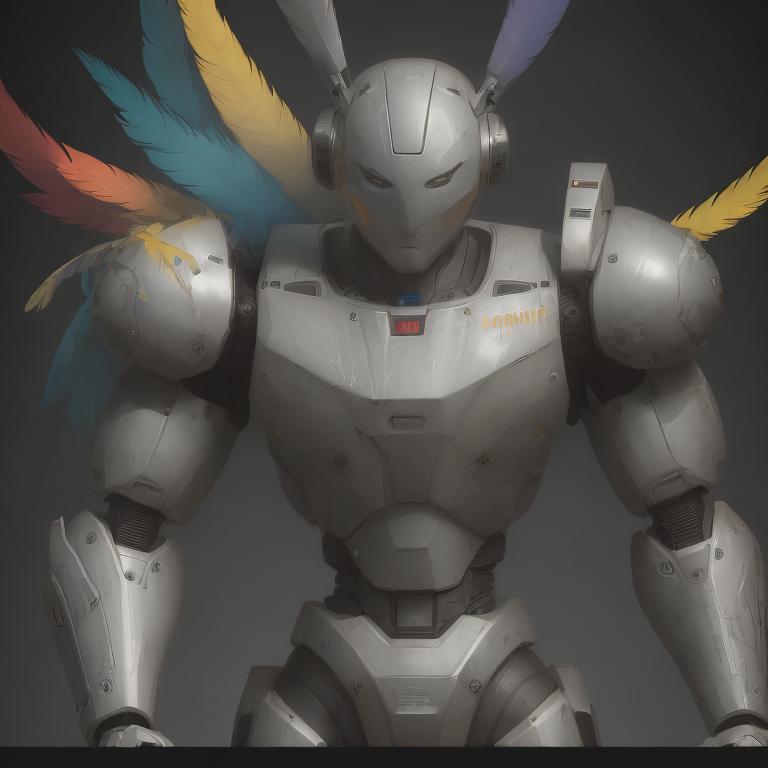} 
\includegraphics[width=3cm]{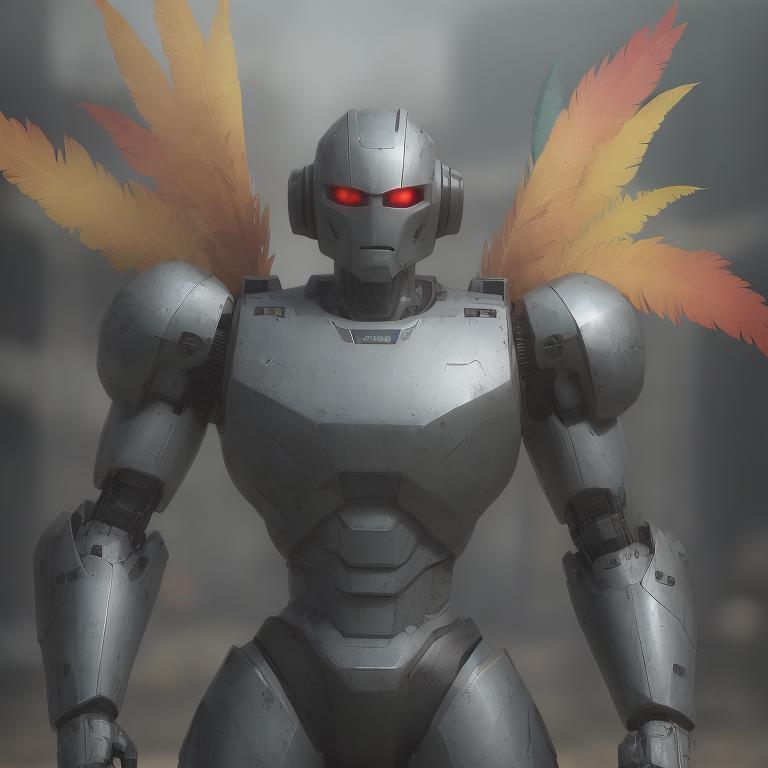}
\includegraphics[width=3cm]{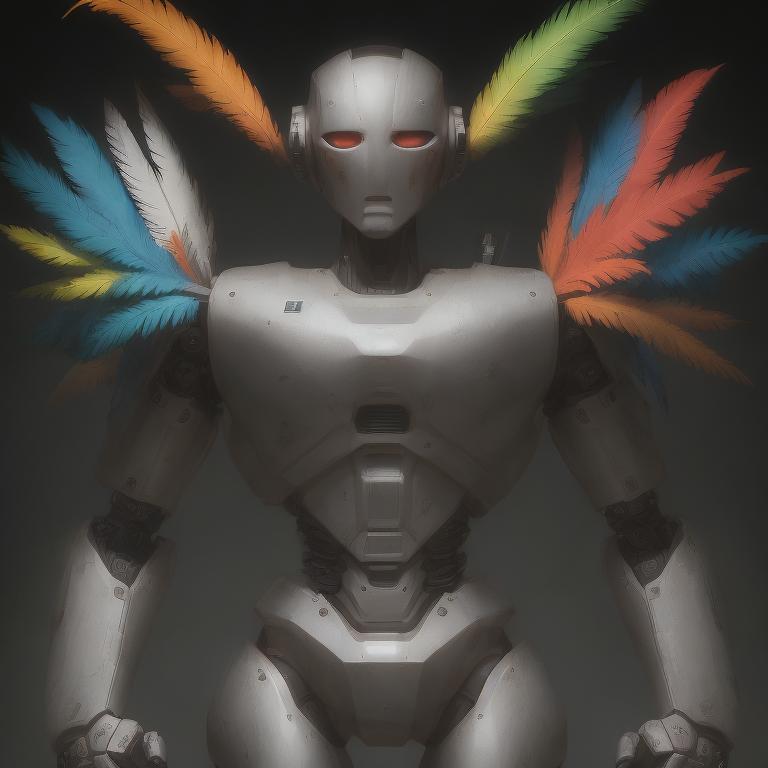}

\end{minipage}
\begin{minipage}[c]{0.49\textwidth}
\centering
\includegraphics[width=3cm]{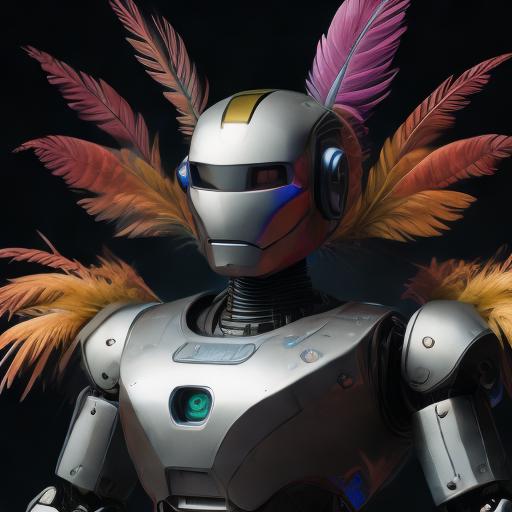}
\includegraphics[width=3cm]{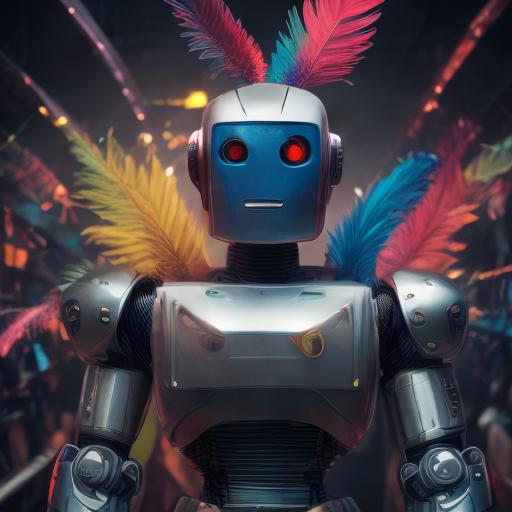} 
\includegraphics[width=3cm]{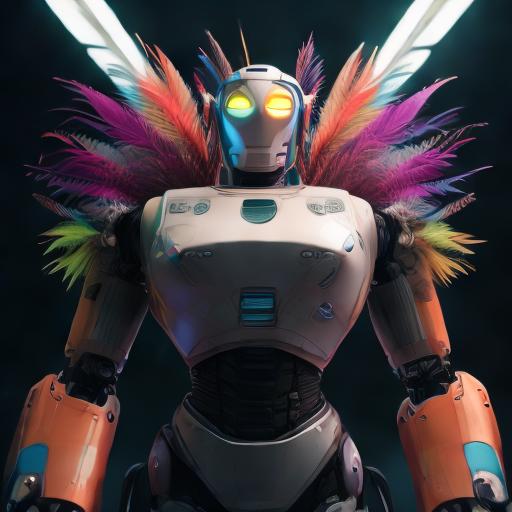}
\includegraphics[width=3cm]{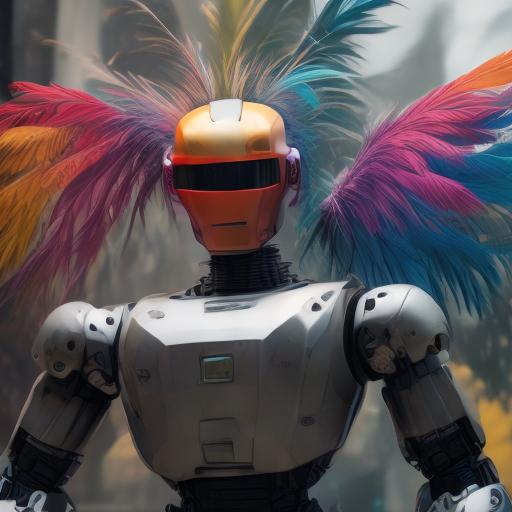}
\end{minipage}
 \begin{center}
\footnotesize \qquad LCM (8-steps)\qquad\qquad\qquad\qquad\qquad\qquad\qquad\qquad\qquad\qquad\qquad\qquad\qquad\qquad SCott (2 steps)
\end{center} 
\caption{Prompt: A cinematic shot of robot with colorful feathers.}
\label{fig: robot}
\end{figure}

\begin{figure}[htbp]
\centering
\begin{minipage}[t]{0.49\textwidth}
\centering
\includegraphics[width=3cm]{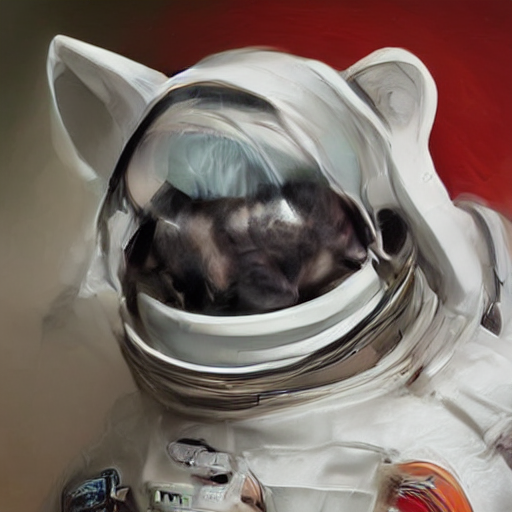}
\includegraphics[width=3cm]{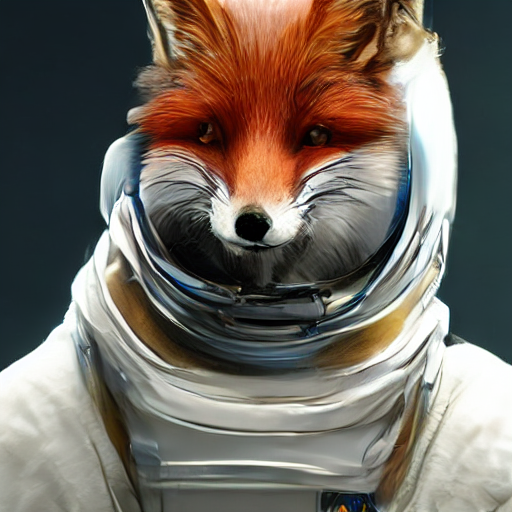}
\includegraphics[width=3cm]{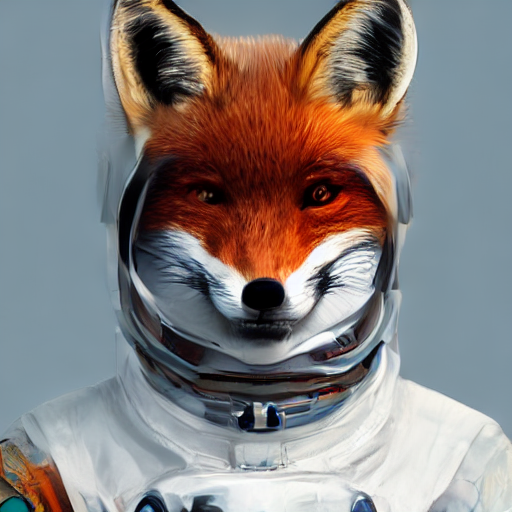}
\includegraphics[width=3cm]{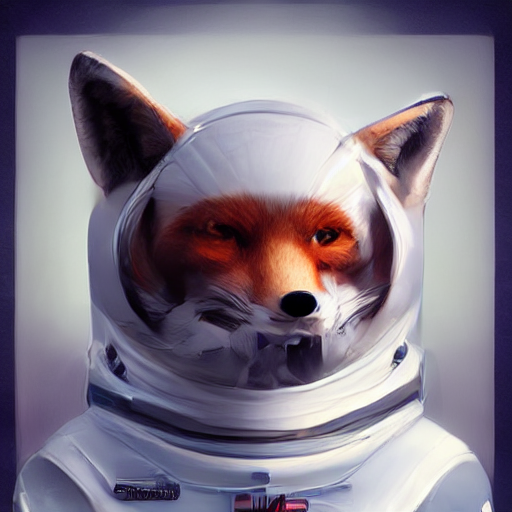}
\end{minipage}
\begin{minipage}[t]{0.49\textwidth}
\centering
\includegraphics[width=3cm]{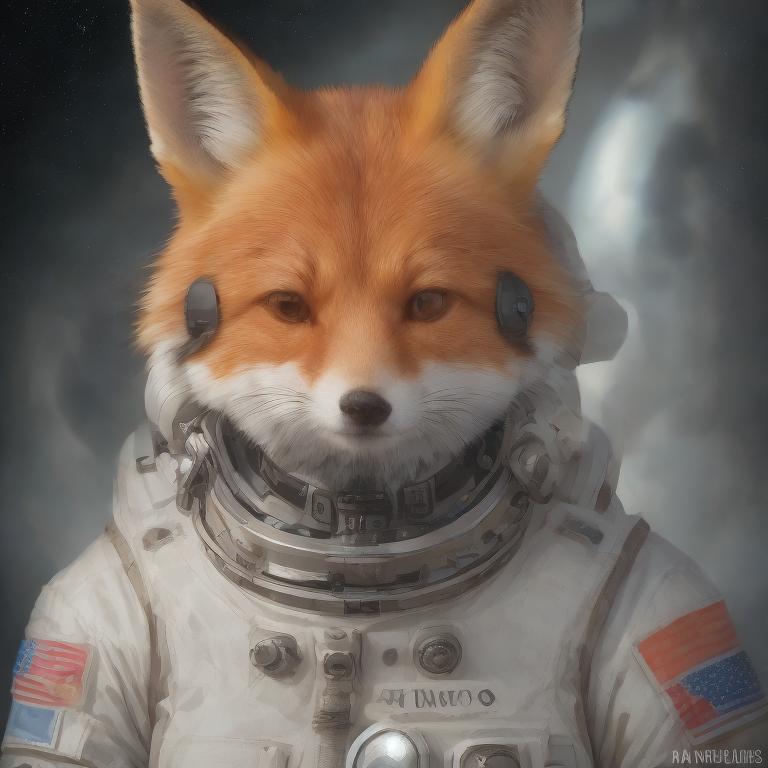}
\includegraphics[width=3cm]{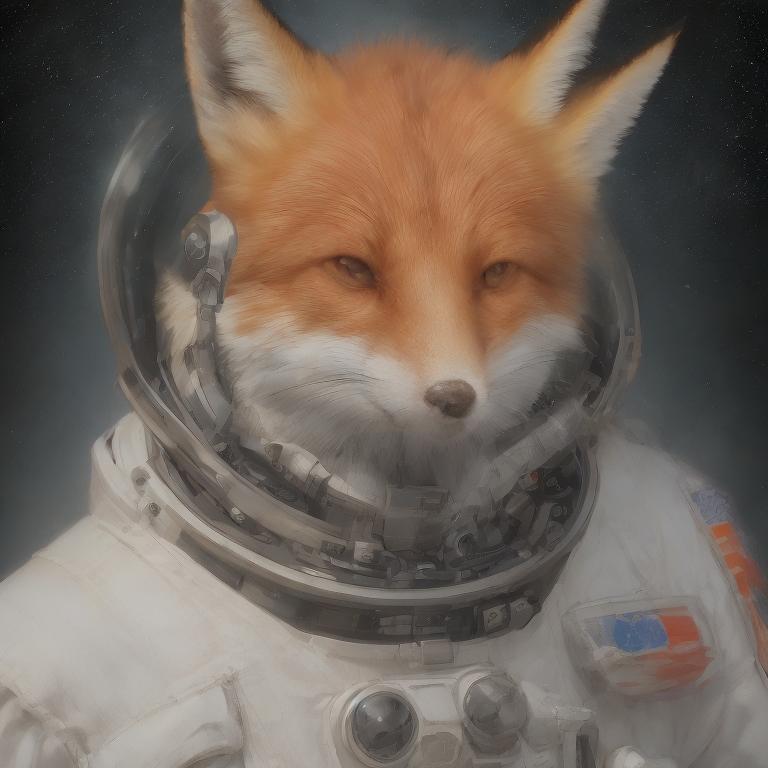}
\includegraphics[width=3cm]{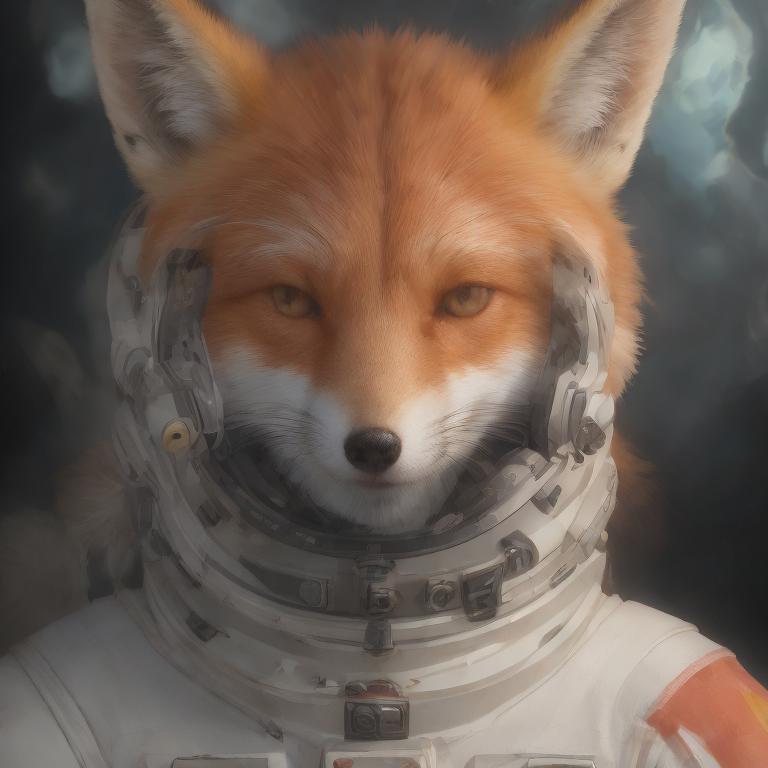}
\includegraphics[width=3cm]{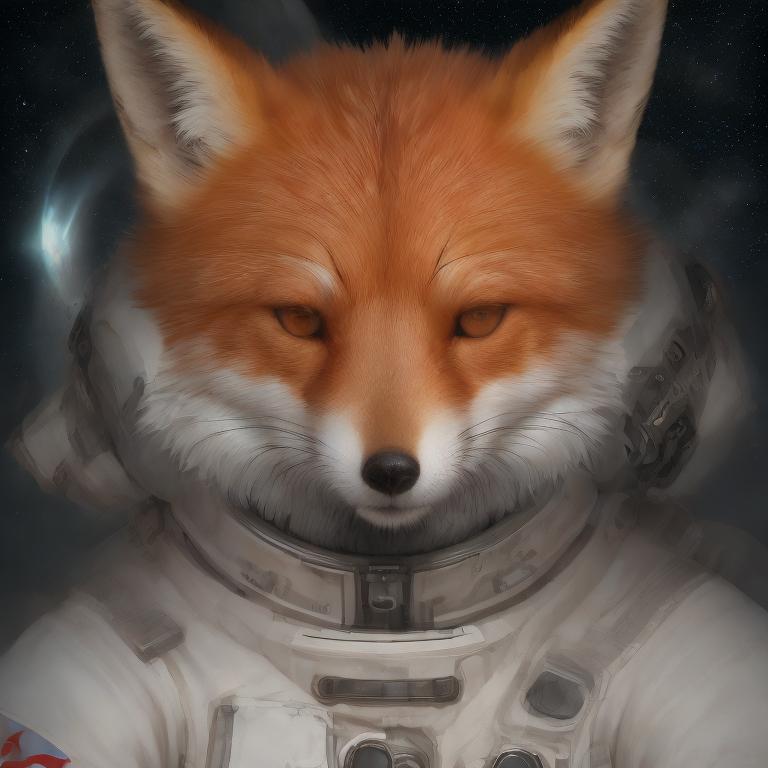}
\end{minipage}
 \begin{center}
\footnotesize \qquad InstaFlow (1 step)\qquad\qquad\qquad\qquad\qquad\qquad\qquad\qquad\qquad\qquad\qquad\qquad\qquad\qquad LCM (4 steps)
\end{center} 
\centering
\flushleft 
\begin{minipage}[c]{0.49\textwidth}
\centering
\includegraphics[width=3cm]{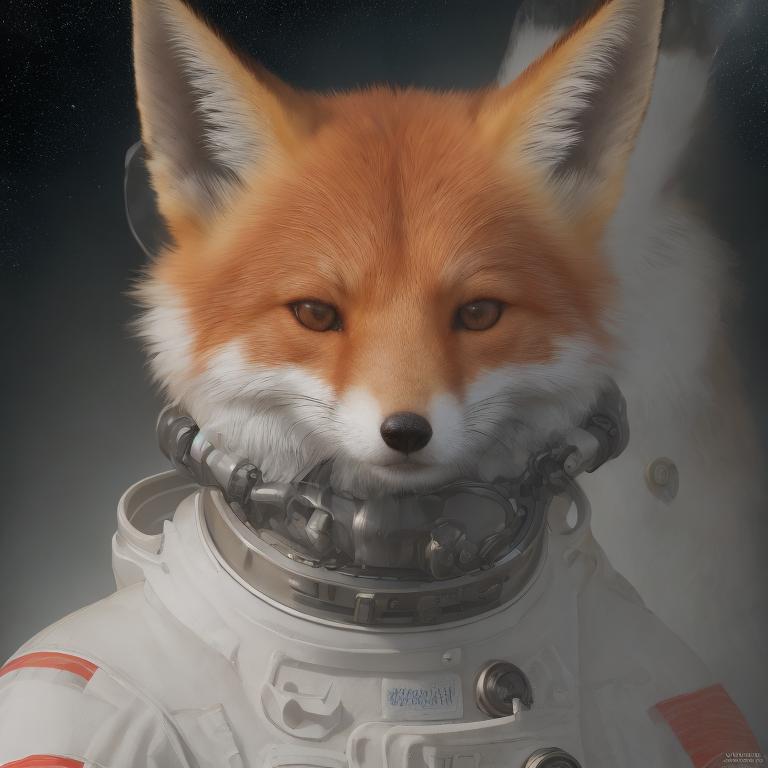}
\includegraphics[width=3cm]{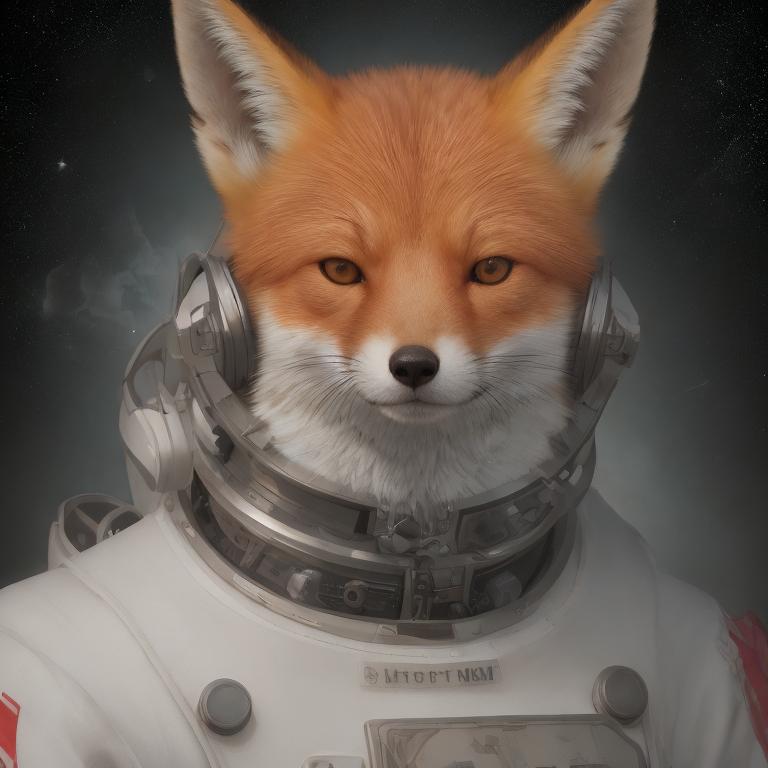} 
\includegraphics[width=3cm]{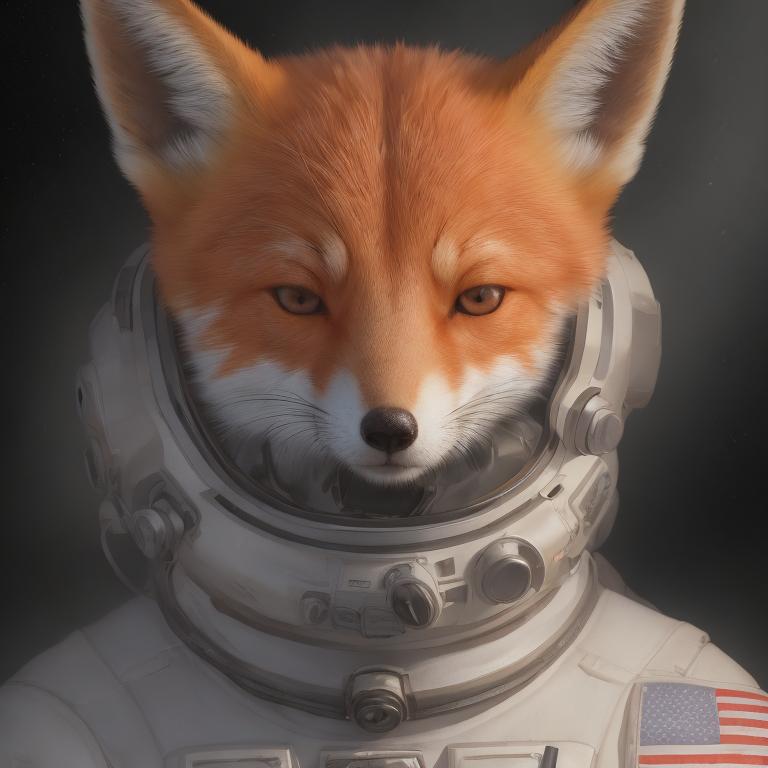}
\includegraphics[width=3cm]{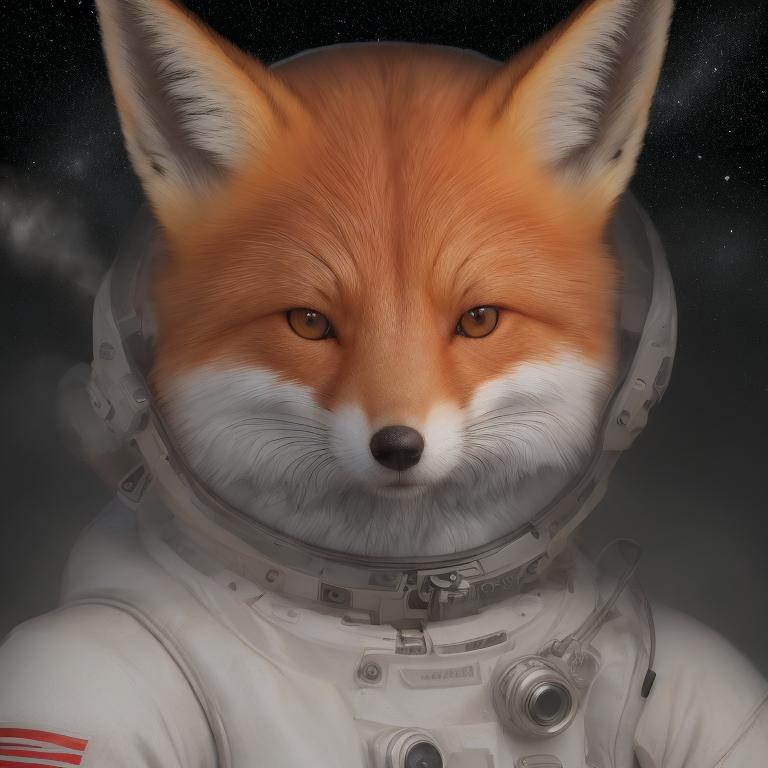}

\end{minipage}
\begin{minipage}[c]{0.49\textwidth}
\centering
\includegraphics[width=3cm]{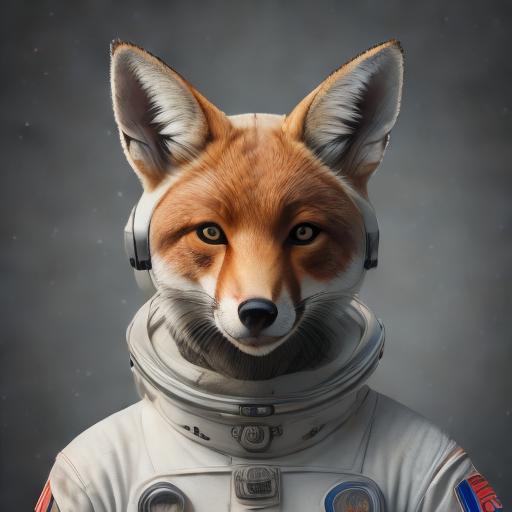}
\includegraphics[width=3cm]{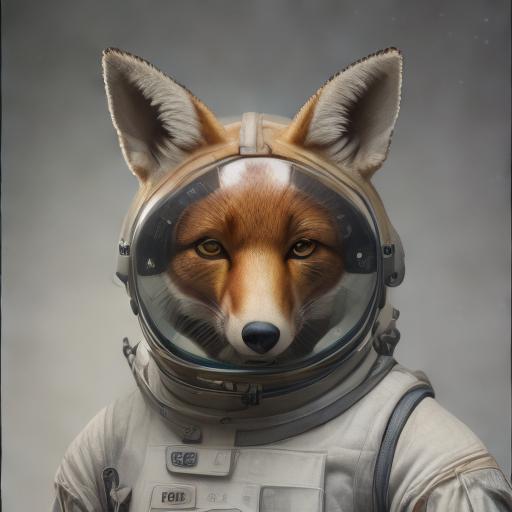} 
\includegraphics[width=3cm]{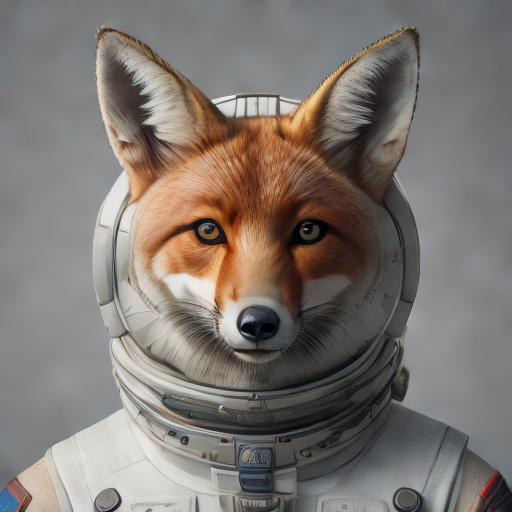}
\includegraphics[width=3cm]{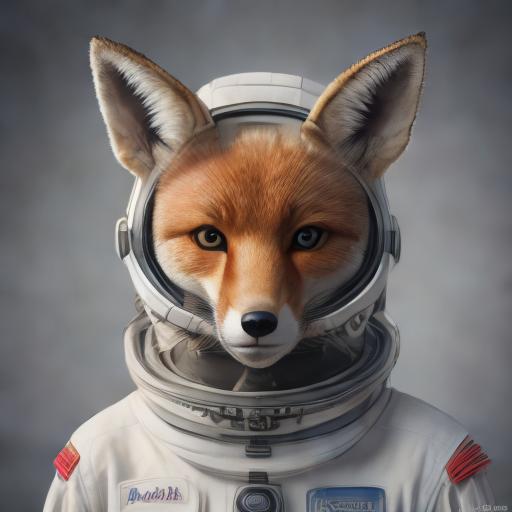}
\end{minipage}
 \begin{center}
\footnotesize \qquad LCM (8-steps)\qquad\qquad\qquad\qquad\qquad\qquad\qquad\qquad\qquad\qquad\qquad\qquad\qquad\qquad SCott (2 steps)
\end{center} 
\caption{Prompt: Hyperrealistic photo of a fox astronaut, perfect face, artstation.}
\label{fig: fox}
\end{figure}

\begin{figure}[htbp]
\centering
\begin{minipage}[t]{0.49\textwidth}
\centering
\includegraphics[width=3cm]{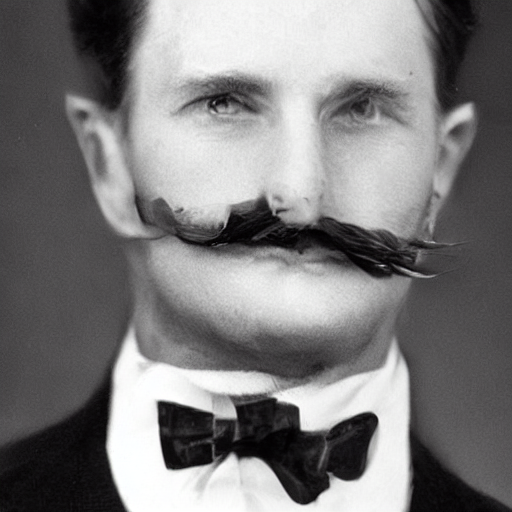}
\includegraphics[width=3cm]{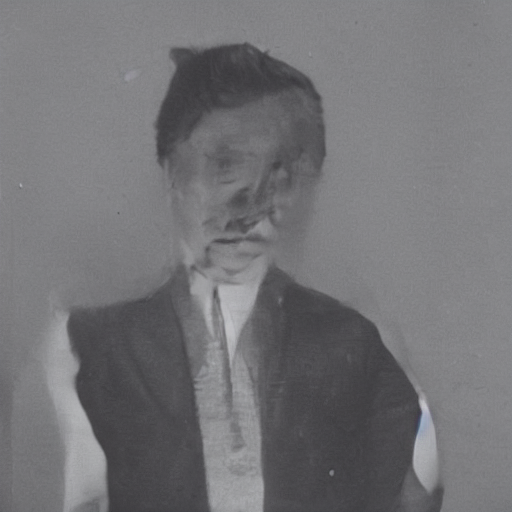}
\includegraphics[width=3cm]{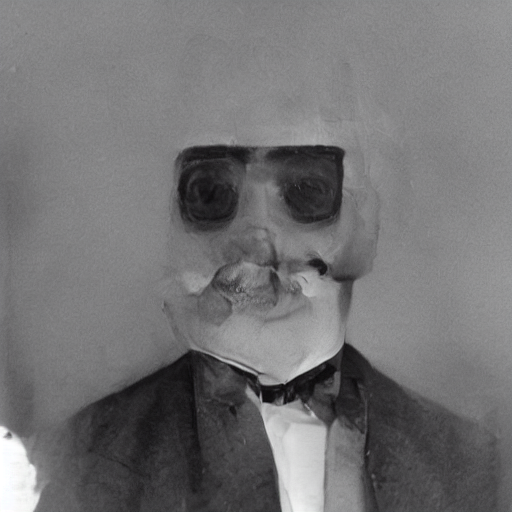}
\includegraphics[width=3cm]{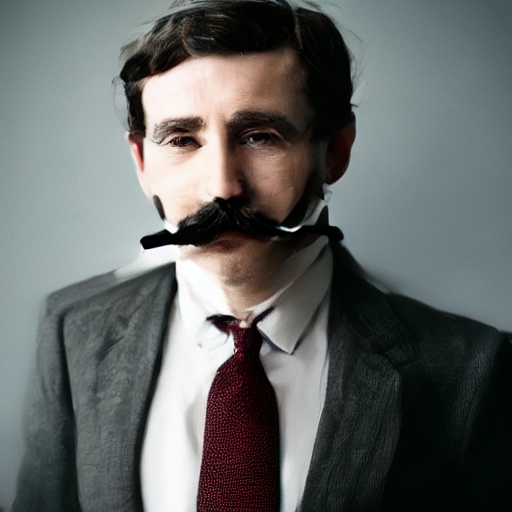}
\end{minipage}
\begin{minipage}[t]{0.49\textwidth}
\centering
\includegraphics[width=3cm]{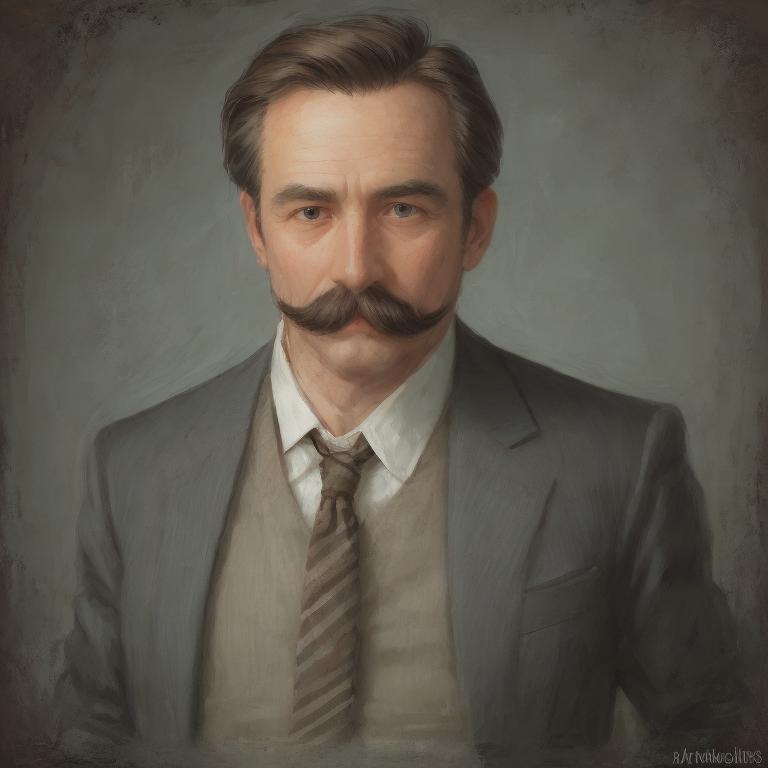}
\includegraphics[width=3cm]{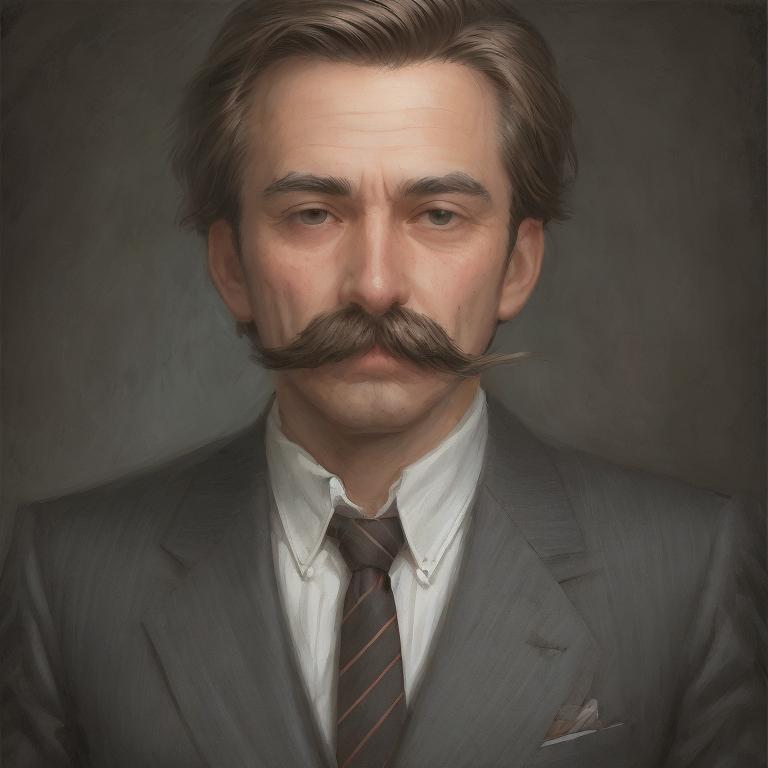}
\includegraphics[width=3cm]{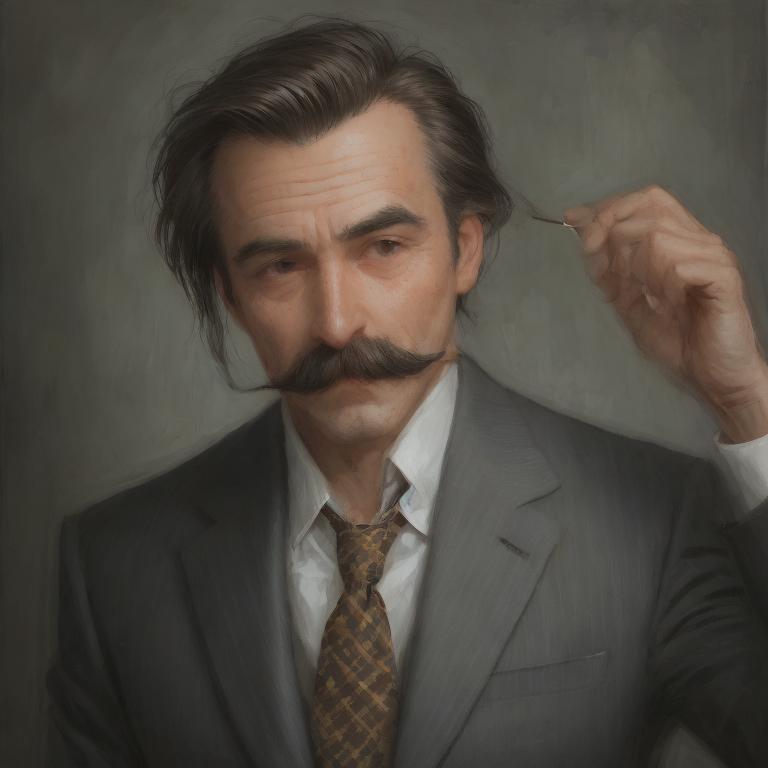}
\includegraphics[width=3cm]{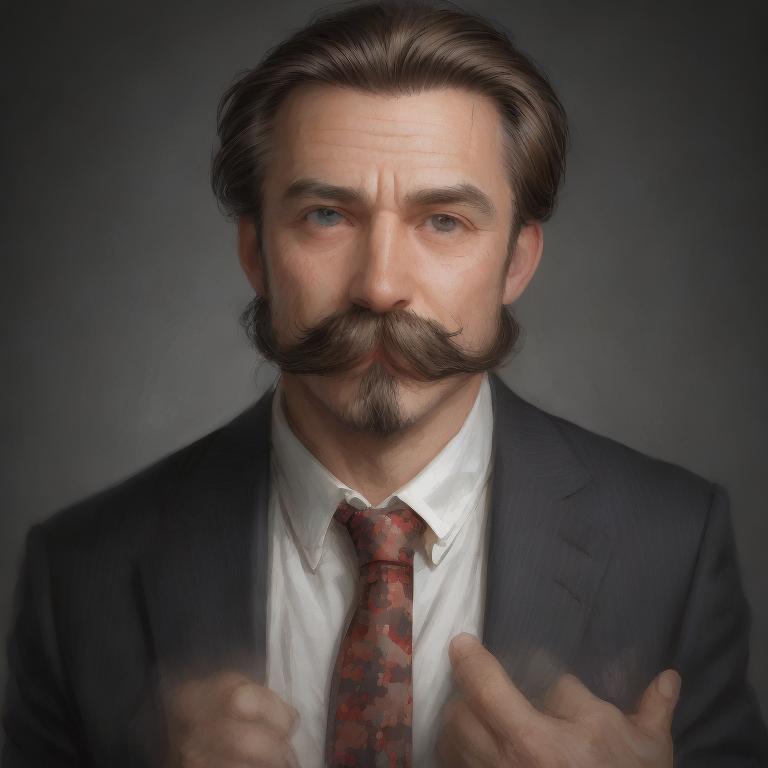}
\end{minipage}
 \begin{center}
\footnotesize \qquad InstaFlow (1 step)\qquad\qquad\qquad\qquad\qquad\qquad\qquad\qquad\qquad\qquad\qquad\qquad\qquad\qquad LCM (4 steps)
\end{center} 
\centering
\flushleft 
\begin{minipage}[c]{0.49\textwidth}
\centering
\includegraphics[width=3cm]{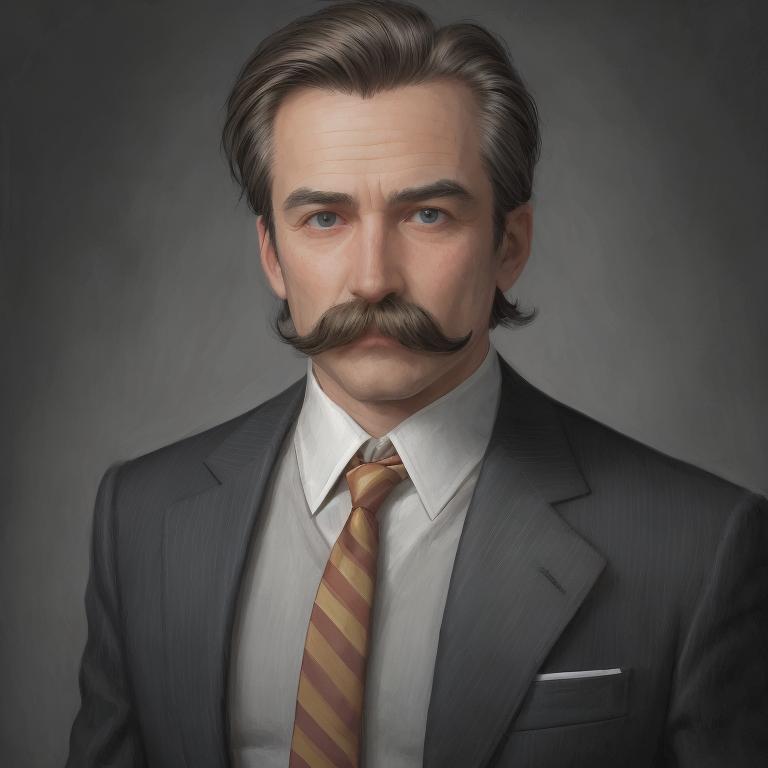}
\includegraphics[width=3cm]{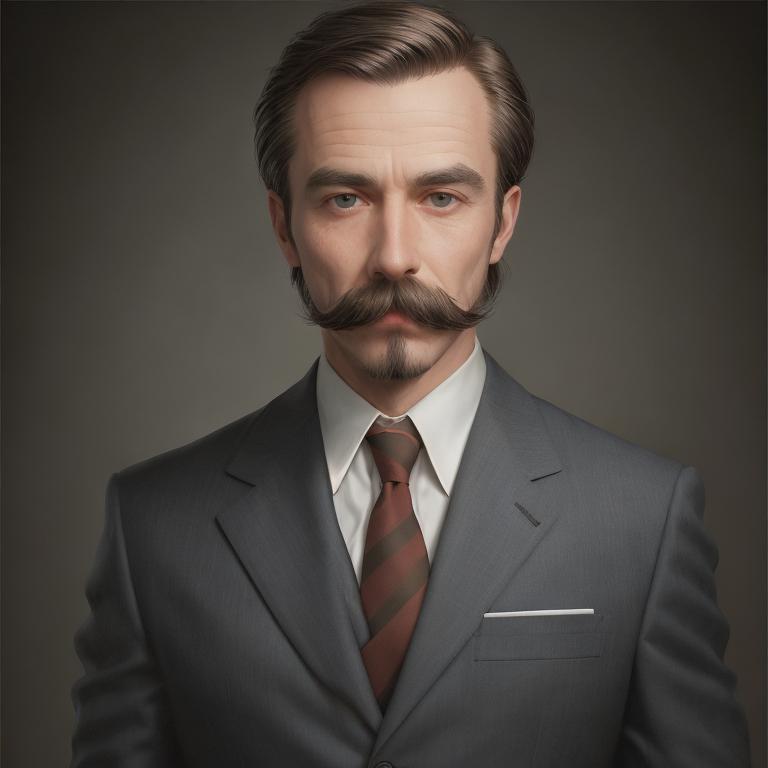} 
\includegraphics[width=3cm]{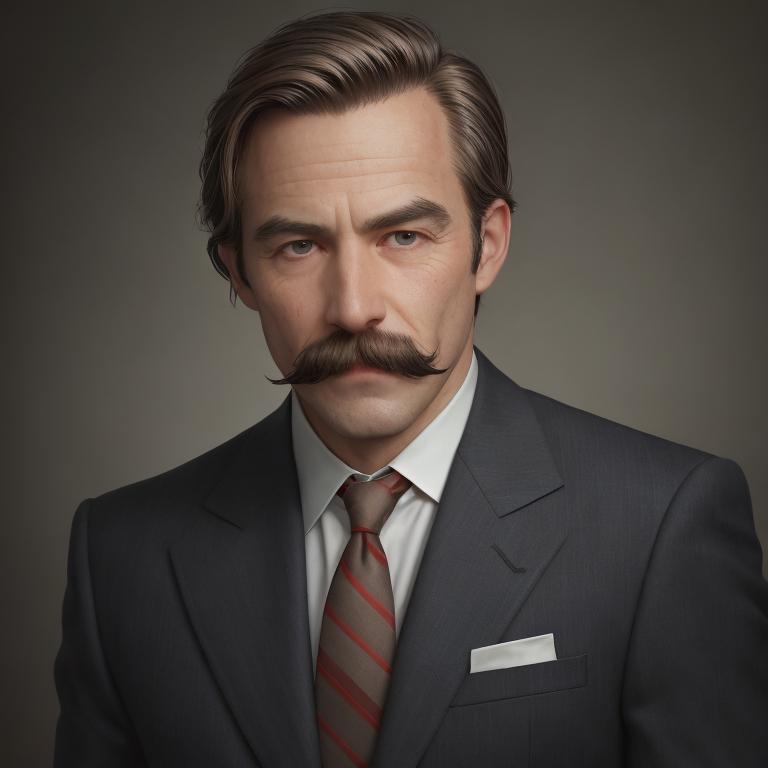}
\includegraphics[width=3cm]{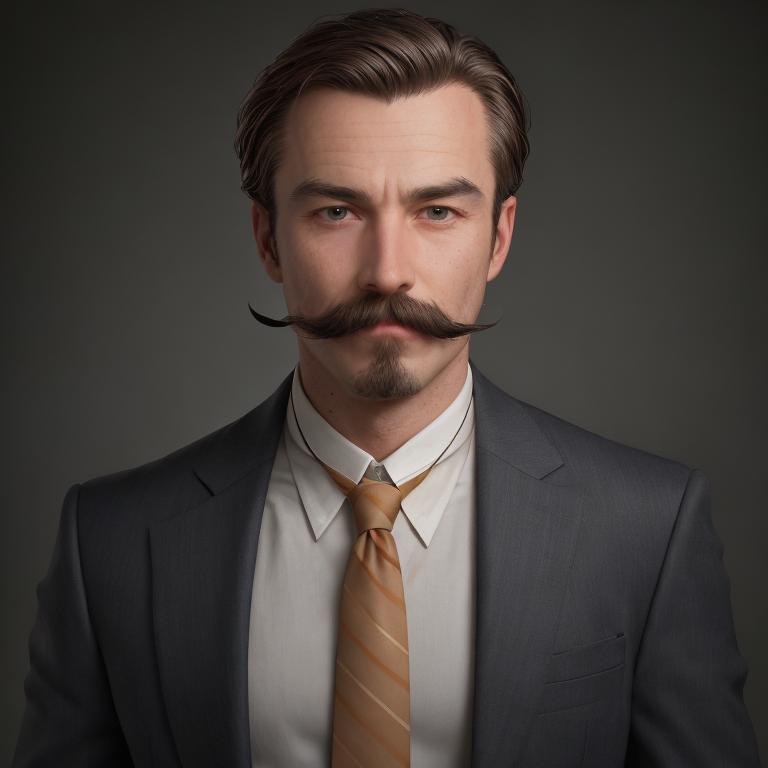}

\end{minipage}
\begin{minipage}[c]{0.49\textwidth}
\centering
\includegraphics[width=3cm]{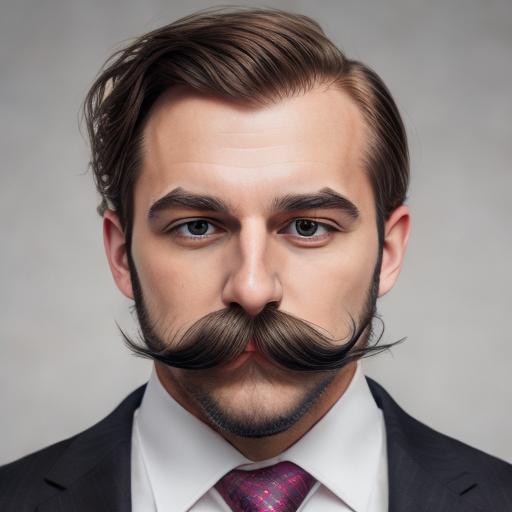}
\includegraphics[width=3cm]{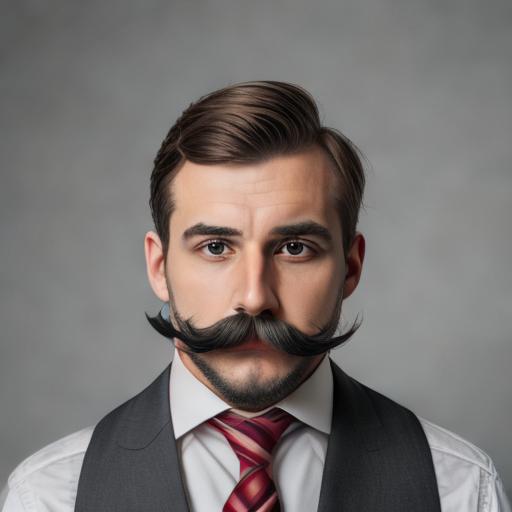} 
\includegraphics[width=3cm]{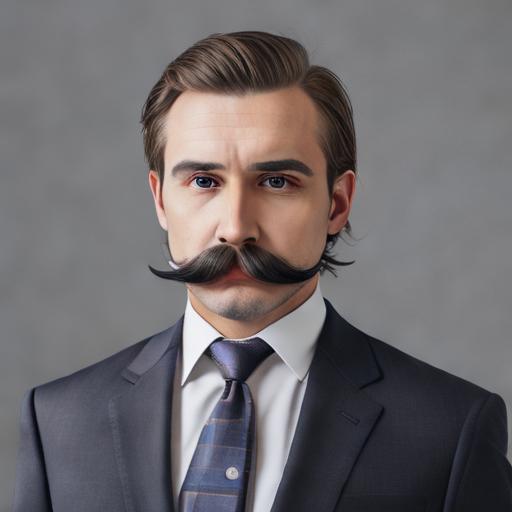}
\includegraphics[width=3cm]{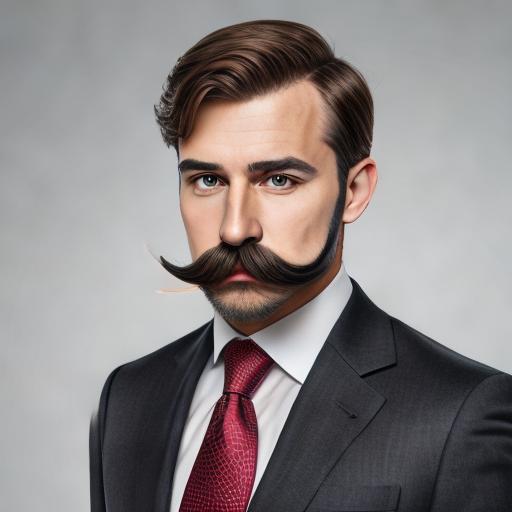}
\end{minipage}
 \begin{center}
\footnotesize \qquad LCM (8-steps)\qquad\qquad\qquad\qquad\qquad\qquad\qquad\qquad\qquad\qquad\qquad\qquad\qquad\qquad SCott (2 steps)
\end{center} 
\caption{Prompt: A man in a tie and a fake moustache.}
\label{fig: man}
\end{figure}

\end{document}